\documentclass{article}
\usepackage{inputenc}
\usepackage{microtype}
\usepackage{graphicx}
\usepackage{booktabs} % for professional tables
\usepackage{amsmath}
\usepackage{amsthm}
\usepackage{amssymb}
\usepackage{url}
\usepackage{hyperref}
\usepackage[accepted]{icml2021}
\usepackage{multirow}
\usepackage{subfig}
\usepackage{natbib}
\newcommand{\bw}{\text{\boldmath{$w$}}}

\newcommand{\bz}{\boldsymbol{z}}
\newcommand{\bx}{\boldsymbol{x}}

\newcommand{\bdelta}{\boldsymbol{\delta}}

\newcommand{\bu}{\boldsymbol{u}}

\newcommand{\bbP}{\mathbb{P}}

\newcommand{\bbR}{\mathbb{R}}

\newcommand{\bv}{\boldsymbol{v}}

\newcommand{\sW}{\mathsf{W}}

\newtheorem{assumption}{\textbf{Assumption}}\newtheorem{definition}{\textbf{Definition}}\newtheorem{corollary}{\textbf{Corollary}}\newtheorem{lemma}{\textbf{Lemma}}\newtheorem{theorem}{\textbf{Theorem}}\newtheorem{proposition}{\textbf{Proposition}}\newtheorem{remark}{\textbf{Remark}}

\newcommand{\mE}{\mathbb{E}}
\newcommand{\Var}{\mathsf{Var}}
\newcommand{\TV}{\mathsf{TV}}
\newcommand{\cE}{\mathcal{E}}

\newcommand{\cX}{\mathcal{X}}

\newcommand{\cL}{\mathcal{L}}

\newcommand{\cN}{\mathcal{N}}

\newcommand{\cP}{\mathcal{P}}

\newcommand{\cF}{\mathcal{F}}

\newcommand{\cO}{\mathcal{O}}

\makeatother

\icmltitlerunning{Improved OOD Generalization via Adversarial Training and Pre-training}

\begin{document}
	\twocolumn[
	% 	\icmltitle{Towards Better OOD Generalization via Adversarial Training}
	\icmltitle{Improved OOD Generalization via Adversarial Training and Pre-training}
	
	% It is OKAY to include author information, even for blind
	% submissions: the style file will automatically remove it for you
	% unless you've provided the [accepted] option to the icml2021
	% package.
	
	% List of affiliations: The first argument should be a (short)
	% identifier you will use later to specify author affiliations
	% Academic affiliations should list Department, University, City, Region, Country
	% Industry affiliations should list Company, City, Region, Country
	
	% You can specify symbols, otherwise they are numbered in order.
	% Ideally, you should not use this facility. Affiliations will be numbered
	% in order of appearance and this is the preferred way.
	\icmlsetsymbol{intern}{\dag}
	
	\begin{icmlauthorlist}
		\icmlauthor{Mingyang Yi}{ucas,amss,intern}
		\icmlauthor{Lu Hou}{huawei}
		\icmlauthor{Jiacheng Sun}{huawei}
		\icmlauthor{Lifeng Shang}{huawei}
		\icmlauthor{Xin Jiang}{huawei}
		\icmlauthor{Qun Liu}{huawei}
		\icmlauthor{Zhi-Ming Ma}{ucas,amss}
	\end{icmlauthorlist}
	
	\icmlaffiliation{ucas}{University of Chinese Academy of Sciences, Beijing, China}
	\icmlaffiliation{amss}{Academy of Mathematics and Systems Science, Chinese Academy of Sciences, Beijing, China}
	\icmlaffiliation{huawei}{Huawei Noah’s Ark Lab, Shenzhen, China}
	
	\icmlcorrespondingauthor{Mingyang Yi}{yimingyang17@mails.ucas.edu.cn}
	\icmlcorrespondingauthor{Lu Hou}{houlu3@huawei.com}
%	\icmlcorrespondingauthor{Lifeng Shang, Xin Jiang, Qun Liu}{\{shang.lifeng,Jiang.Xin,qun.liu\}@huawei.com}
%	\icmlcorrespondingauthor{Zhi-Ming Ma}{mazm@amt.ac.cn}
	% You may provide any keywords that you
	% find helpful for describing your paper; these are used to populate
	% the "keywords" metadata in the PDF but will not be shown in the document
	\icmlkeywords{Machine Learning, ICML}
	
	\vskip 0.3in
	]
	% this must go after the closing bracket ] following \twocolumn[ ...
	
	% This command actually creates the footnote in the first column
	% listing the affiliations and the copyright notice.
	% The command takes one argument, which is text to display at the start of the footnote.
	% The \icmlEqualContribution command is standard text for equal contribution.
	% Remove it (just {}) if you do not need this facility.
	
	\printAffiliationsAndNotice{\icmlEqualContribution}  % leave blank if no need to mention equal contribution
	% 	\printAffiliationsAndNotice{\icmlEqualContribution} % otherwise use the standard text.
	\begin{abstract}
		Recently, learning a model that generalizes well on out-of-distribution (OOD) data has attracted great attention in the machine learning community. In this paper, 
		after defining OOD generalization via Wasserstein distance,
		we theoretically show that a model robust to input perturbation generalizes well on OOD data.   
		Inspired by previous findings that adversarial training helps improve input-robustness, we theoretically show that adversarially trained models have converged excess risk on OOD data, and empirically verify it on both image classification and natural language understanding tasks. 
		Besides, in the paradigm of first pre-training and then fine-tuning, we theoretically show that a pre-trained model that is more robust to input perturbation 
		provides a better initialization for generalization on downstream OOD data.
		Empirically,   
		after fine-tuning, this better-initialized model from adversarial pre-training also has better OOD generalization.
		% on image classification tasks. 
		% 		Interestingly, we also empirically find that recent popular pre-traind language models trained by the objective of  masked language modeling already have a certain amount of input-robustness.
		% 		More data and longer training further enhances this input-robustness.
		
		% 		Finally, various experiments conducted on image classification and natural language understanding tasks verify our theoretical findings.
	\end{abstract}
	\section{Introduction}
	In the machine learning community, the training and test distributions are often not identically distributed. Due to  this mismatching, it is desired to learn a model that generalizes well on out-of-distribution (OOD) data though only trained on data from one certain distribution. OOD generalization is empirically studied in \citep{hendrycks2019using,hendrycks2020many,hendrycks2020pretrained} by 
	evaluating the performance of the model on the test set that is close to the original training samples. However, the theoretical understanding of these empirical OOD generalization behaviors remains unclear.
	\par
	Intuitively, the OOD generalization measures the performance of the model on the data from a shifted distribution around the original training distribution \citep{hendrycks2018benchmarking}. This is equivalent to the distributional robustness \citep{namkoong2019reliable,shapiro2017distributionally} which measures the model's robustness to perturbations the distribution of training data. Inspired by this, we study the OOD generalization by utilizing the Wasserstein distance to measure the shift between distributions (Definition \ref{def: model robustness}). We theoretically find that if a model is robust to input perturbation on training samples (namely, input-robust model), it also generalizes well on OOD data. 
	% 	The intuition is that the OOD data are in a neighborhood of the in-distribution data with high probability, and thus an input-robust model would generalize well on the OOD data.
	\par    
	% adversarial learning => good ood performance
	The connection of input-robustness and OOD generalization inspires us to find an input-robust model since it generalizes well on OOD data. Thus we consider adversarial training (AT)~\citep{madry2018towards}
	as 
	% 	it if shown in 
	\citet{athalye2018obfuscated} show that a model is input-robust if it defends  adversarial perturbations \citep{szegedy2013intriguing}. 
	%,goodfellow2015explaning}. 
	Mathematically, AT can be formulated as a minimax optimization problem and solved by the multi-step %stochastic gradient descent 
	SGD algorithm \citep{nouiehed2019solving}. Under mild assumptions, we prove that the convergence rate of this multi-step SGD for AT is 
	% 	in the order of 
	$\tilde{\cO}(1/T)$ 
	% 	(definition of $\tilde{\cO}(\cdot)$ is in the paragraph of notations) 
	both in expectation and in high probability, 
	where $T$ is the number of training steps and $\tilde{\cO}(\cdot)$ is defined in the paragraph of notations. Then, combining the convergence result with the relationship between input-robustness and OOD generalization, we theoretically show that for the model adversarially trained with $n$ training samples for $T$ steps, its excess risk on the OOD data is upper bounded by $\tilde{\cO}(1/\sqrt{n} + 1/T)$, which guarantees its performance on the OOD data.          
	\par
	Besides models trained from scratch, we also study the OOD generalization
	on downstream tasks 
	of pre-trained models, as 
	the paradigm of first pre-training on a large-scale dataset and then fine-tuning on downstream tasks 
	has achieved remarkable performance in both computer vision (CV)~\citep{hendrycks2019using,kornblith2019better} and natural language processing (NLP) domains \citep{devlin2019bert} recently. Given the aforementioned relationship of input-robustness and OOD generalization, we theoretically show that a pre-trained model more robust to input perturbation also provides a better initialization for generalization on downstream OOD data. Thus, we suggest conducting adversarial pre-training like \citep{salman2020adversarially,hendrycks2019using,utrera2020adversarially},  to
	improve the OOD generalization in downstream tasks.
	\par
	We conduct various experiments on both image classification (IC) and natural language understanding (NLU) tasks to verify our theoretical findings. 
	\par
	For IC task, we conduct AT on \texttt{CIFAR10} \citep{krizhevsky2009learning} and \texttt{ImageNet} \citep{deng2009imagenet}, and then evaluate the OOD generalization of these models on corrupted OOD data \texttt{CIFAR10-C} and \texttt{ImageNet-C} \citep{hendrycks2018benchmarking}. For NLU tasks, we similarly conduct AT as in \citep{zhu2019freelb} on datasets \texttt{SST-2}, \texttt{IMBD}, \texttt{MNLI} and \texttt{STS-B}. 
	Then we follow the strategy in \citep{hendrycks2020pretrained} to evaluate the OOD generalization. 
	Empirical results on both IC and NLU tasks verify that AT improves OOD generalization. 
	\par
	To see the effect of the initialization provided by an input-robust pre-trained model, we adversarially pre-train a model on \texttt{ImageNet} to improve the input-robustness, and then fine-tune the pre-trained model on \texttt{CIFAR10}. 
	Empirical results show that this initialization enhances the OOD generalization on downstream tasks after fine-tuning.
	Another interesting observation is that for language models, standard pre-training by masked language modeling \citep{devlin2019bert,liu2019roberta} improves the input-robustness of the model. 
	Besides, models pre-trained with more training samples and updating steps 
	% 	in the pre-training stage 
	are more input-robust. This may also explain the better OOD generalization on downstream tasks~\citep{hendrycks2020pretrained} of these models. 
	% 	$\text{BERT}_{\text{BASE}}$ is more robust to 
	% 	compared with the randomly-initialized counterpart. 
	% 	It also explains the observation in  that directly fine-tuning the $\text{BERT}_{\text{BASE}}$ on downstream tasks already generalizes well on OOD data.  
	\paragraph{Notations.}
	% 	$\{(\bx_{i}, y_{i})\}$ is the
	% 	training set with 
	% 	$n$ i.i.d training samples $\{\bx_{i}\}$ and their labels $\{y_{i}\}$. 
	For vector $\bx\in\bbR^{d_{0}}$, $\|\bx\|_{p}$ is its $\ell_{p}$-norm, and 
	its $\ell_{2}$-norm is simplified as $\|\bx\|$. 
	% 	\yi{The input data $\bx$ has compact support $\cX\subseteq \bbR^{d_{0}}$ such that $\|\bu - \bv\|_{1}\leq D$ for $D > 0$ and $\bu, \bv\in\cX$.}  
	$\cP(\cX)$ is the set of probability measures on metric space $(\cX, \|\cdot\|_{p})$ with $\cX \subseteq \bbR^{d_{0}}$. 
	$\cO(\cdot)$ is the order of a number, and $\tilde{\cO}(\cdot)$ hides a poly-logarithmic factor in problem parameters e.g., 
	% 	$M_1 = 1/n$ and 
	$\cO(M_1\log{d_{0}}) = \tilde{O}(M_{1})$. 
	For $P, Q\in\cP(\cX)$, 
	% 		For $P, Q\in \cP(\cX)$. 
	let $(P, Q)$ be their couplings (measures on $\cX\times \cX$).
	The $p$-th ($p<\infty$) Wasserstein distance \citep{villani2008optimal} between $P$ and $Q$ is 
	\begin{equation}
	\label{eq:w distance}
	\small
	\sW_{p}(P, Q) = \left(\inf_{\pi\in(P, Q)}\mE_{(\bu, \bv)\sim \pi}\left[\|\bu - \bv\|^{p}_{p}\right]\right)^{\frac{1}{p}}. 
	\end{equation}
	When $p=\infty$, the $\infty$-Wasserstein distance is $\sW_{\infty}(P, Q) = \lim_{p\to\infty}\sW_{p}(P, Q)$. 
	In the sequel, the $p$-Wasserstein distance is abbreviated as $\sW_{p}$-distance. 
	The total variation distance \citep{villani2008optimal}  is a kind of distributional distance and is defined as 
	\begin{equation}\label{eq:tv}
	\small
	\TV(P, Q) = \frac{1}{2}\int_{\cX}\left|dP(\bx) - dQ(\bx)\right|.
	\end{equation}
	\section{Related Work}
	
	\paragraph{OOD Generalization.}
	OOD generalization measures  a model's ability to extrapolate beyond the training distribution~\citep{hendrycks2018benchmarking}, and  
	has been widely explored in both CV~\citep{recht2019imagenet,schneider2020improving,salman2020unadversarial} and NLP domains~\citep{tu2020empirical,lohn2020estimating}. 
	\citet{hendrycks2018benchmarking} observe that the naturally trained models are sensitive to  artificially constructed OOD data. They also find that adversarial logit pairing \citep{kannan2018adversarial} can improve a model's performance on noisy corrupted OOD data. 
	% 	 It is also empirically found in 
	\citet{hendrycks2020pretrained} also empirically find that pre-trained language models %\citep{vaswani2017attention} 
	generalize on downstream OOD data. 
	But the theoretical understanding behind these observations remains unclear.
	\paragraph{Adversarial Training.}
	% 	and Robust Optimization.}
	% ADVERSARIAL TRAINING IMPROVE INPUT-ROBUSTNESS
	% ood <=> distributional robustness which is conventionally solved by robust optimization
	% this paper uses at -> distributional robustness.
	Adversarial training \citep{madry2018towards} is proposed to improve input-robustness  %\citep{szegedy2013intriguing,goodfellow2015explaning}, 
	by dynamically constructing the augmented adversarial samples \citep{szegedy2013intriguing,goodfellow2015explaning} 
	using projected gradient descent across training. %\citet{xie2020adversarial} empirically shows that adversarially trained model with an extra separated normalization layer has better performance on the OOD data. 
	In this paper, we first show the close relationship between OOD generalization and distributional robustness \citep{ben2013robust,shapiro2017distributionally}, and then
	explore the OOD generalization by 
	connecting input-robustness 
	and distributional robustness. 
	% 	The other methods of training a distributional robust model refer to \citep{gao2017wasserstein,sinha2018certifying,lee2018minimax,volpi2018generalizing,staib2019distributionally}.
	
	\par
	% 	\footnote{To be short}
	The most related works to ours are 
	\citep{sinha2018certifying,lee2018minimax,volpi2018generalizing}.
	They also use AT to train distributionally robust models under the Wasserstein distance, but their results are restricted to a specialized AT objective with an additional regularizer. 
	The regularizer can be impractical due to its large penalty parameter. 
	Moreover, their bounds are built upon the entropy integral and increase with model capacity, which can be meaningless for high-dimensional models. 
	On the other hand, our bound is  
	(i) based on the input-robustness, regardless of how it is obtained; and (ii) irrelevant to model capacity.
	% 	or the way of obtaining the input-robustness. 
	% In addition, based on this analysis, we further theoretically and empirically explore why adversarial pre-training improves OOD performance\footnote{better move this to the next para?}.
	% 	we also explore adversarial pre-training and conduct quite a lot empirical studies on OOD generalization.   
	
	\paragraph{Pre-Training.}
	Pre-trained models transfer the  knowledge in the pre-training stage to downstream tasks, 
	and are widely used in both CV \citep{kornblith2019better} and NLP \citep{devlin2019bert} domains. 
	For instance, \citet{dosovitskiy2020image,brown2020language,radford2021learning} pre-train the transformer-based models on large-scale datasets, and obtain remarkable results on downstream tasks.
	% results by transferring its knowledge to the downstream tasks. % \citet{brown2020language} pre-train a model with billions of parameters on a large corpus of text, 
	% and the model achieves strong performance on various downstream NLP tasks without even 
	% 	any gradient updates or 
	% fine-tuning.
	% 	\par
	% 		This paper also discusses the relationship between OOD generalization and pre-training.
	Standard pre-training is empirically found to help
	% improves the input-robustness and 
	reduce the uncertainty of the model for both image data \citep{hendrycks2019using,hendrycks2020many}
	and textual data~\citep{hendrycks2020pretrained}. 
	% \citet{hendrycks2019using,hendrycks2020many} empirically find that pre-training on image data  improves the input-robustness and reduce the uncertainty of the model. 
	% 	Similar observations are also empirically found for textual data \citep{hendrycks2020pretrained}. 
	Adversarial pre-training is explored in \citep{hendrycks2019using} and \citep{salman2020adversarially}, and is shown to improve the robustness and generalization on downstream tasks
	, respectively.
	% 	 In contrast to these literature, we discuss  
	% 	However, 
	In this work, we theoretically analyze the OOD generalization on downstream tasks from the perspective of the input-robustness of the pre-trained model.  
	
	\section{Adversarial Training Improves OOD Generalization}
	\label{sec:Learning Robust Model Results in Better OOD Generalization}
	In this section, we first show that the input-robust model can generalize well on OOD data after specifying the definition of OOD generalization. 
	Then, to learn a robust model, we suggest adversarial training (AT) \citep{madry2018towards}. 
	Under mild conditions, we prove a $\tilde{\cO}(1/T)$ convergence rate for AT  both in expectation and in high probability. 
	With this, we show that the excess risk of an adversarially trained model on OOD data is upper bounded by $\tilde{\cO}(1/\sqrt{n} + 1/T)$ where $n$ is the number of training samples.   
	
	\subsection{Input-Robust Model Generalizes on OOD Data}
	\label{sec:Robustness Corresponds with Better OOD Generalization}
	Suppose
	$\{(\bx_{i}, y_{i})\}$ is the
	training set with 
	$n$ i.i.d. training samples $\{\bx_{i}\}$ and their labels $\{y_{i}\}$. 
	We assume the training sample distribution $P$  
	% 	$\bx\sim P$ 
	has compact support $\cX\subseteq \bbR^{d_{0}}$, thus there exists $D > 0$, such that $\forall \bu, \bv\in\cX$, $\|\bu - \bv\|_{1}\leq D$.
	% 	such that 
	% 	$\|\bu - \bv\|_{1}\leq D$ and for $D > 0$ and $ \bu, \bv\in\cX$.
	For training sample $\bx$ and its label $y$,
	% 	respectively represents the data and label, 
	the loss on $(\bx, y)$ with model parameter $\bw$ is $\cL(\bw, (\bx, y))$, where $\cL(\bw, (\bx, y))$ is continuous and differentiable for both $\bw$ and $(\bx, y)$. 
	Besides, we assume $0 \leq \cL(\bw, (\bx, y))\leq M$ for constant $M$ without loss of generality. 
	We represent the expected risk under training distribution $P$ and label distribution $P_{y\mid \bx}$ 
	\footnote{$P_{y\mid \bx_{i}}(\cdot)=\textbf{1}_{\{\cdot = y_{i}\}}$ where $\textbf{1}_{\{\cdot = y_{i}\}}$ is the indicator function.}
	as $R_{P}(\bw) = \mE_{P}[\mE_{P_{y\mid \bx}}[\cL(\bw, (\bx, y))]]$.
	% 	since $P_{y\mid \bx}$ keeps unchanged. 
	% 	We
	For simplicity of notation, let $\mE_{P_{y\mid \bx}}[\cL(\bw, (\bx, y))] = f(\bw, \bx)$ in the sequel,
	e.g., $f(\bw, \bx_{i}) = \cL(\bw, (\bx_{i}, y_{i}))$.  
	\par
	Intuitively, the OOD generalization is decided by the performance of the model on a shifted distribution close to the training data-generating distribution $P_{0}$ \citep{hendrycks2018benchmarking,hendrycks2020pretrained}. 
	Thus defining OOD generalization should involve the distributional distance which measures the distance between distributions. 
	We use the Wasserstein distance as in \citep{sinha2018certifying}. 
	\par
	Let %$\{\bx_{i}\}$ be the $n$ i.i.d. draws from $P_{0}$, 
	$P_{n}(\cdot)=\frac{1}{n}\sum_{i=1}^{n}\textbf{1}_{\{\cdot = \bx_{i}\}}$ be the empirical distribution, and $B_{\sW_{p}}(P_{0}, r) = \{P: \sW_{p}(P_{0}, P) \leq r\}$.
	% as the $r$-neighborhood of $P_{0}$.
	Then we define the OOD generalization error as
	\begin{equation}
	\label{eq:ood gen}
	\small
	\cE_{\text{gen}}^{\text{ood}}(p, r) = \left|\sup_{P\in B_{\sW_{p}}(P_{0}, r)}R_{P}(\bw) - R_{P_{n}}(\bw)\right|,
	\end{equation}
	under the $\sW_{p}$-distance with $p\in\{2, \infty\}$. Extension to the other OOD generalization with $p < \infty$ is straightforward by generalizing the analysis for $p=2$.
	Note that \eqref{eq:ood gen} reduces to the generalization error on in-distribution data when $r=0$.
	% 	\begin{remark}
	
	% 	\end{remark}
	\begin{definition}\label{def: model robustness} 
		A model is $(r, \epsilon, P, p)$-input-robust, if  
		\begin{equation}
		\small
		\mE_{P}\left[\sup_{\|\bdelta\|_{p}\leq r}|f(\bw, \bx + \bdelta) - f(\bw, \bx)|\right] \leq \epsilon.
		\end{equation} 
		% 		where the expectation is taken over $\bx\sim P$. 
	\end{definition}
	% 	With this, 
	With the input-robustness in Definition~\ref{def: model robustness},
	the following Theorems~\ref{thm:ood generalization upper bound} and \ref{thm:ood generalization upper bound l2} give the generalization bounds on the OOD data drawn from $Q\in B_{\sW_{p}}(P_{0}, r_{0})$ with $p\in\{2, \infty\}$. % which indicate that an input-robust model also generalizes well on OOD data.
	% 	, for $p=2$ and $\infty$, respectively.
	\begin{theorem}\label{thm:ood generalization upper bound}
		% 	Let $\{\bx_{i}\}$ be $n$ i.i.d drawn from $P_{0}$. 
		If a model is $(2r, \epsilon, P_{n}, \infty)$-input-robust, then with probability at least $1 - \theta$,
		%	its OOD generalization error satisfies
		\begin{equation}
		\label{eq:ood bound linf}
		\small
		\begin{aligned}
		\cE_{\text{\emph{gen}}}^{\text{\emph{ood}}}(\infty, r_{0}) 
		% 		& = \left|\sup_{P\in B_{\sW_{\infty}}(P_{0}, r_{0})}R_{P}(\bw) - R_{P_{n}}(\bw)\right| \\
		% 		& 
		\leq \epsilon + M\sqrt{\frac{(2d_{0})^{\frac{2D}{r^2} + 1}\log{2} + 2\log{(\frac{1}{\theta}})}{n}},
		\end{aligned}
		\end{equation}
		for any $r_{0}\leq r$. Here $D$ is the $\ell_{1}$-diameter of data support $\cX$ with dimension $d_{0}$, and $M$ is an upper bound of $f(\bw, \bx)$. 
	\end{theorem}
	\begin{theorem}
		\label{thm:ood generalization upper bound l2}
		% 		Let $\{\bx_{i}\}$ are $n$ i.i.d drawn from $P_{0}$. 
		If a model is $(2r/\epsilon, \epsilon, P_{n}, 2)$-input-robust, then with probability at least $1 - \theta$,  
		% its OOD generalization error satisfies
		\begin{equation}
		\label{eq:ood bound l2}
		\small
		\begin{aligned}
		% 			& 
		\!\!\!\!\!	\cE_{\text{\emph{gen}}}^{\text{\emph{ood}}}(2, r_{0}) 
		% 			= \left|\sup_{P\in B_{\sW_{2}}(P_{0}, r_{0})}R_{P}(\bw) - R_{P_{n}}(\bw)\right|\\
		% 			& 
		\!\leq \! (M\!+\!1)\epsilon \! + \! M\sqrt{\frac{(2d_{0})^{\frac{2\epsilon^{2}D}{r^2}\! + \! 1}\log{2} \!+\! 2\log{(\frac{1}{\theta})}}{n}},
		\end{aligned}
		\end{equation}
		for any $r_{0}\leq r$, where the notations follow Theorem \ref{thm:ood generalization upper bound}. 
	\end{theorem}
	% 	Proofs the two theorems are in Appendix \ref{app:proof in Robustness Corresponds with Better OOD Generalization}. 
	% 	These theorems give the generalization error bounds on OOD data drawn from $Q\in B_{\sW_{p}}(P_{0}, r_{0})$ with $r_{0} \leq r$ and $p=2,\infty$. 
	\begin{remark}
		When $r_{0}=0$, the bounds in Theorems \ref{thm:ood generalization upper bound} and \ref{thm:ood generalization upper bound l2} become the generalization bounds on in-distribution data. 
	\end{remark}
	% 	Specifically, when we take $r_{0}=0$, the bounds characterize the generalization error on in-distribution data. 
	\par
	\begin{remark}
		The $\epsilon$ in Theorem \ref{thm:ood generalization upper bound l2} can not be infinitely small, as the model is required to be robust in  $B(\bx_{i}, 2r/\epsilon)$ for each $\bx_{i}$. Specifically, when $\epsilon \to 0$, the robust region $B(\bx_{i}, 2r/\epsilon)$ can cover the data support $\cX$, then the model has almost constant output in $\cX$.
	\end{remark}
	\begin{remark}
		The bounds \eqref{eq:ood bound linf} and \eqref{eq:ood bound l2} become vacuous when $r$ is large. Thus, our results can not be applied to those OOD data from distributions far away from the original training distribution. For example, ImageNet-R \citep{hendrycks2020many} consists of data from different renditions e.g., photo vs. cartoon, where most pixels vary, leading to large $\|\bu - \bv\|_{p}^{p}$ in \eqref{eq:w distance}, and thus large distributional distance.
	\end{remark}
	The proofs of Theorems \ref{thm:ood generalization upper bound} and \ref{thm:ood generalization upper bound l2} are in Appendix \ref{app:proof in Robustness Corresponds with Better OOD Generalization}. 
	Lemmas \ref{lem:equivalence} and \ref{lem:optimal} in Appendix \ref{app:proof in Learning Robust Model Results in Better OOD Generalization} show that the OOD data concentrates around the in-distribution data with high probability. Thus, the robustness of model on training samples guarantees the generalization on OOD data.
	% 	The $\sW_{p}(P_{0}, Q)$, according to the Definition \ref{eq:w distance}, is actually the minimum averaged moving distance of turning the data follows $P_{0}$ into the OOD data $Q$. 
	% 	Then, the OOD data should mostly concentrate around the in-distribution data\footnote{there is no direct link between this sentence and the prev sentence. Use refs to clarify this.}. 
	% 	Thus the distributional perturbation can be converted into the perturbation on input data, which leads to connection with the robustness of the model. A detailed description to the OOD data can be found in the proof of Theorem \ref{thm:pretrain generalize} and \ref{thm:pretrain generalize l2}\footnote{waht is this sentence used for ?}.     
	% We have the following 
	The observations from Theorems \ref{thm:ood generalization upper bound} and \ref{thm:ood generalization upper bound l2} are summarized as follows.
	\begin{enumerate}
		\item  
		The right-hand sides of bounds \eqref{eq:ood bound linf} and \eqref{eq:ood bound l2} imply that a more input-robust model (i.e., a larger $r$ and a smaller $\epsilon$ in Definition \ref{def: model robustness}) has smaller OOD generalization bound, and thus better performance on OOD data.
		\item  For both \eqref{eq:ood bound linf} and \eqref{eq:ood bound l2},
		a larger number of training samples $n$ results in smaller upper bounds. This indicates that in a high-dimensional data regime with a large feature dimension  $d_{0}$ of data and diameter  $D$ of data support, more training samples can compensate for generalization degradation caused by large $d_{0}$ and $D$.
		% 		to guarantee the OOD generalization .
		\item   The bounds \eqref{eq:ood bound linf} and \eqref{eq:ood bound l2} are independent of the model capacity.
		Compared with other uniform convergence generalization bounds which increase with the model capacity (e.g., Rademacher complexity \citep{yin2019rademacher} or entropy integral \citep{sinha2018certifying}), our bounds are superior for models with high capacity.
	\end{enumerate}
	
	\subsection{Adversarial Training Improves Input-Robustness}\label{sec:robust training}
	\begin{algorithm}[t!]
		\caption{Multi-Step SGD.}
		\label{alg:sgd}
		\textbf{Input:} Number of training steps $T$, learning rate for model parameters $\eta_{\bw_{t}}$ and adversarial input $\eta_{\bx}$, two initialization points $\bw_{1}, \bdelta_{1}$, constant $p\in\{2, \infty\}$ and perturbation size $r$.\\
		\textbf{Return} $\bw_{T + 1}$.
		\begin{algorithmic}[1]
			\FOR {$t=1, \cdots, T$}
			\STATE {Uniformly sample $i_{t}$ from $\{1,\cdots, n\}$.}
			\FOR {$k=1, \cdots, K$}
			\STATE{$\bdelta_{k + 1} = \text{Proj}_{B_{p}(\textbf{0}, r)}\left(\bdelta_{k} \!+\! \eta_{\bx}\nabla_{\bx}f(\bw_{t}, \bx_{i_{t}} \!+\! \bdelta_{k})\right)$.}
			\ENDFOR
			\STATE {$\bw_{t + 1} = \bw_{t} - \eta_{\bw_{t}}\nabla_{\bw}f(\bw_{t}, \bx_{i_{t}} + \bdelta_{K + 1})$.}
			\ENDFOR
		\end{algorithmic}
	\end{algorithm}
	As is justified in Theorems \ref{thm:ood generalization upper bound} and \ref{thm:ood generalization upper bound l2}, the input-robust model can generalize on OOD data. 
	% In this section, 
	Thus we consider
	% 	We aim at 
	training an input-robust model with  the following  objective
	\begin{equation}
	\small
	\begin{aligned}
	& \min_{\bw}\tilde{R}_{P_{n}}(\bw, p)  = \min_{\bw}\frac{1}{n}\sum\limits_{i=1}^{n}\sup_{\|\bdelta\|_{p}\leq r(p)}f(\bw, \bx_{i} + \bdelta)\\
	& \!=\! \min_{\bw}\frac{1}{n}\sum\limits_{i=1}^{n}[\underbrace{f(\bw, \bx_{i})}_{\text{clean acc}} \!+\! \sup_{\|\bdelta\|_{p}\leq r(p)}\underbrace{(f(\bw, \bx_{i} \!+\! \bdelta) \!-\! f(\bw, \bx_{i}))}_{\text{input-robustness}}],	\label{eq:objective}
	\end{aligned}
	\end{equation}
	which is from AT~\citep{madry2018towards}, and can be decomposed into the clean accuracy term and the input-robustness term. 
	We consider $p\in\{2, \infty\}$ as in Section \ref{sec:Robustness Corresponds with Better OOD Generalization}, with $r(2) \!=\! 2r/\epsilon_{0}, r(\infty) \!=\! 2r$ for any given small constant $\epsilon_{0}$. 
	\par
	Besides the general assumptions in Section~\ref{sec:Robustness Corresponds with Better OOD Generalization}, we also use the following mild assumptions in this subsection. 
	\begin{assumption}
		\label{ass:Lip continuous}
		The loss $f(\bw, \bx)$ satisfies the following Lipschitz smoothness conditions
		\begin{equation}
		\small
		\begin{aligned}
		\|\nabla_{\bw}f(\bw_{1}, \bx) - \nabla_{\bw}f(\bw_{2}, \bx)\| & \leq L_{11}\|\bw_{1} - \bw_{2}\|, \\
		\|\nabla_{\bw}f(\bw, \bx_{1}) - \nabla_{\bw}f(\bw, \bx_{2})\| & \leq L_{12}\|\bx_{1} - \bx_{2}\|, \\
		\|\nabla_{\bx}f(\bw_{1}, \bx) - \nabla_{\bx}f(\bw_{2}, \bx)\| & \leq L_{21}\|\bw_{1} - \bw_{2}\|, \\
		\|\nabla_{\bx}f(\bw, \bx_{1}) - \nabla_{\bx}f(\bw, \bx_{2})\| & \leq L_{22}\|\bx_{1} - \bx_{2}\|.
		\end{aligned}
		\end{equation}
	\end{assumption}
	\begin{assumption}
		\label{ass:grad_bound}
		$\|\nabla_{\bw}f(\bw, \bx)\|$ is upper bounded by $G$.   
	\end{assumption}
	\begin{assumption}
		\label{ass:PL inequality}
		For $p\in\{2, \infty\}$, $\tilde{R}_{P_{n}}(\bw, p)$ in \eqref{eq:objective} satisfies the PL-inequality:
		% 		such that 
		\begin{equation}
		\small
		\!\!\frac{1}{2}\|\nabla_{\bw} \tilde{R}_{P_{n}}(\bw, p)\|^{2} \!\geq\! \mu_{\bw}\left(\tilde{R}_{P_{n}}(\bw, p) \!-\! \inf_{\bw}\tilde{R}_{P_{n}}(\bw, p)\right).
		\end{equation}
		For any $\bw$ and training sample $\bx_{i}$, $f(\bw, \bx_{i} + \bdelta)$ is $\mu_{\bx_{i}}$-strongly concave in $\bdelta$ for $\|\bdelta\|_{p} \leq r(p)$:
		% 		, which means
		\begin{equation}
		\small
		f(\bw, \bx_{i} + \bdelta) - f(\bw,\bx_{i}) \leq \langle \nabla_{\bx}f(\bw, \bx_{i}), \bdelta\rangle - \frac{\mu_{\bx_{i}}}{2}\|\bdelta\|^{2},
		\end{equation}
		where $\mu_{\bw}$ and $\mu_{\bx_{i}}$ are constants. 
	\end{assumption}
	Assumptions \ref{ass:Lip continuous} and \ref{ass:grad_bound} are widely used in minimax optimization problems~\citep{nouiehed2019solving,sinha2018certifying}. 
	PL-inequality in  Assumption \ref{ass:PL inequality} means that although $f(\bw, \bx)$ may be non-convex on $\bw$, all the stationary points are global minima. 
	This is observed or proved recently for over-parameterized neural networks \citep{xie2017diversity,du2019gradient,allen2019convergence,liu2020toward}. 
	The local strongly-concavity in Assumption \ref{ass:PL inequality} is reasonable when the perturbation size $\|\bdelta\|_{p}$ is small. 
	\par
	To solve the minimax optimization problem \eqref{eq:objective}, we consider the multi-step stochastic gradient descent (SGD) in Algorithm \ref{alg:sgd}~\citep{nouiehed2019solving}. $\text{Proj}_{A}(\cdot)$ in Algorithm \ref{alg:sgd} is the $\ell_{2}$-projection operator onto 
	% 	set 
	$A$. 
	Note that the update rule of $\bdelta_{k}$ in Algorithm \ref{alg:sgd} is different from that in PGD adversarial training \citep{madry2018towards}, where $\nabla_{\bx}f(\bw_{t}, \bx_{i_{t}} + \bdelta_{k})$ in Line 4 is replaced with the sign of it. 
	\par
	The following theorem gives the convergence rate of Algorithm $\ref{alg:sgd}$ both in expectation and in high probability. 
	\begin{theorem}
		\label{thm:convergence}
		Let $\bw_{t}$ be updated by 
		% 		\hou{For $\bw_{t}$ in }
		Algorithm \ref{alg:sgd}, 
		% 		under Assumptions \ref{ass:Lip continuous}, \ref{ass:grad_bound}, and \ref{ass:PL inequality}, 
		% 		and
		$p\!\in\!\{2,\infty\}$, $\eta_{\bw_{t}}\!=\!\frac{1}{\mu_{\bw}t}$, $\eta_{\bx}\!=\!\frac{1}{L_{22}}$,
		% 		respectively equal to
		% 		$\frac{1}{\mu_{\bw}t}$ and $\frac{1}{L_{22}}$, 
		$K\geq  \frac{L_{22}}{\mu_{\bx}}\log{\left(\frac{8T\mu_{\bw}d_{0}r^{2}(p)}{GL}\right)}$, where $\mu_{\bx} = \min_{1\leq i \leq n}\mu_{\bx_{i}}$ and $L = L_{11} + \frac{L_{12}L_{21}}{\mu_{\bx}}$. 
		Under Assumptions \ref{ass:Lip continuous}, \ref{ass:grad_bound}, and \ref{ass:PL inequality}, 
		we have 
		\begin{equation}
		\small
		\mE[\tilde{R}_{P_{n}}(\bw_{T + 1}, p)] - \tilde{R}_{P_{n}}(\bw^{*}, p) \leq \frac{G^{2}L}{T\mu^{2}_{\bw}},
		\end{equation}
		and with probability at least $1 - \theta$,
		\begin{equation}\label{eq:convergence in probability}
		\small
		\begin{aligned}
		\tilde{R}_{P_{n}}&(\bw_{T + 1}, p) - \tilde{R}_{P_{n}}(\bw^{*}, p) \\
		& \leq \frac{G^{2}\log{(\log{(T/\theta)})}(64L + 16\mu_{\bw}) + G^{2}L}{T\mu_{\bw}^{2}},
		\end{aligned}
		\end{equation}
		for $0< \theta < 1/e$, $T\geq 4$, with $\bw^{*} \in\arg\min_{\bw}\tilde{R}_{P_{n}}(\bw, p)$. 
	\end{theorem} 
	This theorem shows that Algorithm \ref{alg:sgd} is able to find the global minimum of the adversarial objective \eqref{eq:objective} both in expectation and in high probability. 
	Specifically, the convergence rate of Algorithm~\ref{alg:sgd} is $\cO(1/\lceil T/K\rceil) = \cO(K/T) = \tilde{\cO}(1/T)$, since the number of inner loop steps $K$ is $\cO(\log{(Td_0r(p)^2)})$, which increases with the feature dimension of input data $d_{0}$ and the size of perturbation $r$. The proof of Theorem \ref{thm:convergence} is in Appendix  \ref{app:proof in robust training}.
	\par
	The following Proposition~\ref{pro:robustness} (proof is in Appendix \ref{app:proof of proposition robustness}) shows that the model trained by Algorithm \ref{alg:sgd} has a small error on clean training samples, and satisfies the
	% 	robust condition 
	condition of input-robustness 
	in Theorems \ref{thm:ood generalization upper bound} and \ref{thm:ood generalization upper bound l2}. 
	% 	The proof is in Appendix \ref{app:proof of proposition robustness}. 
	% Besides, the model should be robust to 
	% indicating robustness to the adversarial perturbation on training samples. 
	\begin{proposition}
		\label{pro:robustness}
		If $\tilde{R}_{P_{n}}(\bw) \leq \epsilon$ for $\bw$ and a constant $\epsilon$, then $R_{P_{n}}(\bw) \leq \epsilon$, and $f(\bw, \bx)$ is $(r(p), 2\epsilon, P_{n}, p)$-input-robust. 
	\end{proposition}
	\par
	According to Theorem \ref{thm:convergence} and Proposition \ref{pro:robustness}, 
	after $T$ training steps in Algorithm \ref{alg:sgd}, we can obtain a $(r(p), \tilde{\cO}(1/T), P_{n}, p)$-input-robust model when $\tilde{R}_{P_{n}}(\bw^{*})$ is close to zero. 
	Thus, combining Theorems \ref{thm:ood generalization upper bound} and \ref{thm:ood generalization upper bound l2}, we get the following corollary which shows that the adversarially trained model generalizes on OOD data. 
	\begin{corollary}
		\label{cor:excess risk}
		For $p\in\{2, \infty\}$, with the same notations as 
		% 		adopting the parameters in 
		Theorem \ref{thm:ood generalization upper bound} and \ref{thm:convergence}, 
		% 		let $\bw_{t}$ updated by Algorithm \ref{alg:sgd}. 
		if $\tilde{R}_{P_{n}}(\bw^{*}, p)\leq \epsilon_{0}$, then with probability at least $1 - \theta$, 
		\begin{equation*}\label{eq:excess risk bound l2}
		\small
		\begin{aligned}
		& \sup_{P\in B_{\sW_{2}}(P_{0}, r/\epsilon_{0})} R_{P}(\bw_{T + 1}, 2) \leq (2M + 3)\epsilon_{0} \\
		& + (2M\! + \!3)\!\left(\frac{G^{2}\log{(\log{(2T/\theta)})}(64L + 16\mu_{\bw}) + G^{2}L}{T\mu_{\bw}^{2}}\right)\\
		& + \!M\sqrt{\frac{(2d_{0})^{\frac{2\epsilon_{0}^{2}D}{r^{2}} \!+\! 1}\log{2} \!+\! 2\log{(2/\theta)}}{n}}, 
		\end{aligned}
		\end{equation*}
		and 
		\begin{equation*}\label{eq:excess risk bound linf}
		\small
		\begin{aligned}
		\sup_{P\in B_{\sW_{\infty}}(P_{0}, r)} & R_{P}(\bw_{T + 1}, \infty) \leq 3\epsilon_{0} \\
		& + \frac{G^{2}\log{(\log{(2T/\theta)})}(192L + 48\mu_{\bw}) + 3G^{2}L}{T\mu_{\bw}^{2}} \\
		& + M\sqrt{\frac{2d_{0}^{\frac{2D}{r^{2}} + 1}\log{2} + 2\log{(2/\theta)}}{n}},
		\end{aligned}
		\end{equation*}
		for any $0\leq \theta \leq 1/e$ and $T\geq 4$.  
	\end{corollary}
	\par
	This corollary is directly obtained by combining Theorem \ref{thm:ood generalization upper bound}, \ref{thm:ood generalization upper bound l2}, \ref{thm:convergence}, and Proposition \ref{pro:robustness}. 
	It shows that the excess risk (i.e., the terms in the left-hand side of the  above two inequalities) of the adversarially trained model on OOD data is upper bounded by $\tilde{\cO}(1/\sqrt{n} + 1/T)$ after $T$ steps. 
	The dependence of the bounds on hyperparameters like input data dimension $d_{0}$, $\ell_{1}$-diameter $D$ of data support $\cX$ are from the OOD generalization bounds \eqref{eq:ood bound linf}, \eqref{eq:ood bound l2}, and convergence rate \eqref{eq:convergence in probability}.  
	
	\section{Robust Pre-Trained Model has Better Initialization on Downstream Tasks}\label{sec:pretrain improves ood}
	%Pre-trained models are empirically shown to generalize on both CV and NLP downstream tasks \citep{radford2021learning,hendrycks2020pretrained}. 
	The paradigm of ``first pre-train and then fine-tune'' has been widely explored recently \citep{radford2021learning,hendrycks2020pretrained}. 
	In this section, we theoretically show that the input-robust pre-trained model provides an initialization that generalizes on downstream OOD data. 
	\par
	Assume the $m$ i.i.d. samples $\{\bz_{i}\}$ in the pre-training stage are from distribution $Q_{0}$. 
	For a small constant $\epsilon_{\text{pre}}$ and given $r(2) = r/\epsilon_{\text{pre}}, r(\infty) = r$, 
	the following Theorems \ref{thm:pretrain generalize} and \ref{thm:pretrain generalize l2}
	show that the pre-trained model with a small excess risk on OOD data in the  pre-training stage also generalizes on downstream OOD data. The proofs are in Appendix \ref{app:proof of theorem pretrain generalize}.    
	\begin{theorem}
		\label{thm:pretrain generalize}
		If $\sup_{Q\in B_{\sW_{\infty}}(Q_{0}, r(\infty))}R_{Q}(\bw_{\emph{\text{pre}}})\leq \epsilon_{\emph{\text{pre}}}$, then 
		\begin{equation}\label{eq:initialized error linf}
		\small
		\sup_{P\in B_{\sW_{\infty}}(P_{0}, r(\infty))}R_{P}(\bw_{\emph{\text{pre}}}) \leq \epsilon_{\emph{\text{pre}}} + 2M\TV(P_{0}, Q_{0}),	
		\end{equation}
		and with probability at least $1 - \theta$,
		\begin{equation}
		\small
		\tilde{R}_{P_{n}}(\bw_{\emph{\text{pre}}}, \infty) \leq \epsilon_{\emph{\text{pre}}}\! + \!2M\TV(P_{0}, Q_{0}) +  M\sqrt{\frac{\log{(1/\theta)}}{2n}}.
		\end{equation} 
	\end{theorem}
	\begin{theorem}
		\label{thm:pretrain generalize l2}
		If $\sup_{Q\in B_{\sW_{2}}(Q_{0}, r_{0})}R_{Q}(\bw_{\emph{\text{pre}}})\leq \epsilon_{\emph{\text{pre}}}$ with $r_{0}= \sqrt{2D^{2}\TV(P_{0}, Q_{0}) + r(2)^{2}}$, then 
		\begin{equation}\label{eq:initialized error l2}
		\small
		\sup_{P\in B_{\sW_{2}}(P_{0}, r(2))}R_{P}(\bw_{\emph{\text{pre}}}) \leq \epsilon_{\emph{\text{pre}}} + 2M\TV(P_{0}, Q_{0}).	
		\end{equation}
	\end{theorem}
	\par
	\begin{remark}
		The self-supervised pre-training (e.g., masked language modeling in BERT \citep{devlin2019bert}) can also be included into the $f(\bw, \bx)$ in Section \ref{sec:Robustness Corresponds with Better OOD Generalization}, 
		% 		this framework
		if we take label $y\sim P_{y \mid \bx}$ as the distribution of the artificially constructed labels (e.g., masked tokens in BERT).
	\end{remark}
	When we implement fine-tuning on downstream tasks, the model is initialized by $\bw_{\text{pre}}$. 
	Combining the results in 
	Theorems \ref{thm:ood generalization upper bound} and \ref{thm:ood generalization upper bound l2}
	(an input-robust model has small OOD generalization error) 
	with Theorems \ref{thm:pretrain generalize} and \ref{thm:pretrain generalize l2}, 
	we conclude that the input-robust model 
	has small excess risk on the OOD data in the pre-training stage, and thus 
	generalizes on the OOD data of downstream tasks. Specifically, \eqref{eq:initialized error linf} and \eqref{eq:initialized error l2} show that 
	the initial OOD excess risk in the fine-tuning stage $\sup_{P\in B_{\sW_{p}}(P_{0}, r(p))}R_{P}(\bw_{\text{pre}})$ is decided by terminal OOD excess risk in pre-training stage $\sup_{Q\in B_{\sW_{p}}(Q_{0},  r(p))}R_{Q}(\bw_{\text{pre}})$ and the total variation distance $\TV(P_{0}, Q_{0})$. 
	% 	(see \eqref{eq:tv}). 
	The intuition is that if $\bw_{\text{pre}}$ generalizes well on distributions around $Q_{0}$, and $P_{0}$ is close to $Q_{0}$ under the total variation distance, then $\bw_{\text{pre}}$ generalizes on downstream OOD data.   
	\par
	To satisfy the condition $\sup_{Q\in B_{\sW_{p}}(Q_{0}, r(p))}R_{Q}(\bw_{\text{pre}})\leq \epsilon_{\text{pre}}$ in Theorems \ref{thm:pretrain generalize} and \ref{thm:pretrain generalize l2}, we can use adversarial pre-training. 
	Corollary \ref{cor:excess risk} implies $\epsilon_{\text{pre}}=\cO(1/\sqrt{m})$ by implementing sufficient adversarial pre-training. Thus, massive training samples $m$ in the adversarial pre-training stage improves the OOD generalization on downstream tasks as $\epsilon_{\text{pre}}=\cO(1/\sqrt{m})$ appears in the bounds \eqref{eq:initialized error linf} and \eqref{eq:initialized error l2}. 
	\par
	\citet{radford2021learning,hendrycks2020pretrained} empirically verify that the standardly pre-trained model also generalizes well on downstream OOD data. It was shown that sufficient standard training by gradient-based algorithm can also find the most input-robust model under some mild conditions \citep{soudry2018implicit,lyu2019gradient}. Thus, $\sup_{Q\in B_{\sW_{\infty}}(Q_{0}, r(p))}R_{Q}(\bw_{\text{pre}})\leq \epsilon_{\text{pre}}$ can hold even for standardly pre-trained model. However, the convergence to the most input-robust model of standard training is much slower compared with AT, e.g., for 
	% 	, e.g., $\cO(1/\log{T})$ 
	linear model \citep{soudry2018implicit,li2019inductive}. %while $\cO(1/\sqrt{T})$ for adversarially trained linear model \citep{li2019inductive}. 
	Hence, to efficiently learn an input-robust model in the pre-training stage, we suggest adversarial pre-training. 
	
	\section{Experiments}
	\begin{table*}[t!]
		% 		\vspace{-0.1in}
		\caption{Clean and corruption accuracy (\%) of ResNet34 on \texttt{CIFAR10-C} and \texttt{ImageNet-C} using standard training and adversarial training under both $\ell_{2}$-norm and $\ell_{\infty}$-norm.}
		\label{tbl:adversarial training on image}
		%		\vspace{-0.1in}
		\centering
		\scalebox{0.645}{
			{
				\begin{tabular}{l|c|c|ccc|cccc|cccc|cccc|c}
					\hline
					%				\multicolumn{6}{c}{CIFAR-10}    &                         &                        &  \\ \hline
					%				dataset                         &          Model          &      Baseline(DA)      & Baseline(DA+Long) & Baseline(DA+UNI) &     MMEL-H     &     MMEL-S     &                &  \\ \hline
					\multirow{2}{*}{Dataset}             & \multirow{2}{*}{Method} & \multirow{2}{*}{Clean} &              \multicolumn{3}{c|}{Noise}               &                     \multicolumn{4}{c|}{Blur}                     &                   \multicolumn{4}{c|}{Weather}                    & \multicolumn{4}{|c|}{Digital}                                     & \multirow{2}{*}{Avg.} \\
					&                         &                        &       Gauss       &       Shot       &    Impulse     &    Defocus     &     Glass      &     Motion     &      Zoom      &      Snow      &     Frost      &      Fog       &     Bright     &    Contrast    &    Elastic     &     Pixel      &      JPEG      &  \\ \hline
					\multirow{3}{*}{\texttt{CIFAR10-C}}  &           Std           &         94.82          &       34.75       &      40.43       &     25.45      &     59.85      &     48.95      &     67.58      &     63.85      &     73.31      &     62.87      & \textbf{67.03} & \textbf{90.69} & \textbf{36.83} &     76.00      &     42.89      &     75.84      &        57.75         \\
					&     Adv-$\ell_{2}$      &    94.93     &       70.39       &      74.24       &     45.17      &     72.77      &     71.34      &     73.51      &     80.26      &     83.28      &     81.36      &     51.08      &     89.37      &     19.49      &     83.39      &     79.78      & \textbf{89.52} &        71.00         \\
					&   Adv-$\ell_{\infty}$   &         93.48          &  \textbf{80.18}   &  \textbf{80.80}  & \textbf{62.73} & \textbf{77.71} & \textbf{77.10} & \textbf{75.46} & \textbf{82.47} & \textbf{83.45} & \textbf{82.32} &     41.00      &     88.15      &     16.10      & \textbf{83.82} & \textbf{85.98} &     89.36      &    \textbf{73.78}    \\ \hline
					\multirow{3}{*}{\texttt{ImageNet-C}} &           Std           &     74.01     &       18.97       &      18.39       &     12.98      &      6.32      &      9.76      &     11.49      &      9.37      &      8.78      &     12.98      &      6.21      &     33.74      &      4.31      &     18.29      &     23.91      &     29.08      &        14.97         \\
					&     Adv-$\ell_{2}$      &         73.66          &  \textbf{30.13}   &  \textbf{28.93}  & \textbf{25.05} & \textbf{32.91} &     25.61      & \textbf{34.50} &     32.84      & \textbf{27.39} & \textbf{33.82} & \textbf{36.52} & \textbf{62.18} & \textbf{31.73} &     42.91      &     47.86      &     51.55      &    \textbf{36.26}    \\
					&   Adv-$\ell_{\infty}$   &         68.36          &       25.94       &      25.61       &     21.17      &     24.56      & \textbf{32.81} &     32.20      & \textbf{34.57} &     26.70      &     33.47      &     11.22      &     56.07      &     12.34      & \textbf{47.67} & \textbf{57.32} & \textbf{59.10} &        33.38         \\ \hline
		\end{tabular}}}
	\end{table*}
	
	% 	In this section, we conduct experiments to verify our theoretical findings in the previous sections.
	% of the connection between OOD generalization and adversarial training and pre-training in the previous sections. 
	% that OOD generalization can be improved by
	% adversarial training (Corollary \ref{cor:excess risk}) and adversarial  pre-training (Theorems \ref{thm:pretrain generalize}-\ref{thm:pretrain generalize l2}).
	\subsection{Adversarial Training Improves OOD Generalization}\label{sec:at improves ood}
	In this section, we verify our conclusion in Section \ref{sec:Learning Robust Model Results in Better OOD Generalization} that OOD generalization can be improved by AT (Corollary \ref{cor:excess risk}). 
	%	Experiments are performed on image classification and natural language understanding tasks. 
	
	\subsubsection{Experiments on Image Classification}
	\label{sec:Experiments on Image Classification}
	\paragraph{Data.} We use the following benchmark datasets.
	\begin{itemize}
		\item \texttt{CIFAR10} \citep{krizhevsky2009learning} has 50000 colorful images as training samples from 10 object classes. \texttt{CIFAR10-C} simulates OOD colorful images with 15 types of common visual corruptions, which serves as a benchmark to verify the OOD generalization of model trained on \texttt{CIFAR10}. Each type of corruption has five levels of severity, and each severity has 10000 validation samples. The 15 types of corruptions are divided into 4 groups: Noise, Blur, Weather and Digital.
		\item \texttt{ImageNet} \citep{deng2009imagenet} contains colorful images with over 1 million training samples from 1,000 categories. Similar to \texttt{CIFAR10-C}, \texttt{ImageNet-C} serves as a benchmark of OOD data with 15 types of corruptions. Each type of corruption has five levels of severity with 50000 validation samples in it. A visualization of \texttt{ImageNet-C} is in Figure \ref{fig:imagenet-c} in Appendix.      
	\end{itemize}
	
	\paragraph{Setup.} 
	The model used in this subsection is ResNet34 \citep{he2016deep}. To verify  that adversarial training helps improve OOD performance, we conduct Algorithm \ref{alg:sgd} on \texttt{CIFAR10}, \texttt{ImageNet} and evaluate the model on \texttt{CIFAR10-C} and \texttt{ImageNet-C}, respectively. The number of inner loop steps $K$  is 8  for \texttt{CIFAR10}, and 3 for \texttt{ImageNet}. The models are trained by SGD with momentum. 
	The number of training  epochs is 200 for \texttt{CIFAR10}, and 100  for \texttt{ImageNet}. 
	The learning rate starts from 0.1 and decays by a factor 0.2 at epochs 60, 120, 160 (resp. 30, 60, 90) for \texttt{CIFAR10} (resp. \texttt{ImageNet}). 
	Detailed hyperparameters are in Appendix \ref{app:hyp on adv}.
	\par
	We compare adversarial training under  $\ell_{2}$- and $\ell_{\infty}$-norm (respectively abbreviated as ``Adv-$\ell_{2}$'' and ``Adv-$\ell_{\infty}$'') against standard training (abbreviated as ``Std''). 
	For Adv-$\ell_{\infty}$, we replace $\nabla_{\bx}f(\bw_{t}, \bx_{i_{t}} + \bdelta_{k})$ in Line 4 of  Algorithm~\ref{alg:sgd} with the sign of it as in \citep{madry2018towards},
	% 	\yi{for Adv-$\ell_{\infty}$ (i.e., PGD adversarial training) can significantly improve the OOD generalization. 
	% 	We suggest this is because PGD adversarial training finds the adversarial perturbation (the inner loop in Algorithm \ref{alg:sgd}) more efficiently 
	% 	as \citep{goodfellow2015explaning}, 
	in order to find stronger adversarial perturbation~\citep{goodfellow2015explaning}.
	% 	} 
	% 	Thus the Adv-$\ell_{\infty}$ is conducted in this scenario. 
	%\footnote{so the experiment is not cosnistent with the theory, five explanations.} 
	% 	An ablation study on the effect of robust radius\footnote{is this a formal term?} $r$ in AT can be found in Appendix \ref{app:ablation}. 
	%	We also provide an ablation study to the factors robust radius $r$ and number of training samples related to OOD generalization appear in Theorem \ref{thm:ood generalization upper bound} and \ref{thm:ood generalization upper bound l2}. The results are in Appendix  
	\paragraph{Main Results.} 
	In Table \ref{tbl:adversarial training on image}, for each type of corruption, we report the test accuracy on \texttt{CIFAR10-C} under the strongest corruption severity level 5\footnote{Lighter severity levels exhibit similar trends but with smaller performance gaps between adversarial and standard training}. 
	For \texttt{ImageNet-C}, we report the average test accuracy of five severity levels as in \citep{hendrycks2018benchmarking}. 
	We also report the test accuracy  on \texttt{CIFAR10} and \texttt{ImageNet} in the column of ``Clean'' for comparison. 
	% 	The results are shown in Table \ref{tbl:adversarial training on image}. 
	
	As can be seen, Adv-$\ell_{2}$ and Adv-$\ell_{\infty}$  improve the average accuracy on OOD data, especially under corruption types Noise and Blur. This supports our finding in Section~\ref{sec:Learning Robust Model Results in Better OOD Generalization} that AT makes the model generalize on OOD data. 
	% 	For \texttt{CIFAR10-C}, Adv-$\ell_{2}$ has comparable performance with Adv-$\ell_{\infty}$ on OOD data while better performance on clean data. For \texttt{ImageNet-C}, Adv-$\ell_{2}$ consistently beats Adv-$\ell_{\infty}$ on both clean data and corrupted OOD data. \footnote{any explanation for this? If not, remove is better?}
	% 	\par
	Though AT improves the OOD generalization on all corruption types for \texttt{ImageNet-C}, it degenerates the performance
	for data corrupted under types Fog, Bright and Contrast in \texttt{CIFAR10-C}. 
	We speculate  this is because these three corruptions intrinsically rescale the adversarial  perturbation,
	% For data corrupted under the category Fog Bright and Contrast, adversarial training degenerates their performance in Table \ref{tbl:adversarial training on image}. 
	% We speculate this is because the three corruptions rescale the input pixel values to smaller values and the same perturbation size $r$ leads to relatively large perturbation. 
	and refer readers  to Appendix \ref{app:perturbation size} for a detailed discussion.
	
	% We speculate this may because the three shifted distributions for \texttt{CIFAR10-C} are far away from the training distribution (as illustrated in Figure \ref{fig:imagenet-c} in Appendix). In this cases, our OOD generalization based on distributional distance may be invalid.
	% 	, while the definition is suitable for most cases. 
	
	\paragraph{Ablation Study.}
	We study the effect of perturbation size $r$ and the number of training samples $n$ for adversarial training
	% 	as they appear
	in bounds \eqref{eq:ood bound linf} and \eqref{eq:ood bound l2}. 
	Due to the space limit,  we put the implementation details and results in Appendix \ref{app:perturbation}.
	\par
	The results for the effect of perturbation size $r$ are in  Figures~\ref{fig:adv_l2_r}-\ref{fig:adv_linf_r} in Appendix \ref{app:perturbation size}. 
	As can be seen, the accuracy on OOD data \texttt{CIFAR10-C} first increases and then decreases with an increasing $r$. 
	This is because the upper bounds of excess risk in \eqref{eq:ood bound linf} and \eqref{eq:ood bound l2} are decided by both the clean accuracy and input-robustness.
	However, an increasing perturbation size $r$ improves the input-robustness, but harms the clean accuracy~\citep{raghunathan2019adversarial}. 
	Specifically, when the perturbation size $r$ is small, the clean accuracy is relatively stable and the robustness dominates. 
	Thus the overall OOD performance increases as $r$ increases. 
	However, when $r$ is relatively large, a larger $r$ leads to worse clean accuracy though better robustness, and can lead to  worse overall OOD performance. 
	Thus, to achieve the optimal performance on OOD data, we should properly choose the perturbation size $r$ rather than continually increasing it. 
	\par
	The results for the effect of the number of training samples $n$ are in  Figures~\ref{fig:adv_l2_num}-\ref{fig:adv_linf_num} in Appendix \ref{app:number of training samples}. 
	The accuracy on OOD data increases with the number of training samples, which is consistent with our findings in Theorems \ref{thm:ood generalization upper bound} and \ref{thm:ood generalization upper bound l2}.  
	
	\subsubsection{Experiments on Natural Language Understanding}\label{sec:Experiments on Natural Language Understanding}
	\begin{table}[t!]
		% 	\vspace{-0.1in}
		\caption{Performance of  $\text{BERT}$ base model on NLU tasks using standard training and adversarial training under both $\ell_{2}$-norm and $\ell_{\infty}$-norm.}
		% 		Performance of $\text{BERT}$ model on NLU tasks.}
		\label{tbl:adversarial training on text}
		%		\vspace{-0.1in}
		\centering
		\scalebox{0.6}{
			{
				\begin{tabular}{c|cc|ccc}
					\hline
					%				\multicolumn{6}{c}{CIFAR-10}           &           &  \\ \hline
					%				dataset                               &   Model   & Baseline(DA) & Baseline(DA+Long) & Baseline(DA+UNI) &       MMEL-H        & MMEL-S &  \\ \hline
					Dataset     		  &    Train                    &    Test      &     Std     &    Adv-$\ell_{2}$   &  Adv-$\ell_{\infty}$ \\
					\hline
					\multirow{8}{*}{\texttt{STS-B}}       &    \multirow{2}{*}{Images}     &   Images     &    98.38    &      97.81       &         96.39     \\
					&                                &   MSRvid     &    89.52(-8.86)    &      \textbf{90.61}(-7.20)       &         90.09(-6.30)      \\
					\cline{2-6}
					&    \multirow{2}{*}{MSRvid}     &   MSRvid     &    98.55            &      97.45               &         96.65     \\
					&                                &   Images     &    \textbf{84.12}(-14.43)    &      83.63(-13.82)       &         83.11(-13.54)     \\
					\cline{2-6}
					&    \multirow{2}{*}{Headlines}  &   Headlines  &   97.59             &     96.73                &    95.75          \\
					&                                &   MSRpar     &   62.07(-35.52)     &     64.48(-32.25)        &    \textbf{67.67}(-28.08)           \\
					\cline{2-6}
					&    \multirow{2}{*}{MSRpar}     &   MSRpar     &   97.55             &     97.33                &    97.55               \\
					&                                &   Headlines  &   75.58(-21.97)     &     75.27(-22.06)        &    \textbf{76.12}(-21.43)        \\
					\hline
					\multirow{4}{*}{\texttt{SST-2}; \texttt{IMDb}}  &    \multirow{2}{*}{SST-2}&   SST-2      &   93.57     &      93.57       &  93.92             \\
					&                                &   IMDb       &   90.06(-3.51)     &      \textbf{91.50}(-2.07)       &  91.32(-2.60)                 \\
					\cline{2-6}
					&    \multirow{2}{*}{IMDb}       &   IMDb       &   94.36            &       94.88             &  94.68            \\
					&                                &   SST-2      &   87.00(-7.36)     &       \textbf{88.53}(-6.35)      &  88.07(-6.61)             \\
					\hline
					\multirow{3}{*}{\texttt{MNLI}}          & \multirow{3}{*}{Telephone}     &   Telephone  &  83.01           &     83.16       &  82.90 \\ 
					&                                &   Letters    &  82.45(-0.56)      &     83.76(+0.60)        &  \textbf{84.07}(+1.17) \\                        &                                & Face-to-face &  81.56(-1.45)      &     \textbf{83.59}(+0.43)        &  \textbf{83.59}(+0.69) \\
					\hline 
		\end{tabular}}}
	\end{table}
	
	\paragraph{Data.}
	As in \citep{hendrycks2020pretrained}, we use three pairs of datasets as the original and OOD datasets for NLU tasks. 
	% \footnote{rewrite the following, too many descriptions are the same as \cite{hendrycks2020pretrained} especially for MNLI. }
	\begin{itemize}
		\item \texttt{SST-2} \citep{socher2013recursive} and \texttt{IMDb} \citep{maas2011learning} are sentiment analysis datasets,
		% 		which consist of 
		with pithy expert and full-length lay movie reviews, respectively. 
		As in \citep{hendrycks2020pretrained}, we train on one dataset and evaluate on the other.
		% 		, and vice versa. 
		Then we report the accuracy of a review's binary sentiment predicted by the model. 
		\item \texttt{STS-B} consists of texts from different genres and sources. 
		It requires the model to predict the textual similarity between pairs of sentences \citep{cer2017semeval}. 
		As in \citep{hendrycks2020pretrained}, we use four sources from two genres: MSRpar(news), Headlines (news); MSRvid(captions), Images(captions). The evaluation metric is Pearson's correlation coefficient. 
		\item \texttt{MNLI} is a textual entailment dataset which contains sentence pairs from different genres of text \citep{williams2018broad}. 
		We select training samples from two genres of transcribed text (Telephone and Face-to-Face) and the other of written text (Letters) as in \citep{hendrycks2020pretrained}, and report the classification accuracy.
	\end{itemize}
	
	\paragraph{Setup.} For a pre-trained  language model e.g., BERT, 
	each input token
	% 	$\bx$ 
	is encoded as a one-hot vector and then mapped into a continuous embedding space. 
	Instead of adding perturbations to the  one-hot vectors, 
	we construct adversarial samples in the word embedding space as in \citep{zhu2019freelb}.    
	\par
	The backbone model is the base version of $\text{BERT}$ \citep{devlin2019bert} which has been widely used in the NLP community. 
	We conduct AT in the fine-tuning stage to see its effectiveness on OOD generalization. 
	The models are trained by AdamW \citep{loshchilov2018decoupled} for 10 epochs. 
	Detailed hyperparameters are in Appendix \ref{app:hyp on adv}. 
	As in Section \ref{sec:Experiments on Image Classification}, 
	we compare Adv-$\ell_{2}$ and Adv-$\ell_{\infty}$ with Std.
	\paragraph{Main Results.}  
	In Table \ref{tbl:adversarial training on text},
	we report the results on in-distribution data and OOD data, and the gap between them (in the brackets) as in \citep{hendrycks2020pretrained}. The gaps in brackets are used to alleviate the interference by the general benefits from AT itself, since it was shown in \citep{zhu2019freelb} that AT can improve the generalization ability of model on in-distribution textual data. 
	\par
	As can be seen, adversarially trained models perform similarly or even better than standardly trained models on in-distribution data, while significantly better on OOD data especially for \texttt{MNLI}. The smaller gaps between in-distribution and OOD data support our finding that AT can be used to improve OOD generalization. 
	\subsection{Robust Pre-Trained Model Improves OOD Generalization}
	\label{expt:pretrain}
	Previously in Section \ref{sec:pretrain improves ood}, we theoretically show that 
	% 	initializing the downstream model with 
	an input-robust pre-trained model gives a better initialization for fine-tuning on downstream task, in terms of OOD generalization. 
	% 	at initialization.
	In this section, we empirically show that this better initialization also leads to better OOD generalization after finetuning on image classification tasks.
	% 	in Section \ref{sec:pretrain improves ood}.
	% about the connection between pre-training and OOD generalization. 
	\begin{table*}[htbp!]
		% 	\vspace{-0.1in}
		\caption{Clean and corruption accuracy (\%) of ResNet34 on \texttt{CIFAR10-C} with no pre-training, standard pre-training, and adversarial pre-training under $\ell_{2}$-norm and $\ell_{\infty}$-norm.}
		\label{tbl:adversarial pre-training}
		%		\vspace{-0.1in}
		\centering
		\scalebox{0.64}{
			{
				\begin{tabular}{l|c|c|ccc|cccc|cccc|cccc|c}
					\hline
					%				\multicolumn{6}{c}{CIFAR-10}    &                               &                        &                   &                  &                &  \\ \hline
					%				dataset                         &             Model             &      Baseline(DA)      & Baseline(DA+Long) & Baseline(DA+UNI) &     MMEL-H     &     MMEL-S     &                &                &                &                &  \\ \hline
					\multirow{2}{*}{Fine-Tuning}         & \multirow{2}{*}{Pre-Training} & \multirow{2}{*}{Clean} &              \multicolumn{3}{c|}{Noise}               &                     \multicolumn{4}{c|}{Blur}                     &                   \multicolumn{4}{c|}{Weather}                    & \multicolumn{4}{|c|}{Digital}                                     & \multirow{2}{*}{Avg.} \\
					&                               &                        &       Gauss       &       Shot       &    Impulse     &    Defocus     &     Glass      &     Motion     &      Zoom      &      Snow      &     Frost      &      Fog       &     Bright     &    Contrast    &    Elastic     &     Pixel      &      JPEG      &                      \\ \hline
					\multirow{4}{*}{Std}                 &              No               &     95.21     &       40.55       &      40.64       &     19.91      &     83.21      &     67.77      &     77.86      &     90.31      &     80.71      &     77.91      &     67.27      & \textbf{90.88} &     48.14      &     80.80      &     81.99      &     80.84      &        68.59         \\
					&              Std              &         94.65          &       41.25       &      42.91       &     22.58      &     85.19      &     71.03      &     78.49      &     \textbf{90.82}      &     82.78      & \textbf{80.04} &     67.66      &     89.97      &     45.70      & \textbf{83.89} &     82.03      & \textbf{80.99} &        69.69         \\
					&          Adv-$\ell_{2}$          &         95.06          &  \textbf{45.10}   &  \textbf{50.58}  &     27.57      &     87.27      & \textbf{72.95} &     79.08      &     90.57      & \textbf{83.29} &     77.25      &     65.41      &     90.15      & \textbf{50.41} &     82.81      &     78.01      &     78.95      &        70.63         \\
					&       Adv-$\ell_{\infty}$        &         94.30          &       40.94       &      46.42       & \textbf{29.39} & \textbf{87.60} &     70.79      & \textbf{81.44} & 90.69 &     82.77      &     79.28      & \textbf{68.84} &     89.19      &     45.29      &     83.59      & \textbf{83.13} &     80.86      &    \textbf{70.68}    \\ \hline
					\multirow{4}{*}{Adv-$l_{2}$}         &              No               &         94.43          &       56.82       &      60.58       &     29.34      &     85.44      &     71.67      &     81.80      &     90.08      &     83.68      &     80.37      &     61.68      &     89.96      &     34.76      &     83.76      &     85.16      &     83.24      &        71.89         \\
					&              Std              &         94.09          &       57.64       &      60.96       &     26.35      &     86.78      & \textbf{73.52} &     82.16      &     90.46      &     82.12      &     80.64      &     62.58      &     88.98      &     34.68      &     84.29      &     83.42      &     83.42      &        71.87         \\
					&        Adv-$\ell_{2}$         &         94.45          &  \textbf{58.98}   &  \textbf{62.99}  & \textbf{35.08} &     87.07      &     72.29      &     81.66      &     91.07      &     83.53      &     81.38      &     62.82      &     89.52      & \textbf{39.53} &     84.35      & \textbf{86.60} &     88.55      &        73.69         \\
					&      Adv-$\ell_{\infty}$      &     95.25     &       58.64       &      62.18       &     29.86      & \textbf{88.15} &     73.00      & \textbf{82.95} & \textbf{91.98} & \textbf{84.76} & \textbf{83.86} & \textbf{64.76} & \textbf{91.00} &     37.35      & \textbf{84.65} &     86.57      & \textbf{88.59} &    \textbf{73.89}    \\ \hline
					\multirow{4}{*}{Adv-$\ell_{\infty}$} &              No               &         92.46          &       80.91       &      81.69       &     52.00      &     79.58      &     80.94      &     77.42      &     80.21      &     80.57      &     79.35      &     35.41      &     83.15      &     18.06      &     83.51      &     87.79      &     87.44      &        72.54         \\
					&              Std              &         92.05          &       80.21       &      81.06       & \textbf{63.02} &     77.94      &     77.80      &     75.60      &     80.04      & \textbf{83.77} &     81.22      &     41.57      & \textbf{89.94} & \textbf{19.04} &     82.39      &     85.49      & \textbf{88.76} &        73.86         \\
					&        Adv-$\ell_{2}$         &     92.55     &  \textbf{81.96}   &  \textbf{82.86}  &     58.95      & \textbf{80.51} & \textbf{82.66} & \textbf{78.21} & \textbf{86.56} &     81.49      &     81.10      &     42.07      &     89.76      &     18.56      & \textbf{84.58} & \textbf{88.53} &     88.05      &    \textbf{75.06}    \\
					&      Adv-$\ell_{\infty}$      &         92.28          &       81.74       &      82.37       &     56.96      &     80.34      &     81.90      &     77.94      &     85.76      &     81.48      & \textbf{81.70} & \textbf{42.99} &     89.00      &     18.45      &     84.50      &     88.07      &     87.50      &        74.71         \\ \hline
		\end{tabular}}}
	\end{table*}

	\paragraph{Setup.} 
	Following \citep{salman2020adversarially}, we pre-train the model on \texttt{ImageNet} and then fine-tune it on \texttt{CIFAR10}.
	To get an input-robust model in the pre-training stage, we consider adversarially pre-train the model. 
	% 	To see the effect brought by input-robustness, we respectively 
	We compare adversarial pre-training (Adv-$\ell_{2}$ and Adv-$\ell_{\infty}$) against standard  pre-training and no pre-training
	% 	We conduct  Adv-$\ell_{2}$, Adv-$\ell_{\infty}$ and Std 
	as in Section \ref{sec:Experiments on Image Classification}.
	In the fine-tuning stage, the data from \texttt{CIFAR10} are resized to $224\times224$ as in \citep{salman2020adversarially}. We also compare stadard fine-tuning and adversarial fine-tuning under both $\ell_{2}$- and $\ell_{\infty}$-norm. After fine-tuning, we verify the OOD generalization on \texttt{CIFAR10-C}. The other settings are the same as Section \ref{sec:Experiments on Image Classification}.
	
	\paragraph{Main Results.} 
	% 	Our theoretical results in Section \ref{sec:pretrain improves ood} show that a more input-robust  pre-trained model also provides a better initialization for generalization on downstream OOD data. 
	% 	initialization that generalizes well on downstream OOD data. 
	% 	We verify whether this better-initialization also performs better 
	% 	the OOD generalization of the fine-tuned model. 
	% 	after fine-tuning. 
	The results 
	% 	of the pre-trained model on \texttt{CIFAR10-C} 
	are shown in Table \ref{tbl:adversarial pre-training}.   
	As can be seen, 
	for all  fine-tuning methods, 
	adversarially pre-trained models consistently achieve better performance on OOD data 
	than standardly pre-trained models or models without pre-training. 
	Thus, the initialization from the adversarially pre-trained input-robust model leads to
	better OOD generalization on downstream tasks after fine-tuning. In addition, standard pre-training slightly improves the OOD generalization compared with no pre-training
	when we conduct Adv-$\ell_{\infty}$ fine-tuning or standard fine-tuning.
	% 	\hou{This is also consistent with the finding in Section \ref{sec:Learning Robust Model Results in Better OOD Generalization}.}
	We also observe that for all four kinds of pre-training, 
	adversarial fine-tuning under $\ell_{\infty}$-norm has better performance than $\ell_{2}$-norm.
	This agrees with the observations in Section~\ref{sec:Experiments on Image Classification}. 
	Note that the results of 
	models without pre-training are different from those in Table \ref{tbl:adversarial training on image} due to the resized input data.
	
	\subsection{Discussion}
	\label{sec:discussion}
	
	% 	: Pre-Training on Large Corpus Improve Input-Robustness}
	It is shown in \citep{hendrycks2020pretrained}  that the language model BERT~\citep{devlin2019bert} pre-trained on large corpus generalizes well  on downstream OOD data, and RoBERTa~\citep{liu2019roberta} pre-trained with more training data and updates generalizes even better than BERT.
	% 	\yi{We verify 
	We speculate this is because (i) sufficient pre-training obtains an input-robust model as discussed in Section \ref{sec:pretrain improves ood}, and this better-initialization leads to better OOD generalization after finetuning as observed in Section~\ref{expt:pretrain}; and (ii) the objective of masked language modeling predicts the masked (perturbed) input tokens and enables a certain amount of input-robustness.
	
	In this section, we empirically show 
	that the model initialized by BERT has higher input-robustness than a randomly initialized model. 
	Besides, compared with BERT, RoBERTa is pre-trained with more training samples and updating steps
	% 	in the pre-training stage 
	and the model initialized by it is more robust to input perturbations. 
	% which also explains their better OOD generalization on downstream task in \citep{hendrycks2020pretrained}.

	% 	\paragraph{Data.} 
	% 	\texttt{MRPC} and \texttt{CoLA} are datasets in the GLUE \citep{wang2018glue} with 3.7k and 8.5k training samples. 
	\paragraph{Setup.} We compare the input-robustness of the base versions of pre-trained language model BERT
	% 	$\text{BERT}_\text{BASE}$ 
	\citep{devlin2019bert} and  RoBERTa
	% 	$\text{RoBERTa}_{\text{BASE}}$ 
	\citep{liu2019roberta},
	against a randomly initialized model whose 
	% 	every entries of 
	parameters 
	% 	$\bw$ 
	are independently sampled from $\cN(0, 0.02^{2})$ \citep{wolf2020transformers}.  
	The three models have exactly the same structure.
	Compared with $\text{BERT}$, $\text{RoBERTa}$ is pre-trained on a larger corpus for more updating steps.  
	Experiments are performed on \texttt{MRPC} and \texttt{CoLA}  datasets from the GLUE  benchmark \citep{wang2018glue},
	with 3.7k and 8.5k training samples, respectively. 
	Similar as Section \ref{sec:Experiments on Natural Language Understanding}, 
	we add adversarial perturbations in the embedding space.
	We use 3 steps of $\ell_{\infty}$-norm attack to construct perturbation.
	The perturbation size is 0.001 and the perturbation step size 0.0005.  
	Since the the last classification layer of $\text{BERT}$ or $\text{RoBERTa}$ is randomly initialized  during downstream task fine-tuning, 
	we study the difference in the hidden states of the last Transformer layer 
	before the classification layer. 
	Denote $\textbf{h}, \textbf{h}_{\text{per}}\in \bbR^{128\times 768}$ as the hidden states from the original input and the adversarially perturbed input, respectively. We use the $\ell_{2}$-norm $\|\textbf{h}_{\text{per}} - \textbf{h}\|$ and the cosine
	similarity $\langle\textbf{h}, \textbf{h}_{\text{per}}\rangle/(\|\textbf{h}\|\|\textbf{h}_{\text{per}}\|)$ to measure the difference. The cosine similarity is used to alleviate the potential interference caused by the scale of $\textbf{h}$ over different pre-trained models. The results are in Figure \ref{fig:comp}.
	\begin{figure}[htbp]	
		%		\vspace{-0.05in}
		\centering
		\subfloat{
			\includegraphics[width=0.3\textwidth]{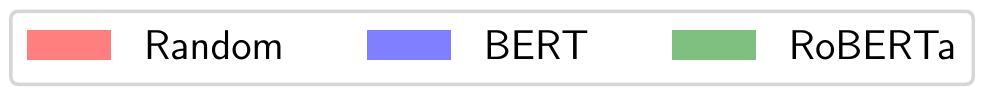}}
		\vspace{-0.1in}
		\\
		\addtocounter{subfigure}{-1}
		\subfloat[\texttt{MRPC.}\label{fig:mrpc_norm}]{
			\includegraphics[width=0.24\textwidth]{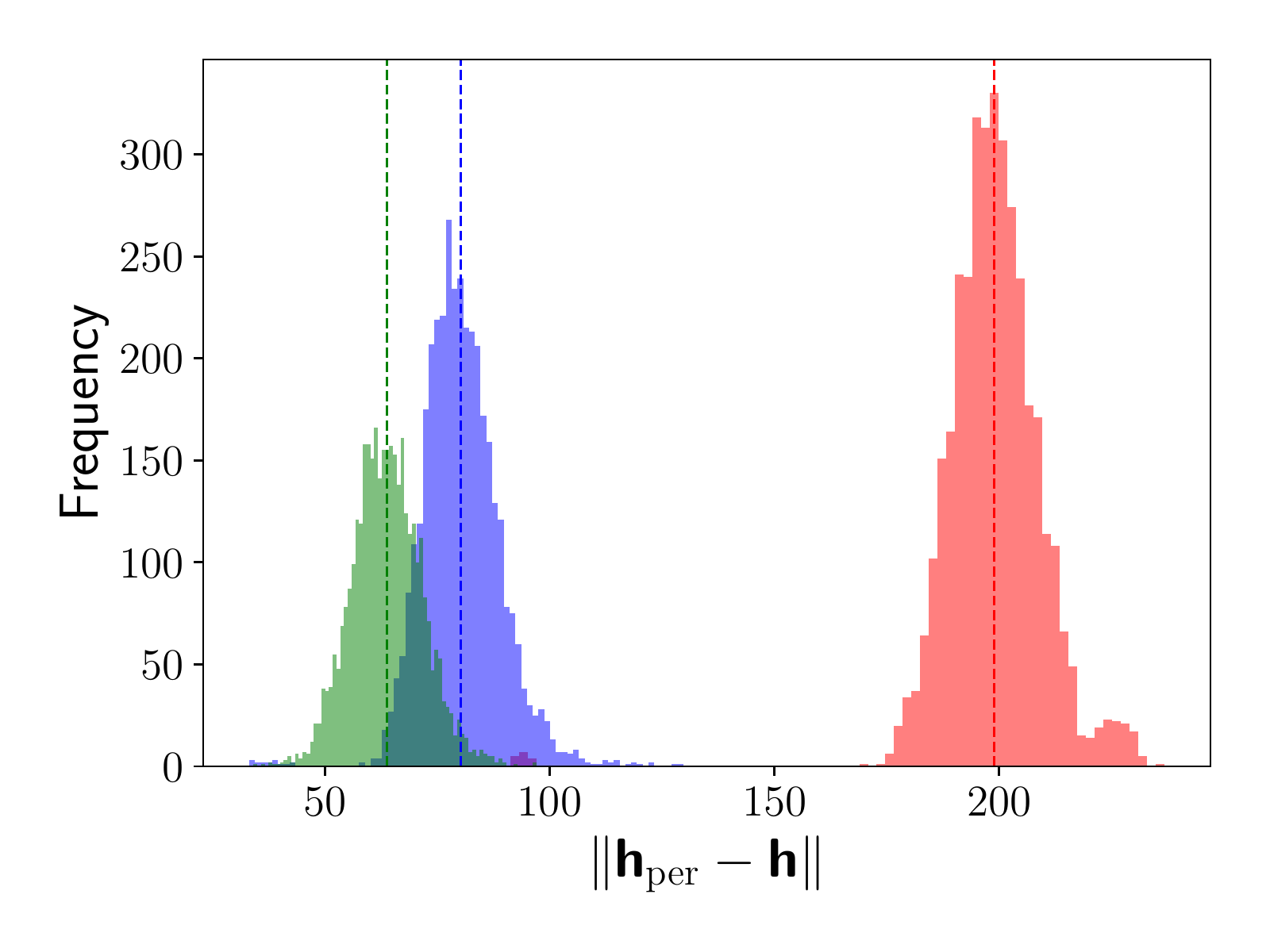}}
		\subfloat[\texttt{CoLA.}\label{fig:cola_norm}]{
			\includegraphics[width=0.24\textwidth]{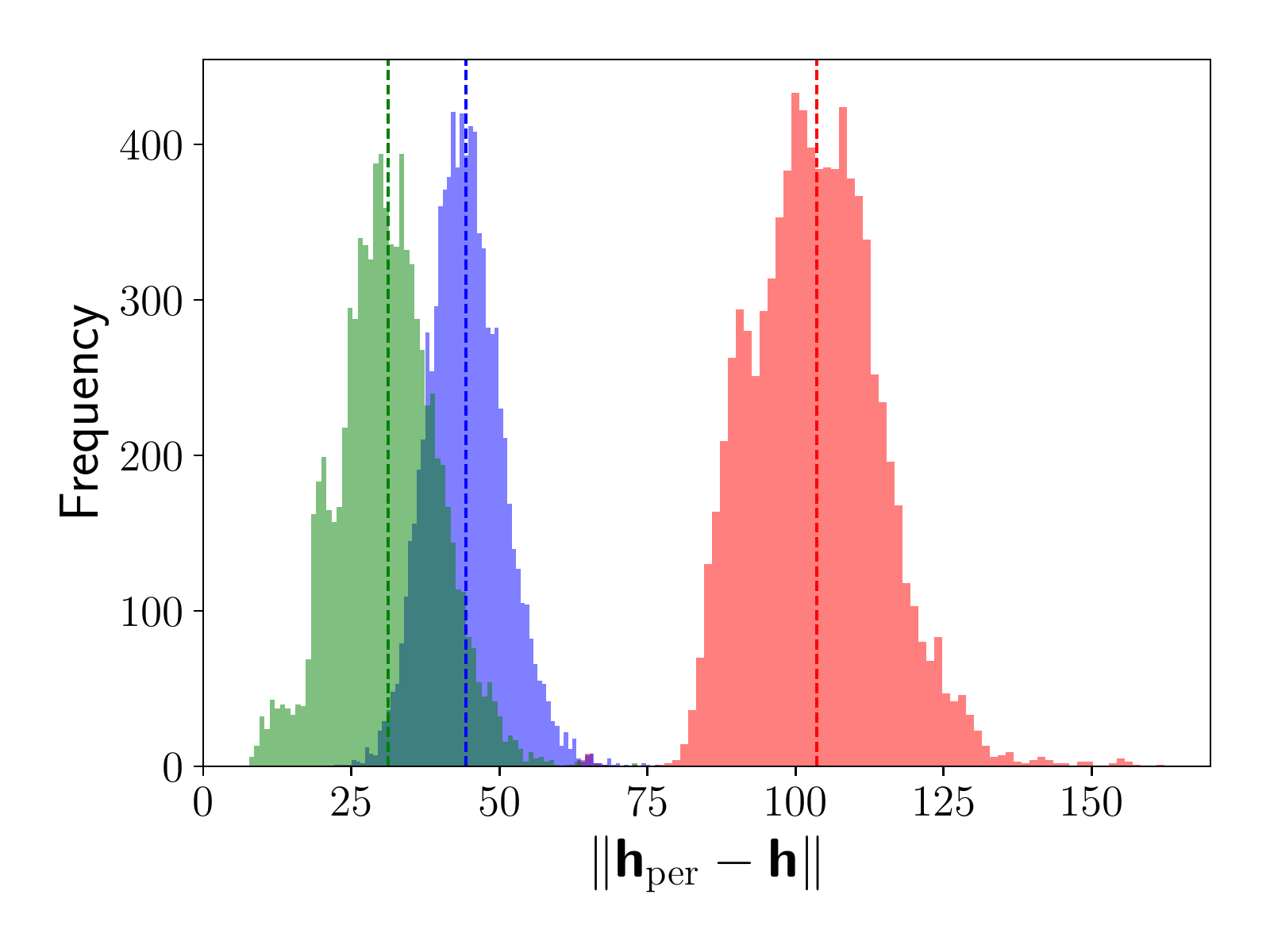}}\\
		\subfloat[\texttt{MRPC.}\label{fig:mrpc}]{
			\includegraphics[width=0.24\textwidth]{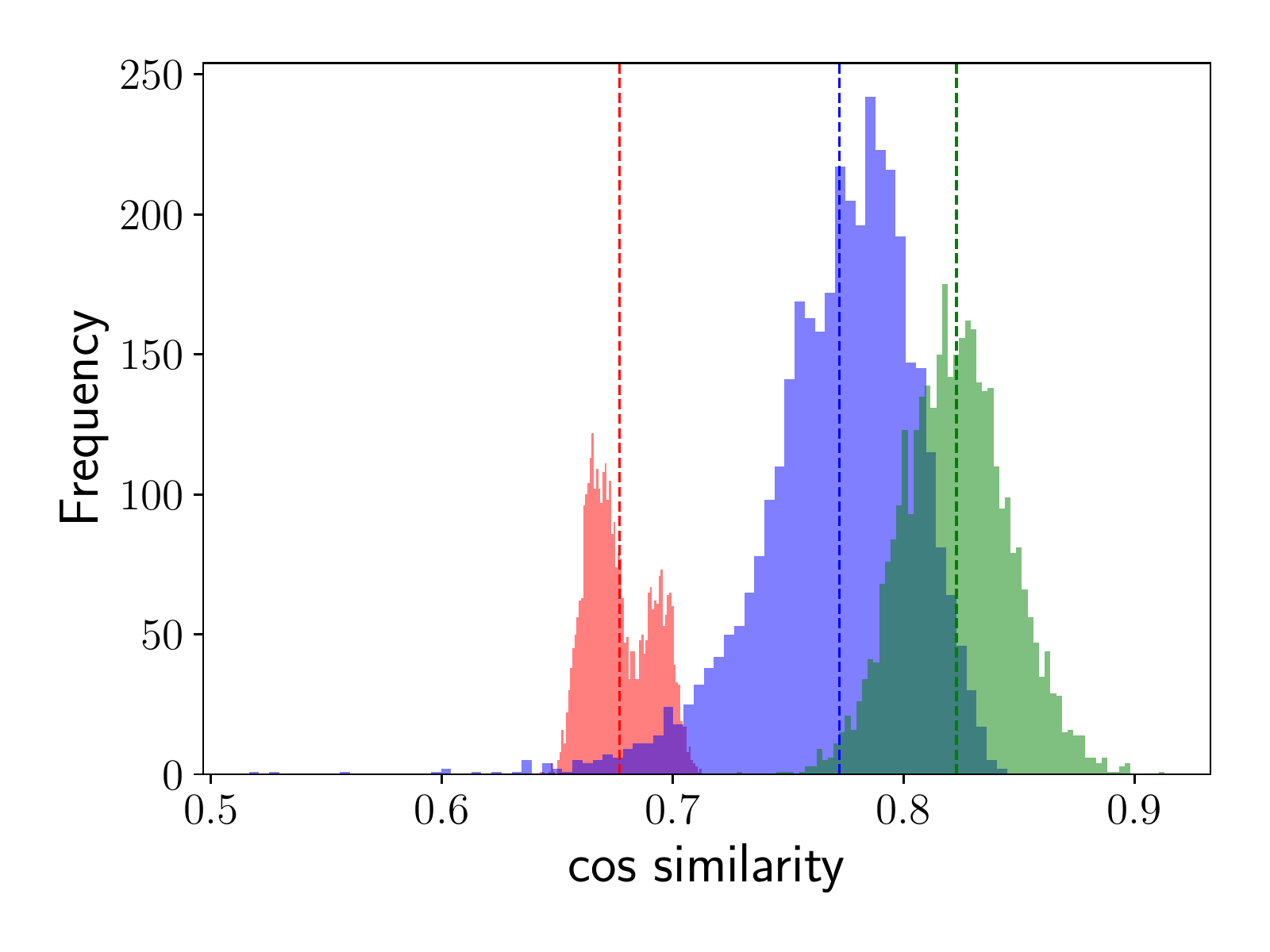}}
		\subfloat[\texttt{CoLA.}\label{fig:cola}]{
			\includegraphics[width=0.24\textwidth]{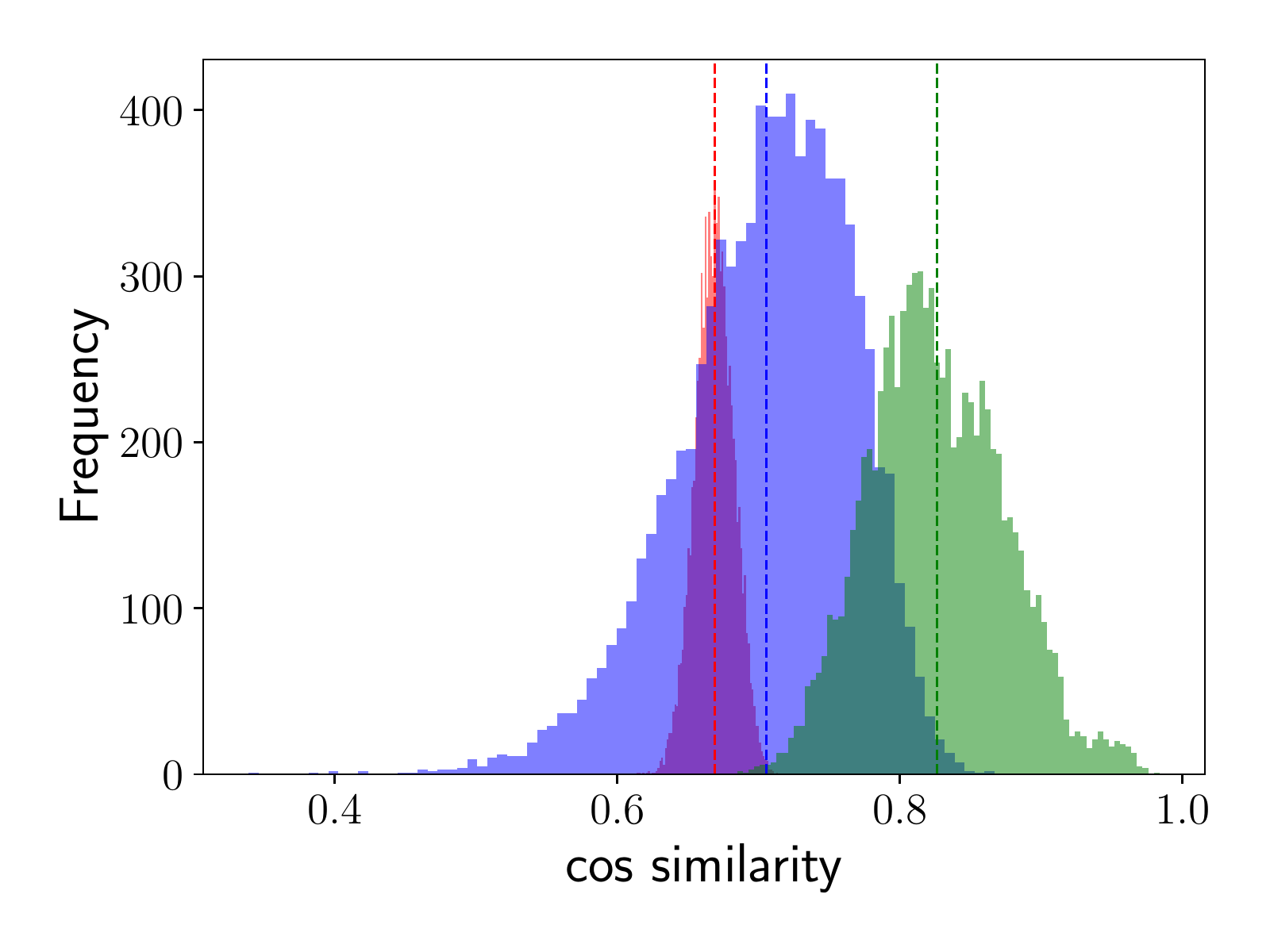}}
		\vspace{-0.05in}
		\caption{Difference of hidden states in the last Transformer layer between the original input and adversarially perturbed input, measured by $\ell_{2}$-norm and cosine similarity. The models compared are randomly initialized model, $\text{BERT}$, and $\text{RoBERTa}$.
		The datasets used are \texttt{MRPC} and \texttt{CoLA} from the GLUE benchamrk. 
		The dashed lines in the upper and bottom figures are respectively the mean of $\|\textbf{h}_{\text{per}} - \textbf{h}\|$ and $\langle\textbf{h}, \textbf{h}_{\text{per}}\rangle/(\|\textbf{h}\|\|\textbf{h}_{\text{per}}\|)$ from all samples in a dataset.} 
		\label{fig:comp}
		\vspace{-0.1in}
	\end{figure}
	
	\paragraph{Main Results.} 
	The histograms of $\|\textbf{h}_{\text{per}} - \textbf{h}\|$ and $\langle\textbf{h}, \textbf{h}_{\text{per}}\rangle/(\|\textbf{h}\|\|\textbf{h}_{\text{per}}\|)$ from all training samples 
	in MRPC and CoLA 
	are shown in Figure \ref{fig:mrpc_norm}, \ref{fig:cola_norm} and Figure \ref{fig:mrpc}, \ref{fig:cola}, respectively. 
	% 	The red dashed line is the mean of $\|\textbf{h}_{\text{per}} - \textbf{h}\|$ in each figure. 
	% 	The results (i.e., the distribution of $\|\textbf{h}_{\text{per}} - \textbf{h}\|$) clearly show that 
	We can observe that (i) $\text{BERT}$ is more robust than the randomly initialized model, 
	indicating that the masked language modeling objective and 
	sufficient  pre-training  improves input-robustness, and  leads to better ood performance after fine-tuning; 
	(ii) $\text{RoBERTa}$ is more input-robust compared with $\text{BERT}$, which implies that
	that more training samples and updating steps in the pre-training stage improve the input-robustness.
	Combining  with that a more input-robust pre-trained model also leads to better OOD generalization on downstream tasks empirically (Section~\ref{sec:pretrain improves ood}),
	the above observations (i) and (ii) may also explain the finding in \citep{hendrycks2020pretrained} that $\text{BERT}$ generalizes worse on downstream OOD data than $\text{RoBERTa}$, but much better than the model without pretraining. 
	% 	We speculate the improvements of $\text{RoBERTa}_{\text{BASE}}$ over $\text{BERT}_\text{BASE}$ and $\text{BERT}_\text{BASE}$ over no pre-trained model are from their better input-robustness of as we discussed in Section \ref{sec:pretrain improves ood}.
	
	\section{Conclusion}  
	In this paper, we explore the relationship between the robustness and OOD generalization of a model.
	% 	, both theoretically and empirically.
	% 	In summary, 
	% 	we first 
	We theoretically show that the input-robust model can generalize well on OOD data
	under the definition of OOD generalization via Wasserstein distance. Thus, for a model trained from scratch, 
	% 	Then, 
	we suggest using adversarial training to improve the input-robustness of the model which results in better OOD generalization. 
	Under mild conditions, we show that the excess risk on OOD data of an adversarially trained model is upper bounded by $\tilde{\cO}(1/\sqrt{n} + 1/T)$. 
	For the framework of first pre-training and then fine-tuning, 
	% 	Furthermore, we explore the relationship between pre-training and OOD generalization. 
	% 	We 
	we show that a pre-trained input-robust model provides a theoretically good initialization which empirically improves OOD generalization after fine-tuning. Various experiments on CV and NLP verify our theoretical findings.  
	\bibliography{reference}
	\bibliographystyle{icml2021}
	\newpage
	\appendix
	\onecolumn
\section{Proofs for Section \ref{sec:Learning Robust Model Results in Better OOD Generalization}}\label{app:proof in Learning Robust Model Results in Better OOD Generalization}
\subsection{Proofs for Section \ref{sec:Robustness Corresponds with Better OOD Generalization}}\label{app:proof in Robustness Corresponds with Better OOD Generalization}
\subsubsection{Proof of Theorem \ref{thm:ood generalization upper bound}}
To start the proof of Theorem \ref{thm:ood generalization upper bound}, we need the following lemma.
\begin{lemma}
	\label{lem:equivalence}
	For any $\bw$ and $r$, we have 
	\begin{equation}
	\small
	\sup_{P\in B_{\sW_{\infty}}(P_{0}, r)}R_{P}(\bw) = \mE_{P_{0}}\left[\sup_{\|\bdelta\|_{\infty}\leq r}f(\bw, \bx + \bdelta)\right].
	\end{equation}
\end{lemma}
\begin{proof}
	Let $T_{r}^{\bw}(\bx) = \bx + \arg\max_{\{\bdelta: \|\bdelta\|_{\infty} \leq r\}}f(\bw, \bx + \bdelta)$ with $\bx$ is an input data. The existence of $T_{r}^{\bw}(\bx)$ is guaranteed by the continuity of $f(\bw, \bx)$. $P_{r}$ is the distribution of $T_{r}^{\bw}(\bx)$ with $\bx\sim P_{0}$. Then
	\begin{equation}
	\small
	\mE_{P_{0}}\left[\sup_{\|\bdelta\|_{\infty}\leq r}f(\bw, \bx + \bdelta)\right] = \mE_{P_{r}}[f(\bw, \bx)].
	\end{equation}
	Since
	\begin{equation}
	\small
	\sW_{\infty}(P_{0}, P_{r}) \leq \mE_{P_{0}}[\|\bx - T_{r}^{\bw}(\bx)\|_{\infty}] \leq r,
	\end{equation}
	we have 
	\begin{equation}
	\small
	\mE_{P_{0}}\left[\sup_{\|\bdelta\|_{\infty}\leq r}f(\bw, \bx + \bdelta)\right] \leq \sup_{P\in B_{\sW_{\infty}}(P_{0}, r)}R_{P}(\bw).
	\end{equation}
	On the other hand, let $P^{*}_{r}\in\arg\max_{P\in B_{\sW_{\infty}}(P_{0}, r)}R_{P}(\bw)$. Due to Kolmogorov's theorem, $P^{*}_{r}$ can be distribution of some random vector $\bz$, due to the definition of $\sW_{\infty}$-distance, we have $\|\bz - \bx\|_{\infty}\leq r$ holds almost surely. Then we conclude
	\begin{equation}
	\small
	\sup_{P\in B_{\sW_{\infty}}(P_{0}, r)}R_{P}(\bw) = R_{P^{*}_{r}}(\bw) = \mE_{P^{*}_{r}}[f(\bw, \bz)] \leq \mE_{P_{0}}\left[\sup_{\|\bdelta\|_{\infty}\leq r}f(\bw, \bx + \bdelta)\right].
	\end{equation}
	Thus, we get the conclusion.   
\end{proof}
This lemma shows that the distributional perturbation measured by $\sW_{\infty}$-distance is equivalent to input perturbation. Hence we can study $\sW_{\inf}$-distributional robustness through $\ell_{\inf}$-input-robustness. The basic tool for our proof is the covering number, which is defined as follows. 
\begin{definition}\citep{wainwright2019}
	A $r$-cover of $(\cX, \|\cdot\|_{p})$ is any point set $\{\bu_{i}\}\subseteq\cX$ such that for any $\bu\in\cX$, there exists $\bu_{i}$ satisfies $\|\bu - \bu_{i}\|_{p}\leq r$. The covering number $\cN(r, \cX, \|\cdot\|_{p})$ is the cardinality of the smallest $r$-cover.  
\end{definition}
Now we are ready to give the proof of Theorem \ref{thm:ood generalization upper bound} which is motivated by \citep{xu2012robustness}.
\begin{proof}[Proof of Theorem \ref{thm:ood generalization upper bound}]
	We can construct a $r$-cover to $(\cX, \|\cdot\|_{2})$ then $\cN(r, \cX, \|\cdot\|_{2})\leq (2d_{0})^{(2D/r^{2} + 1)} = N$, because the $\cX$ can be covered by a polytope with $\ell_{2}$-diameter smaller than $2D$ and $2d_{0}$ vertices, see \citep{vershynin2018} Theorem 0.0.4 for details. Due to the geometrical structure, we have $\cN(r, \cX, \|\cdot\|_{\infty})\leq (2d_{0})^{(2D/r^{2} + 1)}$. Then, there exists $(C_{1}, \cdots, C_{N})$ covers $(\cX, \|\cdot\|_{\infty})$ where $C_{i}$ is disjoint with each other, and $\|\bu - \bv\|_{\infty} \leq r$ for any $\bu, \bv\in C_{i}$. This can be constructed by $C_{i} = \hat{C}_{i}\bigcap\left(\bigcup_{j=1}^{i-1}\hat{C}_{j}\right)^{c}$ with $(\hat{C}_{1},\cdots, \hat{C}_{N})$ covers $(\cX, \|\cdot\|_{\infty})$, and the diameter of each $\hat{C}_{i}$ is smaller than $r$ since $\cN(r, \cX, \|\cdot\|_{\infty}) \leq N$. Let $A_{j} = \{\bx_{i}: \bx_{i}\in C_{j}\}$, and $|A_{j}|$ is the cardinality of $A_{j}$. Due to Lemma \ref{lem:equivalence}, we have 
	\begin{equation}
	\label{eq:generalization decomposition}
	\small
	\begin{aligned}
	\left|\sup_{P\in B_{\sW_{\infty}}(P_{0}, r_{0})}R_{P}(\bw) - R_{P_{n}}(\bw) \right| & = \left|\mE_{P_{0}}\left[\sup_{\|\bdelta\|_{\infty}\leq r_{0}}f(\bw, \bx + \bdelta)\right] -  R_{P_{n}}(\bw)\right| \\
	& = \left|\sum\limits_{j=1}^{N}\mE_{P_{0}}\left[\sup_{\|\bdelta\|_{\infty}\leq r_{0}}f(\bw, \bx + \bdelta)\mid \bx\in C_{j}\right]P_{0}(C_{j}) - R_{P_{n}}(\bw)\right| \\
	& \leq \left|\sum\limits_{j=1}^{N}\mE_{P_{0}}\left[\sup_{\|\bdelta\|_{\infty}\leq r_{0}}f(\bw, \bx + \bdelta)\mid \bx\in C_{j}\right]\frac{|A_{j}|}{n} - \frac{1}{n}\sum\limits_{i=1}^{n}f(\bw, \bx_{i})\right|\\
	& + \left|\sum\limits_{j=1}^{N}\mE_{P_{0}}\left[\sup_{\|\bdelta\|_{\infty}\leq r_{0}}f(\bw, \bx + \bdelta)\mid \bx\in C_{j}\right]\left(\frac{|A_{j}|}{n} - P_{0}(C_{j})\right)\right| \\
	& \leq \left|\frac{1}{n}\sum\limits_{j=1}^{N}\sum\limits_{\bx_{i}\in C_{j}}\sup_{\bx\in C_{j} + B_{\infty}(\textbf{0}, r_{0})}|f(\bw, \bx) - f(\bw, \bx_{i})|\right| + M\sum\limits_{j=1}^{N}\left|\frac{|A_{j}|}{n} - P_{0}(C_{j})\right| \\
	& \overset{a}{\leq} \frac{1}{n}\sum\limits_{i=1}^{n}\sup_{\|\bdelta\|_{\infty}\leq 2r}\left|f(\bw, \bx_{i} + \bdelta) - f(\bw, \bx_{i})\right| + M\sum\limits_{j=1}^{N}\left|\frac{|A_{j}|}{n} - P_{0}(C_{j})\right| \\ 
	& \leq \epsilon + M\sum\limits_{j=1}^{N}\left|\frac{|A_{j}|}{n} - P_{0}(C_{j})\right|.
	\end{aligned}
	\end{equation}
	Here $a$ is due to $C_{j} + B_{\infty}(\textbf{0}, r) \subseteq B_{\infty}(\bx_{i}, 2r)$ when $\bx_{i}\in C_{j}$, since $\ell_{\infty}$-diameter of $C_{j}$ is smaller than $r$. The last inequality is due to $(2r, \epsilon, P_{n}, \infty)$-robustness of $f(\bw, \bx)$. On the other hand, due to Proposition A6.6 in \citep{van2000weak}, we have
	\begin{equation}
	\label{eq:multinomial concentration}
	\small
	\bbP\left(\sum\limits_{j=1}^{N}\left|\frac{|A_{j}|}{n} - P_{0}(C_{j})\right|\geq \theta\right) \leq 2^{N}\exp\left(\frac{-n\theta^{2}}{2}\right).
	\end{equation}
	Combine this with \eqref{eq:generalization decomposition}, due to the value of $N$, we get the conclusion.        
\end{proof}
\subsubsection{Proof of Theorem \ref{thm:ood generalization upper bound l2}}
There is a little difference of proving Theorem \ref{thm:ood generalization upper bound l2} compared with Theorem \ref{thm:ood generalization upper bound}. Because the out-distribution $P$ constrained in $B_{\sW_{\infty}}(P_{0}, r)$ only correspond with OOD data that contained in a $\ell_{\infty}$-ball of in-distribution data almost surely, see Lemma \ref{lem:equivalence} for a rigorous description. Hence, we can utilize $\ell_{\infty}$-robustness of model to derive the OOD generalization under $\sW_{\infty}$-distance by Theorem \ref{thm:ood generalization upper bound}. 
However, in the regime of $\sW_{2}$-distance, roughly speaking, the transformed OOD data $T_{r}^{\bw}(\bx)$ is contained in a $\ell_{2}$-ball of $\bx$ in expectation. Thus, Lemma \ref{lem:equivalence} is invalid under $\sW_{2}$-distance. 
\par
To discuss the OOD generalization under $\sW_{2}$-distance, we need to give a delicate characterization to the distribution $P\in B_{\sW_{2}}(P_{0}, r)$. First, we need the following lemma.
\begin{lemma}\label{lem:optimal}
	For any $r$ and $\bw$, let $P^{*}_{r}\in\arg\max_{P\in B_{\sW_{2}}(P_{0}, r)}R_{P}(\bw)$. Then, there exists a mapping $T_{r}^{\bw}(\bx)$ such that  $T_{r}^{\bw}(\bx)\sim P^{*}_{r}$ with $\bx\sim P_{0}$.  
\end{lemma}
\begin{proof}
	The proof of Theorem 6 in \citep{sinha2018certifying} shows that 
	\begin{equation}
	\small
	R_{P^{*}_{r}}(\bw) = \sup_{P\in B_{\sW_{2}}(P_{0}, r)}R_{P}(\bw) = \inf_{\lambda\geq 0}\sup_{P, \pi\in (P, P_{0})}\left(\int_{\cX\times\cX}f(\bw, \bx) - \lambda\|\bx - \bz\|^{2}d\pi(\bx, \bz) + \lambda r\right).
	\end{equation}
	We next show that the supremum over $\pi$ in the last equality is attained by the joint distribution $(T_{r}^{\bw}(\bx), \bx)$, which implies our conclusion. For any $\lambda > 0$, we have 
	\begin{equation}
	\small
	\sup_{P, \pi\in (P, P_{0})}\left(\int_{\cX\times\cX}f(\bw, \bx) - \lambda\|\bx - \bz\|^{2}d\pi(\bx, \bz)\right) \leq \int_{\cX}\sup_{\bx}\left(f(\bw, \bx) - \lambda\|\bx - \bz\|^{2}\right)dP_{0}(\bz),
	\end{equation}
	due to the supremum in the left hand side is taken over $P$ and $\pi$. On the other hand, let $P(\cdot\mid \bz)$ and $\bx(\cdot)$ respectively be the regular conditional distribution on $\cX$ with $\bz$ given and the function on $\cX$. Since $P(\cdot\mid \bz)$ is measurable,  
	\begin{equation}
	\small
	\begin{aligned}
	\sup_{P, \pi\in (P, P_{0})}\left(\int_{\cX\times\cX}f(\bw, \bx) - \lambda\|\bx - \bz\|^{2}d\pi(\bx, \bz)\right) & \geq \sup_{P(\cdot\mid \bz)}\left(\int_{\cX\times\cX}f(\bw, \bx) - \lambda\|\bx - \bz\|^{2}dP(\bx\mid \bz)dP_{0}(\bz)\right) \\
	& \geq \sup_{\bx(\cdot)}\left(\int_{\cX}f(\bw, \bx(\bz)) - \lambda\|\bx(\bz) - \bz\|^{2}dP_{0}(\bz)\right) \\
	& \geq \int_{\cX}\sup_{\bx}\left(f(\bw, \bx) - \lambda\|\bx - \bz\|^{2}\right)dP_{0}(\bz).
	\end{aligned}
	\end{equation}
	Thus, we get the conclusion.  
\end{proof}
\begin{proof}[Proof of Theorem \ref{thm:ood generalization upper bound l2}]
	Similar to the proof of Theorem \ref{thm:ood generalization upper bound}, we can construct a disjoint cover $(C_{1}, \cdots, C_{N})$ to $(\cX, \|\cdot\|_{2})$ such that $N\leq (2d_{0})^{(2\epsilon^{2}D/r^{2} + 1)}$, and the $l_{2}$-diameter of each $C_{i}$ is smaller than $r/\epsilon$. Let $P^{*}_{r}\in\arg\max_{P\in B_{\sW_{2}}(P_{0}, r)}R_{P}(\bw)$, by Lemma \ref{lem:optimal}, we have 
	\begin{equation}
	\label{eq:sup bound}
	\small
	\begin{aligned}
	\sup_{P\in B_{\sW_{2}}(P_{0}, r)}R_{P}(\bw) & = R_{P^{*}_{r}}(\bw) \\
	& = \mE_{P_{0}}\left[f(\bw, T_{r}^{\bw}(\bx))\right] \\
	& = \mE_{P_{0}}\left[f(\bw, T_{r}^{\bw}(\bx))\left(\textbf{1}_{T_{r}^{\bw}(\bx)\in B_{2}(\bx, r / \epsilon)} + \textbf{1}_{T_{r}^{\bw}(\bx)\notin B_{2}(\bx, r / \epsilon)}\right)\right] \\
	& \leq  \mE_{P_{0}}\left[\sup_{\|\bdelta\|_{2}\leq r/\epsilon}f(\bw, \bx + \bdelta)\right] + M\bbP(T_{r}^{\bw}(\bx)\notin B_{2}(\bx, r / \epsilon)).
	\end{aligned}
	\end{equation}
	Due to the definition of $T_{r}^{\bw}(\bx)$, by Markov's inequality, we have  
	\begin{equation}
	\small
	\left(\frac{r}{\epsilon}\right) \bbP(T_{r}^{\bw}(\bx)\notin B_{2}(\bx, r / \epsilon)) \leq \int_{\cX}\|T_{r}^{\bw}(\bx) - \bx\|^{2} dP_{0}(\bx) = \sW_{2}(P_{0}, P^{*}_{r}) \leq r.
	\end{equation}
	Plugging this into \eqref{eq:sup bound}, and due to the definition of Wasserstein distance, we have 
	\begin{equation}
	\label{eq:upper bound on optimal}
	\small
	\mE_{P_{0}}\left[\sup_{\|\bdelta\|_{2}\leq r/\epsilon}f(\bw, \bx + \bdelta)\right] \leq \sup_{P\in B_{\sW_{2}}(P_{0}, r)}R_{P}(\bw) \leq \mE_{P_{0}}\left[\sup_{\|\bdelta\|_{2}\leq r/\epsilon}f(\bw, \bx + \bdelta)\right] + M\epsilon. 
	\end{equation}
	Similar to the proof of Theorem \ref{thm:ood generalization upper bound}, due to the model is $(2r/\epsilon, \epsilon, P_{n}, 2)$-robust, we have  
	\begin{equation}
	\small
	\left|\mE_{P_{0}}\left[\sup_{\|\bdelta\|_{2}\leq r/\epsilon}f(\bw, \bx + \bdelta)\right] - R_{P_{n}}(\bw)\right|\leq \epsilon + M\sqrt{\frac{(2d_{0})^{(2\epsilon^{2}D/r^{2} + 1)}\log{2} + 2\log{(1/\theta)}}{n}}
	\end{equation}
	holds with probability at least $1 - \theta$. Combining this with \eqref{eq:upper bound on optimal}, we get the conclusion.  
\end{proof}
\subsection{Proofs for Section \ref{sec:robust training}}\label{app:proof in robust training}
The proof of Theorem \ref{thm:convergence} is same for $p\in\{2, \infty\}$, we take $p=\infty$ as an example. Before providing the proof, we first give a lemma to characterize the convergence rate of the first inner loop in Algorithm \ref{alg:sgd}.
\begin{lemma}
	\label{lem:convergence}
	For any $\bw, \bx\in\{\bx_{i}\}$, and $r$, there exists $\bdelta^{*}\in\arg\max_{\{\bdelta:\|\bdelta\|_{\infty}\leq r\}}f(\bw, \bx + \bdelta)$ such that 
	\begin{equation}
	\small
	\|\bdelta_{K + 1} - \bdelta^{*}\|^{2} \leq \left(1 - \frac{\mu_{\bx}}{L_{22}}\right)^{K}\|\bdelta_{1} - \bdelta^{*}\|^{2}
	\end{equation}
	when $\bdelta_{k + 1} = \emph{\text{Proj}}_{B_{\infty}(\emph{\textbf{0}}, r)}\left(\bdelta_{k} +\eta_{\bx}\nabla_{\bx}f(\bw, \bx + \bdelta_{k})\right)$ with $\eta_{\bx} = 1 /L_{22}$. 
\end{lemma}
\begin{proof}
	The existence of $\bdelta^{*}$ is due to the continuity of $f(\bw, \cdot)$. Then 
	\begin{equation}
	\label{eq:distance descent}
	\small
	\begin{aligned}
	f(\bw, \bx + \bdelta^{*}) - f(\bw, \bx + \bdelta_{k + 1}) & = f(\bw, \bx + \bdelta^{*}) - f(\bw, \bx + \bdelta_{k}) + f(\bw, \bx + \bdelta_{k}) - f(\bw, \bx + \bdelta_{k + 1})\\
	& \overset{a}{\leq} \langle \nabla_{\bx}f(\bw, \bx + \bdelta_{k}), \bdelta^{*} - \bdelta_{k}\rangle -\frac{\mu_{\bx}}{2}\|\bdelta_{k} - \bdelta^{*}\|^{2}\\
	& + \langle\nabla_{\bx}f(\bw, \bx + \bdelta_{k}), \bdelta_{k} - \bdelta_{k + 1}\rangle + \frac{L_{22}}{2}\|\bdelta_{k + 1} - \bdelta_{k}\|^{2} \\
	& = \langle\nabla_{\bx}f(\bw, \bx + \bdelta_{k}), \bdelta^{*} - \bdelta_{k + 1}\rangle - \frac{\mu_{\bx}}{2}\|\bdelta_{k} - \bdelta^{*}\|^{2} + \frac{L_{22}}{2}\|\bdelta_{k + 1} - \bdelta_{k}\|^{2} \\
	& \overset{b}{\leq}  L_{22}\langle \bdelta_{k + 1} - \bdelta_{k}, \bdelta^{*} - \bdelta_{k + 1}\rangle - \frac{\mu_{\bx}}{2}\|\bdelta_{k} - \bdelta^{*}\|^{2} + \frac{L_{22}}{2}\|\bdelta_{k + 1} - \bdelta_{k}\|^{2} \\
	& = L_{22}\langle \bdelta_{k + 1} - \bdelta_{k}, \bdelta^{*} - \bdelta_{k}\rangle - \frac{\mu_{\bx}}{2}\|\bdelta_{k} - \bdelta^{*}\|^{2} - \frac{L_{22}}{2}\|\bdelta_{k + 1} - \bdelta_{k}\|^{2},
	\end{aligned}
	\end{equation}
	where $a$ is due to the $L_{22}$-Lipschitz continuity of $\nabla_{\bx}f(\bw, \bx)$ and strongly convexity, $b$ is because the property of projection (see Lemma 3.1 in \citep{bubeck2014convex}). Then we get 
	\begin{equation}
	\small
	\begin{aligned}
	\|\bdelta_{k + 1} - \bdelta^{*}\|^{2} & = \|\bdelta_{k + 1} - \bdelta_{k}\|^{2} + \|\bdelta_{k} - \bdelta^{*}\|^{2} + 2\langle\bdelta_{k + 1} - \bdelta_{k}, \bdelta_{k} - \bdelta^{*}\rangle \\
	& \leq \left(1 - \frac{\mu_{\bx}}{L_{22}}\right)\|\bdelta_{k} - \bdelta^{*}\|^{2}
	\end{aligned}	
	\end{equation}
	by plugging \eqref{eq:distance descent} into the above equality and $f(\bw, \bx + \bdelta^{*}) - f(\bw, \bx + \bdelta_{k + 1}) \geq 0$. Thus, we get the conclusion.  
\end{proof}
This lemma shows that the inner loop in Algorithm \ref{alg:sgd} can efficiently approximate the worst-case perturbation for any $\bw_{t}$ and $\bx_{i}$. Now we are ready to give the proof of Theorem \ref{thm:convergence}. 
\par
We need the following lemma, which is Theorem 6 in \citep{rakhlin2012making}. 
\begin{lemma}
	\label{lem:concentration}
	Let $\{\xi_{1},\cdots, \xi_{t}\}$ be a martingale difference sequence with a uniform upper bound $b$. Let $V_{t} = \sum_{j=1}^{t}\Var(\xi_{j}\mid\cF_{j - 1})$ with $\cF_{j}$ is the $\sigma$-field generated by $\{\xi_{1},\cdots,\xi_{j}\}$. Then for every $a$ and $v>0$,
	\begin{equation}
	\small
	\bbP\left(\bigcup_{s\leq t}\left(\left\{\sum\limits_{j=1}^{t}\xi_{j} \geq a\right\}\bigcap \left\{V_{t} \leq v \right\}\right)\right) \leq \exp\left(\frac{-a^{2}}{2(v + ba)}\right).
	\end{equation} 
\end{lemma}
This is a type of Bennett's inequality which is sharper compared with Azuma-Hoeffding's inequality when the variance $v$ is much smaller than uniform bound $b$.   
\subsubsection{Proof of Theorem \ref{thm:convergence}}
\begin{proof}
	With a little abuse of notation, let $r(p) = r$ and define $g(\bw, \bx) = \sup_{\bdelta:\|\bdelta\|_{\infty} \leq r}f(\bw, \bx + \bdelta)$. Lemma A.5 in \citep{nouiehed2019solving} implies $g(\bw, \bx)$ has $L_{11} + \frac{L_{12}L_{21}}{\mu_{\bx}}$-Lipschitz continuous gradient with respect to $\bw$ for any specific $\bx$. Then  $\tilde{R}_{P_{n}}(\bw)$ has $L = L_{11} + \frac{L_{12}L_{21}}{\mu_{\bx}}$-Lipschitz continuous gradient. Let $\bx^{*}\in \bx + \arg\max_{\{\bdelta:\|\bdelta\|_{\infty}\leq r\}}f(\bw, \bx + \bdelta)$, due to the Lipschitz gradient of $\tilde{R}_{P_{n}}(\bw)$, 
	\begin{equation}
	\small
	\begin{aligned}
	\tilde{R}_{P_{n}}(\bw_{t + 1}) - \tilde{R}_{P_{n}}(\bw_{t}) & \leq \langle\nabla \tilde{R}_{P_{n}}(\bw_{t}), \bw_{t + 1} - \bw_{t}\rangle + \frac{L}{2}\|\bw_{t + 1} - \bw_{t}\|^{2} \\
	& = -\eta_{\bw_{t}}\langle\nabla \tilde{R}_{P_{n}}(\bw_{t}), \nabla_{\bw}f(\bw_{t}, \bx_{i_{t}} + \bdelta_{K})\rangle + \frac{\eta_{\bw_{t}}^{2}L}{2}\|\nabla_{\bw}f(\bw_{t}, \bx_{i_{t}} + \bdelta_{K})\|^{2} \\
	& = -\eta_{\bw_{t}}\|\nabla \tilde{R}_{P_{n}}(\bw_{t})\|^{2} + \eta_{\bw_{t}}\langle\nabla\tilde{R}_{P_{n}}(\bw_{t}), \nabla_{\bw}f(\bw_{t}, \bx_{i_{t}}^{*}) - \nabla_{\bw}f(\bw_{t}, \bx_{i_{t}} + \bdelta_{K})\rangle \\
	& + \eta_{\bw_{t}}\langle\nabla\tilde{R}_{P_{n}}(\bw_{t}),\nabla\tilde{R}_{P_{n}}(\bw_{t}) - \nabla_{\bw}f(\bw_{t}, \bx_{i_{t}}^{*})\rangle + \frac{\eta_{\bw_{t}}^{2}L}{2}\|\nabla_{\bw}f(\bw_{t}, \bx_{i_{t}} + \bdelta_{K})\|^{2}.
	\end{aligned}
	\end{equation}
	Here the last equality is due to $\nabla_{\bw}g(\bw, \bx) = \nabla_{\bw}f(\bw, \bx^{*})$ (Similar to Danskin's theorem, see Lemma A.5 in \citep{nouiehed2019solving}), and $\bx^{*}_{i_{t}}$ is the local maxima approximated by $\bx_{i_{t}} + \bdelta_{K}$ in Lemma \ref{lem:convergence}. By taking expectation to $\bw_{t + 1}$ with $\bw_{t}$ given in the both side of the above equation, Jesen's inequality, combining Lemma \ref{lem:convergence} and $\eta_{\bw_{t}} = 1/\mu_{\bw}t$, 
	\begin{equation}
	\label{eq:convergence in exp}
	\small
	\begin{aligned}
	\mE[\tilde{R}_{P_{n}}(\bw_{t + 1})] - \tilde{R}_{P_{n}}(\bw^{*})  & \leq \tilde{R}_{P_{n}}(\bw_{t}) - \tilde{R}_{P_{n}}(\bw^{*}) -\eta_{\bw_{t}}\|\nabla \tilde{R}_{P_{n}}(\bw_{t})\|^{2} \\
	& + \mE\left[\eta_{\bw_{t}}\|\nabla \tilde{R}_{P_{n}}(\bw_{t})\|\|\nabla_{\bw}f(\bw_{t}, \bx_{i_{t}}^{*}) - \nabla_{\bw}f(\bw, \bx_{i_{t}} + \bdelta_{K})\|\right] + \frac{\eta_{\bw_{t}}^{2}G^{2}L}{2} \\
	& \leq \tilde{R}_{P_{n}}(\bw_{t}) - \tilde{R}_{P_{n}}(\bw^{*}) -\eta_{\bw_{t}}\|\nabla \tilde{R}_{P_{n}}(\bw_{t})\|^{2} + \eta_{\bw_{t}}\|\nabla \tilde{R}_{P_{n}}(\bw_{t})\|\left(1 - \frac{\mu_{\bx}}{L_{22}}\right)^{K}\mE\left[\|\bdelta_{1} - \bdelta_{i_{t}}^{*}\|^{2}\right] + \frac{\eta_{\bw_{t}}^{2}G^{2}L}{2} \\
	& \leq \left(1 - 2\mu_{\bw}\eta_{\bw_{t}}\right)\left(\tilde{R}_{P_{n}}(\bw_{t}) - \tilde{R}_{P_{n}}(\bw^{*})\right) + \eta_{\bw_{t}}^{2}G^{2}L \\
	& = \left(1 - \frac{2}{t}\right)\left(\tilde{R}_{P_{n}}(\bw_{t}) - \tilde{R}_{P_{n}}(\bw^{*})\right) + \frac{G^{2}L}{\mu_{\bw}^{2}t^{2}}.
	\end{aligned}
	\end{equation}
	Here the third inequality is because  
	\begin{equation}
	\small
	\eta_{\bw_{t}}\|\nabla \tilde{R}_{P_{n}}(\bw_{t})\|\left(1 - \frac{\mu_{\bx}}{L_{22}}\right)^{K}\|\bdelta_{1} - \bdelta_{i_{t}}^{*}\|^{2} \leq \eta_{\bw_{t}}G\left(1 - \frac{\mu_{\bx}}{L_{22}}\right)^{K}4d_{0}r^{2} \leq \frac{\eta_{\bw_{t}}^{2}G^{2}L}{2},
	\end{equation}
	for any $\bdelta_{i_{t}}^{*}$, since 
	\begin{equation}
	\small
	K\log{\left(1 - \frac{\mu_{\bx}}{L_{22}}\right)} \leq -K\frac{\mu_{\bx}}{L_{22}} \leq \log{\left(\frac{GL}{8T\mu_{\bw}d_{0}r^{2}}\right)}.
	\end{equation}
	Then by induction,  
	\begin{equation}
	\small
	\begin{aligned}
	\mE[\tilde{R}_{P_{n}}(\bw_{t + 1})] - \tilde{R}_{P_{n}}(\bw^{*}) & \leq \frac{G^{2}L}{\mu^{2}_{\bw}}\sum\limits_{j=2}^{t}\frac{1}{j^{2}}\prod_{k=j + 1}^{t}\left(1 - \frac{2}{k}\right) \\
	& = \frac{G^{2}L}{\mu^{2}_{\bw}}\sum\limits_{j=2}^{t}\frac{1}{j^{2}}\frac{(j - 1)j}{(t - 1)t} \\
	& \leq \frac{G^{2}L}{t\mu^{2}_{\bw}}.
	\end{aligned}
	\end{equation}
	Thus we get the first conclusion of convergence in expectation by taking $t=T$ for $t\geq 2$. For the second conclusion, let us define $\xi_{t} = \langle\nabla\tilde{R}_{P_{n}}(\bw_{t}),\nabla\tilde{R}_{P_{n}}(\bw_{t}) - \nabla_{\bw}f(\bw_{t}, \bx_{i_{t}}^{*})\rangle$. Then Schwarz inequality implies that
	\begin{equation}
	\small
	|\xi_{t}| \leq \|\nabla\tilde{R}_{P_{n}}(\bw_{t})\|\|\nabla\tilde{R}_{P_{n}}(\bw_{t}) - \nabla_{\bw}f(\bw_{t}, \bx_{i_{t}}^{*})\| \leq 2G^{2}.
	\end{equation}
	Similar to \eqref{eq:convergence in exp}, for $t \geq 2$,
	\begin{equation}
	\label{eq:high probablity bound}
	\small
	\begin{aligned}
	\tilde{R}_{P_{n}}(\bw_{t + 1}) - \tilde{R}_{P_{n}}(\bw^{*}) & \leq \left(1 - 2\mu_{\bw}\eta_{\bw_{t}}\right)\left(\tilde{R}_{P_{n}}(\bw_{t}) - \tilde{R}_{P_{n}}(\bw^{*})\right) + \eta_{\bw_{t}}^{2}G^{2}L + 2\eta_{\bw_{t}}\xi_{t} \\
	& \leq \frac{G^{2}L}{t\mu^{2}_{\bw}} + \frac{2}{\mu_{\bw}}\sum\limits_{j=2}^{t}\frac{\xi_{j}}{j}\prod_{k = j + 1}^{t}\left(1 - \frac{2}{k}\right)\\
	& = \frac{G^{2}L}{t\mu^{2}_{\bw}} + \frac{2}{\mu_{\bw}}\sum\limits_{j=2}^{t}\frac{1}{j}\frac{(j - 1)j}{(t - 1)t}\xi_{j} \\
	& = \frac{G^{2}L}{t\mu^{2}_{\bw}} + \frac{2}{\mu_{\bw}}\sum\limits_{j=2}^{t}\frac{(j - 1)}{(t - 1)t}\xi_{j}.
	%	& \leq \frac{G^{2}L}{t\mu^{2}_{\bw}} + \frac{2}{\mu_{\bw}t}\sum\limits_{j=2}^{t}\xi_{j}.
	\end{aligned}
	\end{equation}
	Since the second term in the last inequality is upper bonded by $\sum_{j=2}^{t}\xi_{j}$ which is a sum of martingale difference, and $|\xi_{j}| \leq 2G^{2}$, a simple Azuma-Hoeffding's inequality based on bounded martingale difference (Corollary 2.20 in \citep{wainwright2019}) can give a $\cO(1/\sqrt{t})$ convergence rate in the high probability. However, we can sharpen the convergence rate via a Bennett's inequality (Proposition 3.19 in \citep{duchi2016lecture}), because the conditional variance of $\xi_{j}$ will decrease across training. 
	%The following proof is similar to the proof of Proposition 1 in \citep{rakhlin2012making}. 
	We consider the conditional variance of $\sum_{j=2}^{t}(j - 1)\xi_{j}$, let $\cF_{j}$ be the $\sigma$-field generated by $\{\bw_{1}, \cdots, \bw_{j}\}$, since $\mE[\xi_{j}] = 0$ we have 
	\begin{equation}
	\small
	\begin{aligned}
	\Var\left(\sum_{j=2}^{t}(j - 1)\xi_{j}\mid \cF_{j - 1}\right) & = \sum_{j=2}^{t}(j - 1)^{2}\Var\left(\xi_{j}\mid \cF_{j - 1}\right) \\
	& = \sum_{j=2}^{t}(j - 1)^{2}\mE\left[\xi_{j}^{2} \mid \cF_{j - 1}\right] \\
	& \leq  4G^{2}\sum_{j=2}^{t}(j - 1)^{2} \|\nabla\tilde{R}_{P_{n}}(\bw_{j})\|^{2} \\
	& \leq 8G^{2}L\sum_{j=2}^{t}(j - 1)^{2} \left(\tilde{R}_{P_{n}}(\bw_{j}) - \tilde{R}_{P_{n}}(\bw^{*})\right),
	\end{aligned}
	\end{equation}
	%	\begin{equation}
	%	\small
	%	\bbP\left(\frac{2}{\mu_{\bw}t}\sum\limits_{j=2}^{t}\xi_{j} \geq \theta\right) \leq \exp\left(-\frac{\theta(t - 1)\mu_{\bw}^{2}}{32G^{4}}\right).
	%	\end{equation}
	where first inequality is from Schwarz's inequality and the last inequality is because 
	\begin{equation}
	\small
	\begin{aligned}
	\tilde{R}_{P_{n}}\left(\bw^{*}\right) - \tilde{R}_{P_{n}}(\bw) & \leq \tilde{R}_{P_{n}}\left(\bw - \frac{1}{L}\nabla\tilde{R}_{P_{n}}(\bw)\right) - \tilde{R}_{P_{n}}(\bw) \\
	& \leq -\left\langle
	\nabla\tilde{R}_{P_{n}}(\bw), \frac{1}{L}\nabla\tilde{R}_{P_{n}}(\bw)\right\rangle + \frac{L}{2}\left\|\frac{1}{L}\nabla\tilde{R}_{P_{n}}(\bw)\right\|^{2} \\
	& = -\frac{1}{2L}\left\|\nabla\tilde{R}_{P_{n}}(\bw)\right\|^{2},
	\end{aligned}
	\end{equation}
	for any $\bw$. By applying Lemma \ref{lem:concentration}, as long as $T\geq 4$ and $0 < \theta < 1 / e$, then with probability at least $1 - \theta$, for all $t \leq T$, 
	\begin{equation}
	\small
	\begin{aligned}
	& \tilde{R}_{P_{n}}(\bw_{t + 1}) - \tilde{R}_{P_{n}}(\bw^{*})\\
	& \leq \frac{8G}{\mu_{\bw}(t - 1)t}\max\left\{\sqrt{2L\sum_{j=2}^{t}(j - 1)^{2} \left(\tilde{R}_{P_{n}}(\bw_{j}) - \tilde{R}_{P_{n}}(\bw^{*})\right)}, G(t - 1)\sqrt{\log{\left(\frac{\log{T}}{\theta}\right)}}\right\}\sqrt{\log{\left(\frac{\log{T}}{\theta}\right)}} + \frac{G^{2}L}{t\mu^{2}_{\bw}} \\
	& \leq \frac{8G\sqrt{\log{(\log{(T/\theta)})}}}{\mu_{\bw}(t - 1)t}\sqrt{2L\sum_{j=2}^{t}(j - 1)^{2} \left(\tilde{R}_{P_{n}}(\bw_{j}) - \tilde{R}_{P_{n}}(\bw^{*})\right)} + \frac{(8\mu_{\bw}G^{2}\log{(\log{(T/\theta)})} + G^{2}L)}{t\mu_{\bw}^{2}}.
	\end{aligned}
	\end{equation} 
	Then, an upper bound to the first term in the last inequality can give our conclusion. Note that if $\tilde{R}_{P_{n}}(\bw_{j}) - \tilde{R}_{P_{n}}(\bw^{*})$ is smaller than $\cO(1 / j - 1)$, the conclusion is full-filled. To see this, we should find a large constant $a$ such that $\tilde{R}_{P_{n}}(\bw_{t + 1}) - \tilde{R}_{P_{n}}(\bw^{*}) \leq a / t$. This is clearly hold when $a\geq G^{2} / 2\mu_{\bw}$ for $t = 1$ due to the PL inequality and bounded gradient. For $t\geq 2$, we find this $a$ by induction. Let $b = 8G\sqrt{2L\log{(\log{(T/\theta)})}}/\mu_{\bw}$ and $c = (8\mu_{\bw}G^{2}\log{(\log{(T/\theta)})} + G^{2}L) / \mu_{\bw}^{2}$. A satisfactory $a$ yields 
	\begin{equation}
	\small
	\begin{aligned}
	\frac{a}{t}  \geq \frac{b}{(t - 1)t}\sqrt{a\sum\limits_{j=2}^{t}(j - 1)} + \frac{c}{t} 
	= \frac{b}{(t - 1)t}\sqrt{\frac{at(t - 1)}{2}} + \frac{c}{t} \geq\frac{1}{t}\left(b\sqrt{\frac{a}{2}} + c\right). 
	\end{aligned}
	\end{equation}
	By solving a quadratic inequality, we conclude that $a - b\sqrt{a/2} - c \geq 0$. Then
	\begin{equation}
	\small
	a \geq \left(\frac{b + \sqrt{b^{2} + 8c}}{2\sqrt{2}}\right)^{2}.
	\end{equation}
	By taking
	\begin{equation}
	\small
	a \geq 2\left(\frac{2b^{2} + 8c}{8}\right) \geq \left(\frac{b + \sqrt{b + 8c}}{2\sqrt{2}}\right)^{2}, 
	\end{equation}
	we get 
	\begin{equation}
	\small
	a \geq \frac{64G^{2}L\log{(\log{(T/\theta)})}}{\mu_{\bw}^{2}} + \frac{(16\mu_{\bw}G^{2}\log{(\log{(T/\theta)})} + G^{2}L)}{\mu_{\bw}^{2}} =  \frac{G^{2}\log{(\log{(T/\theta)})}(64L + 16\mu_{\bw}) + G^{2}L}{\mu_{\bw}^{2}},
	\end{equation}
	due to the value of $b$ and $c$. Hence, we get the conclusion by taking $t=T$.   
	%	Thus plugging this into \eqref{eq:high probablity bound}, with probability at least $1 - \theta$ we have 
	%	\begin{equation}
	%	\small
	%	\tilde{R}_{P_{n}}(\bw_{t + 1}) - \tilde{R}_{P_{n}}(\bw^{*}) \leq \frac{G^{2}L}{t\mu^{2}_{\bw}} + \frac{32G^{4}}{(t - 1)\mu_{\bw}^{2}}\log{\frac{1}{\theta}},
	%	\end{equation}
	%	which conclude our claim. 
\end{proof}  

\subsubsection{Proof of Proposition \ref{pro:robustness}}\label{app:proof of proposition robustness}
\begin{proof}
	From the definition of $\tilde{R}_{P_{n}}(\bw)$, for any $r\geq 0$, we have 
	\begin{equation}
	\small
	\frac{1}{n}\sum\limits_{i=1}^{n}\sup_{\|\bdelta\|_{p}\leq r}(f(\bw, \bx_{i} + \bdelta) - f(\bw, \bx_{i})) \leq \tilde{R}_{P_{n}}(\bw) \leq \epsilon.
	\end{equation}
	On the other hand
	\begin{equation}
	\small
	\frac{1}{n}\sum\limits_{i=1}^{n}\sup_{\|\bdelta\|_{p}\leq r}(f(\bw, \bx_{i}) - f(\bw, \bx_{i} + \bdelta)) \leq R_{P_{n}}(\bw) \leq \tilde{R}_{P_{n}}(\bw) \leq \epsilon.
	\end{equation} 
	Take a sum to the two above inequalities, we get 
	\begin{equation}
	\small
	\frac{1}{n}\sum\limits_{i=1}^{n}\sup_{\|\bdelta\|_{p}\leq r}|f(\bw, \bx_{i} + \bdelta) - f(\bw, \bx_{i}))| \leq \frac{1}{n}\sum\limits_{i=1}^{n}\left(\sup_{\|\bdelta\|_{p}\leq r}f(\bw, \bx_{i} + \bdelta) - \inf_{\|\bdelta\|_{p}\leq r}f(\bw, \bx_{i} + \bdelta))\right) \leq 2\epsilon.
	\end{equation}
	Then the conclusion is verified.  
\end{proof}
\section{Proofs for Section \ref{sec:pretrain improves ood}}
\subsection{Proof of Theorem \ref{thm:pretrain generalize}}\label{app:proof of theorem pretrain generalize}
\begin{proof}
	We have $r(\infty) = r$ in this theorem. The key is to bound the $|\sup_{P\in B_{\sW_{\infty}}(P_{0}, r)}R_{P}(\bw_{\text{pre}})- \sup_{Q\in B_{\sW_{\infty}}(Q_{0}, r)}R_{Q}(\bw_{\text{pre}})|$, then  triangle inequality and Hoeffding's inequality imply the conclusion. Let $P^{*}_{r}\in \arg\max_{\{P\in B_{\sW_{\infty}}(P_{0}, r)\}}R_{P}(\bw_{\text{pre}})$. For any given $\bx$, due to the continuity of $f(\bw_{\text{pre}},\cdot)$, similar to Lemma \ref{lem:equivalence}, we can find the $T^{\bw_{\text{pre}}}_{r}(\bx) = \bx + \arg\max_{\{\bdelta:\|\bdelta\|_{\infty}\leq r\}}f(\bw_{\text{pre}}, \bx + \bdelta)$. Then due to Lemma \ref{lem:equivalence}, 
	\begin{equation}
	\small
	R_{P^{*}_{r}}(\bw_{\text{pre}}) = \mE_{P_{0}}\left[\sup_{\|\bdelta\|_{\infty}\leq r}f(\bw_{\text{pre}}, \bx + \bdelta)\right].
	\end{equation}
	Thus, $T^{\bw_{\text{pre}}}_{r}(\bx) \sim P^{*}_{r}$ when $\bx\sim P_{0}$. We can find $\bz\sim Q_{0}$ due to the Kolmogorov's Theorem, and let  $T^{\bw_{\text{pre}}}_{r}(\bz)\sim Q^{*}_{r}$. By the definition of $\sW_{\infty}$-distance, one can verify $\sW_{\infty}(Q_{0}, Q^{*}_{r})\leq r$ as well as $R_{Q^{*}_{r}}(\bw_{\text{pre}}) \leq \epsilon_{\text{pre}}$. Note that $0\leq f(\bw_{\text{pre}}, \cdot) \leq M$, then 
	\begin{equation}
	\label{eq:tv distance}
	\small
	\begin{aligned}
	\left|R_{P^{*}_{r}}(\bw_{\text{pre}}) - R_{Q^{*}_{r}}(\bw_{\text{pre}})\right| & = \left|\int_{\cX}f(\bw_{\text{pre}}, \bx)dP^{*}_{r}(\bx) - \int_{\cX}f(\bw_{\text{pre}}, \bx)dQ^{*}_{r}(\bx)\right| \\
	& = \left|\int_{\cX}f(\bw_{\text{pre}}, T^{\bw_{\text{pre}}}_{r}(\bx))dP_{0}(\bx) - \int_{\cX}f(\bw_{\text{pre}}, T^{\bw_{\text{pre}}}_{r}(\bx))dQ_{0}(\bx)\right| \\
	& \leq \int_{\cX}\left|f(\bw_{\text{pre}}, T^{\bw_{\text{pre}}}_{r}(\bx))\right|\left|dP_{0}(\bx) - dQ_{0}(\bx)\right| \\
	& \leq M\int_{\cX}\left|dP_{0}(\bx) - dQ_{0}(\bx)\right| \\
	& = 2M\TV(P_{0}, Q_{0}).
	\end{aligned}
	\end{equation}
	The last equality is from the definition of total variation distance \citep{villani2008optimal}. Thus a simple triangle inequality implies that 
	\begin{equation}
	\label{eq:tv bound}
	\small
	R_{P^{*}_{r}}(\bw_{\text{pre}}) \leq \left|R_{P^{*}_{r}}(\bw_{\text{pre}}) - R_{Q^{*}_{r}}(\bw_{\text{pre}})\right| + R_{Q^{*}_{r}}(\bw_{\text{pre}}) \leq \epsilon_{\text{pre}} + 2M\TV(P_{0}, Q_{0}).
	\end{equation}
	Next we give the concentration result of $\tilde{R}_{P_{n}}(\bw_{\text{pre}})$. Due to the definition of $\tilde{R}_{P_{n}}(\bw_{\text{pre}})$, it can be rewritten as $R_{P_{n}^{*}}(\bw_{\text{pre}})$ where $P_{n}^{*}$ is the empirical distribution on $\{T^{\bw_{\text{pre}}}_{r}(\bx_{i})\}$. Since $0 \leq f(\bw_{\text{pre}}, \cdot) \leq M$ and $\{T^{\bw_{\text{pre}}}_{r}(\bx_{i})\}$ are i.i.d draws from $P^{*}_{r}$. Azuma-Hoeffding's inequality (Corollary 2.20 in \citep{wainwright2019}) shows that with probability at least $1 - \theta$, 
	\begin{equation}
	\small
	\begin{aligned}
	\tilde{R}_{P_{n}}(\bw_{\text{pre}}) - R_{P^{*}_{r}}(\bw_{\text{pre}}) & = \frac{1}{n}\sum\limits_{i=1}^{n}f(\bw_{\text{pre}}, T(\bx_{i})) - R_{P^{*}_{r}}(\bw_{\text{pre}}) \leq M\sqrt{\frac{\log{(1/\theta)}}{2n}}.
	\end{aligned}
	\end{equation}
	Hence we get our conclusion.   
\end{proof}
\subsection{Proof of Theorem \ref{thm:pretrain generalize l2}}
With a little abuse of notation, let $r(2) = r/\epsilon_{\text{pre}}$ denoted by $r$ in the proof, and $P^{*}_{r}\in\arg\max_{P\in B_{\sW_{2}}(P_{0}, r)}R_{P}(\bw)$. By Lemma \ref{lem:optimal}, there exists $T^{\bw_{\text{pre}}}_{r}(\bx)\sim P^{*}_{r}$ with $\bx\sim P_{0}$. Then we can find $\bz\sim Q_{0}$ due to Kolmogorov's Theorem. Let $T^{\bw_{\text{pre}}}_{r}(\bz)\sim Q^{*}_{r}$, we see 
\begin{equation}
\small
\begin{aligned}
\sW_{2}(Q_{0}, Q^{*}_{r})^{2} & \leq \int_{\cX}\|\bz - T^{\bw_{\text{pre}}}_{r}(\bz)\|^{2}dQ_{0}(\bz) \\
& \leq \int_{\cX}\|\bz - T^{\bw_{\text{pre}}}_{r}(\bz)\|^{2}\left|dQ_{0}(\bz) - dP_{0}(\bz)\right| + \int_{\cX}\|\bz - T^{\bw_{\text{pre}}}_{r}(\bz)\|^{2} dP_{0}(\bz) \\
& \leq D^{2}\int_{\cX}\left|dQ_{0}(\bz) - dP_{0}(\bz)\right| + r^{2} \\
& = 2D^{2}\TV(P_{0}, Q_{0}) + r^{2}. 
\end{aligned}
\end{equation}
Thus $R_{Q^{*}_{r}}(\bw_{\text{pre}}) \leq \epsilon_{\text{pre}}$. Similar to \eqref{eq:tv distance} and \eqref{eq:tv bound} we get the conclusion.  
\section{Hyperparameters}\label{app:hyp on adv}
\begin{table*}[htbp]
	
	\centering
	%	\resizebox{\textwidth}{\textwidth}{
	\scalebox{0.9}{
		\begin{minipage}{0.5\linewidth}\label{tbl:hyper adv cifar}
			\caption{Hyperparameters of adversarial training on \texttt{CIFAR10}.}
			\vspace{-0.1in}
			\begin{tabular}{c c c c}
				\hline
				Hyperparam        &  Std     & Adv-$\ell_{2}$       & Adv-$\ell_{\infty}$\\
				\hline
				Learning Rate     &  0.1     &  0.1              & 0.1     \\
				Momentum          &  0.9     &  0.9              & 0.9     \\
				Batch Size        &  128     &  128              & 128     \\
				Weight Decay      &  5e-4    &  5e-4             & 5e-4    \\
				Epochs            &  200     &  200              & 200     \\
				Inner Loop Steps  &    -     &   8               &  8      \\
				Perturbation Size &    -     &  2/12               &  2/255       \\
				Perturbation Step Size &  -  &  1/24              & 1/510     \\
				\hline
			\end{tabular}
		\end{minipage}
		\hspace{0.2in}
		\begin{minipage}{0.5\linewidth}\label{tbl:hyper adv imagenet}
			\caption{Hyperparameters of adversarial training on \texttt{ImageNet}.}
			\vspace{-0.1in}
			\begin{tabular}{c c c c}
				\hline
				Hyperparam        &  Std     & Adv-$\ell_{2}$       & Adv-$\ell_{\infty}$\\
				\hline
				Learning Rate     &  0.1     &  0.1              & 0.1     \\
				Momentum          &  0.9     &  0.9              & 0.9     \\
				Batch Size        &  512     &  512              & 512     \\
				Weight Decay      &  5e-4    &  5e-4             & 5e-4    \\
				Epochs            &  100     &  100              & 100     \\
				Inner Loop Steps  &    -     &   3               &  3      \\
				Perturbation Size &    -     & 0.25              &  2/255       \\
				Perturbation Step Size  &  - & 0.05              &  1/510       \\
				\hline
			\end{tabular}
		\end{minipage}
	}
\end{table*}

%\begin{table*}[htbp]
%	\caption{Hyperparameters of adversarial training on \texttt{ImageNet}}
%	%		\vspace{-0.1in}
%	\label{tbl:hyper adv imagenet}
%	\centering
%	%	\resizebox{\textwidth}{\textwidth}{
%		\scalebox{0.9}{
%		{
%			\begin{tabular}{c c c c}
%				\hline
%				Hyperparam        &  Std     & Adv-$\ell_{2}$       & Adv-$\ell_{\infty}$\\
%				\hline
%				Learning Rate     &  0.1     &  0.1              & 0.1     \\
%				Momentum          &  0.9     &  0.9              & 0.9     \\
%				Batch Size        &  512     &  512              & 512     \\
%				Weight Decay      &  5e-4    &  5e-4             & 5e-4    \\
%				Epochs            &  100     &  100              & 100     \\
%				Inner Loop Steps  &    -     &   8               &  8      \\
%				Perturbation Size &    -     &  2/255   & 0.25     \\
%				Perturbation Step Size  &  - &  1/510   & 0.05     \\
%				\hline
%	\end{tabular}}}
%\end{table*}

\begin{table*}[htbp]
	\caption{Hyperparameters of  adversarial training on $\text{BERT}$ base model.}
	\vspace{-0.1in}
	\label{tbl:hyper}
	\centering
	%	\resizebox{\textwidth}{\textwidth}{
	\scalebox{0.8}{
		{
			\begin{tabular}{c c c c}
				\hline
				Hyperparam &  Std     & Adv-$\ell_{2}$       & Adv-$\ell_{\infty}$\\
				\hline
				Learning Rate &  3e-5 &  3e-5     & 3e-5 \\
				Batch Size    &  32   &  32       & 32   \\
				Weight Decay  &  0    &  0        & 0    \\
				Hidden Layer Dropout Rate &  0.1  &  0.1  & 0.1  \\
				Attention Probability Dropout Rate &   0.1  &   0.1  & 0.1 \\
				Max Epochs    &  10    &  10  & 10 \\
				Learning Rate Decay   & Linear   & Linear  & Linear\\
				Warmup Ratio  &  0   &  0        & 0    \\
				Inner Loop Steps  &    -     &   3               &  3      \\
				Perturbation Size &    -     & 1.0              &  0.001        \\
				Perturbation Step Size  &  - & 0.1              &  0.0005        \\
				\hline
	\end{tabular}}}
\end{table*}

\section{Ablation Study}
\label{app:perturbation}
\subsection{Effect of Perturbation Size}\label{app:perturbation size}		
We study the effect of perturbation size $r$ in adversarial training
% 		as it appears 
in bounds \eqref{eq:ood bound linf} and \eqref{eq:ood bound l2}.
%suggest that a robust model (larger $r$) generalizes better on OOD data for a fixed error on training data $R_{P_{n}}(\bw)$. 
We vary the perturbation size $r$ in $\{2^{-5}/12, 2^{-4}/12, 2^{-3}/12, 2^{-2}/12, 2^{-1}/12, 2^{0}/12, 2^{1}/12, 2^{2}/12, 2^{3}/12, 2^{4}/12, 2^{5}/12, 2^{6}/12, 2^{7}/12\}$ for Adv-$\ell_{2}$ and in $\{2^{-4}/255, 2^{-3}/255, 2^{-2}/255, 2^{-1}/255, 2^{0}/255, 2^{1}/255, 2^{2}/255, 2^{3}/255, 2^{4}/255\}$ for Adv-$\ell_{\infty}$. 
The perturbation step size $\eta_{\bx}$ in Algorithm \ref{alg:sgd} is set to be $r/4$ \citep{salman2020adversarially}. 
Experiments are conducted on $\texttt{CIFAR10}$ and the settings follow those in Section \ref{sec:Experiments on Image Classification}. 
\par
The results are shown in Figures \ref{fig:adv_l2_r} and \ref{fig:adv_linf_r}. In the studied ranges, the accuracy on the OOD data from all categories exhibits similar trend, i.e., %(i) increases overall or (ii) 
first increases and then decreases, as $r$ increases. This is consistent with our discussion in Section \ref{sec:Experiments on Image Classification} that there is an optimal perturbation size $r$ for improving OOD generalization via adversarial training. For data corrupted under types Fog, Bright and Contrast, adversarial training degenerates the performance in Table \ref{tbl:adversarial training on image}. We speculate this is because the three corruption types rescale the input pixel values to smaller values and the same perturbation size $r$
leads to relatively large perturbation. 
%	Besides, from Figures \ref{fig:adv_l2_r} and \ref{fig:adv_linf_r}, 
%	the generalization decreases with an increasing $r$ on these OOD data. 
% 		on them than 
% 		the
% 		clean input data, and 
% 		other corruptions}. 
Thus according to the discussion in Section \ref{sec:Experiments on Image Classification} that there is an optimal $r$ for improving OOD generalization, 
we suggest conducting adversarial training with a smaller perturbation size to defend these three types of corruption. 
Figures \ref{fig:adv_l2_r} and \ref{fig:adv_linf_r} also show 
that smaller optimal perturbation sizes have better performances for these three types of corruption. 
% 		To see it more clear, we further 
%		Thus for these three corruptions, we vary the perturbation size $r$ in $\{2^{-5}, 2^{-4}, 2^{-3}, 2^{-2}, 2^{-1}, 2^{0}, 2^{1}, 2^{2}, 2^{3}\}$ for Adv-$\ell_{2}$ and $r$ in $\{2^{-4}/255, 2^{-3}/255, 2^{-2}/255, 2^{-1}/255, 2^{0}/255, 2^{1}/255, 2^{2}/255, 2^{3}/255, 2^{4}/255\}$ for Adv-$\ell_{\infty}$. 
%		The other settings are the same as the aforementioned experiments. 
%		The results are shown in Figures \ref{fig:adv_small_size_l2} and \ref{fig:adv_small_size}. 
% 		as we take smaller perturbation size $r$.   
%As can be seen, the accuracy on OOD data \texttt{CIFAR10-C} increase with $r$ for small $r$. But when $r$ reaches a large number, the performance of model on OOD data drop with the increasing of $r$. This is because there is a trade off between clean accuracy and input-robustness \citep{raghunathan2019adversarial}, while the upper bounds of excess risk \eqref{eq:ood bound l2} and \eqref{eq:ood bound linf} are decided by both the input-robustness and clean accuracy. Thus, to achieve the optimal performance on OOD data, we should properly choose the perturbation size $r$ rather than continually increasing it.   

\subsection{Effect of the the Number of Training Samples}\label{app:number of training samples}
We study the effect of the number of training samples, as  bounds \eqref{eq:ood bound linf} and \eqref{eq:ood bound l2} suggest that more training samples lead to better OOD generalization. 
We split \texttt{CIFAR10} into 5 subsets, each of which has 10000, 20000, 30000, 40000 and 50000 training samples. 
The other settings follow those in Section \ref{sec:Experiments on Image Classification}. 
The results are in shown Figures \ref{fig:adv_l2_num} and \ref{fig:adv_linf_num}. 
%As can be seen, the accuracy on OOD data \texttt{CIFAR10-C} increases with the number of training samples, which supports our theoretical conclusions in Theorem \ref{thm:ood generalization upper bound} and \ref{thm:ood generalization upper bound l2}.  
\begin{figure*}[htbp]\centering
	\subfloat[Clean.]{\includegraphics[width=0.198\textwidth]{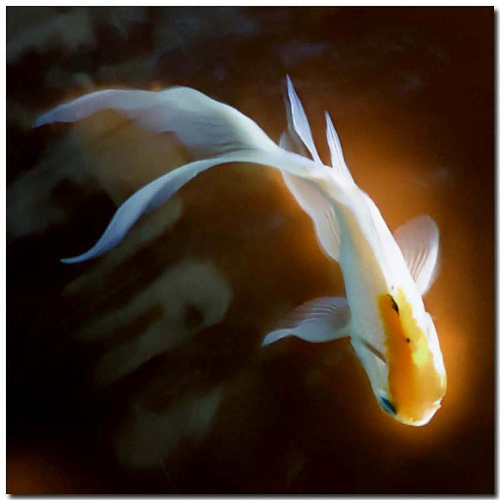}}
	\subfloat[Gauss.]{\includegraphics[width=0.195\textwidth]{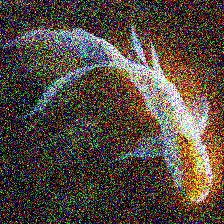}}
	\subfloat[Shot.]{\includegraphics[width=0.195\textwidth]{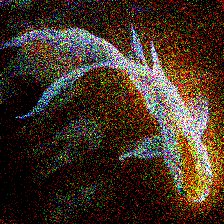}}
	\subfloat[Impulse.]{\includegraphics[width=0.195\textwidth]{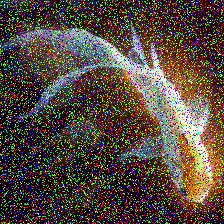}}
	\vspace{-0.1in}
	\\
	\subfloat[Defocus.]{\includegraphics[width=0.195\textwidth]{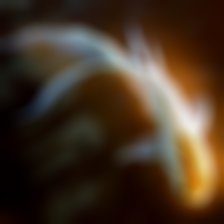}}
	\subfloat[Glass.]{\includegraphics[width=0.195\textwidth]{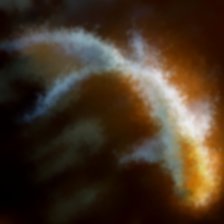}}
	\subfloat[Motion.]{\includegraphics[width=0.195\textwidth]{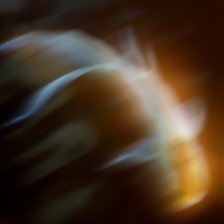}}
	\subfloat[Zoom.]{\includegraphics[width=0.195\textwidth]{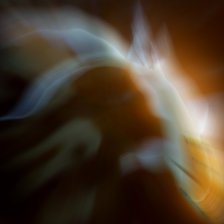}}
	\vspace{-0.1in}
	\\
	\subfloat[Snow.]{\includegraphics[width=0.195\textwidth]{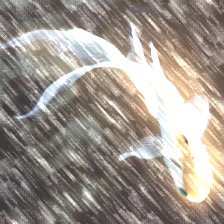}}
	\subfloat[Frost.]{\includegraphics[width=0.195\textwidth]{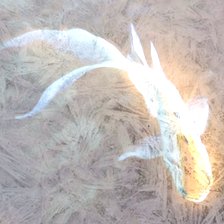}}
	\subfloat[Fog.]{\includegraphics[width=0.195\textwidth]{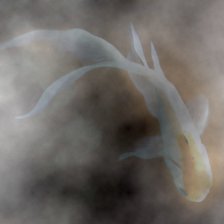}}
	\subfloat[Bright.]{\includegraphics[width=0.195\textwidth]{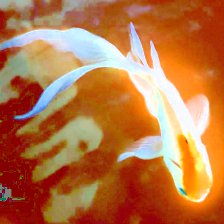}}
	\vspace{-0.1in}
	\\
	\subfloat[Contrast.]{\includegraphics[width=0.195\textwidth]{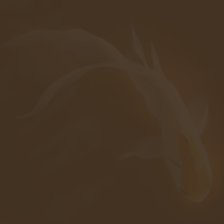}}
	\subfloat[Elastic.]{\includegraphics[width=0.195\textwidth]{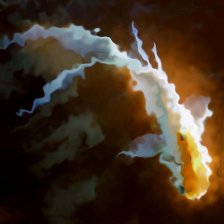}}
	\subfloat[pixel.]{\includegraphics[width=0.195\textwidth]{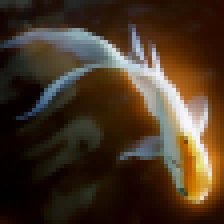}}
	\subfloat[JPEG.]{\includegraphics[width=0.195\textwidth]{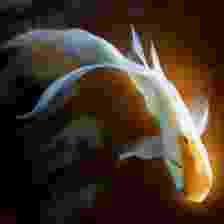}}
	\vspace{-0.1in}

	\caption{
		15 types of artificially constructed corruptions from four categories: Noise, Blur, Weather, and Digital from the \texttt{ImageNet-C} dataset \citep{hendrycks2018benchmarking}. 
		Each corruption has five levels of severity with figures under severity 5 are shown here.}
	\label{fig:imagenet-c}
\end{figure*}

\begin{figure*}[htbp]\centering
	\subfloat{\includegraphics[width=0.33\textwidth]{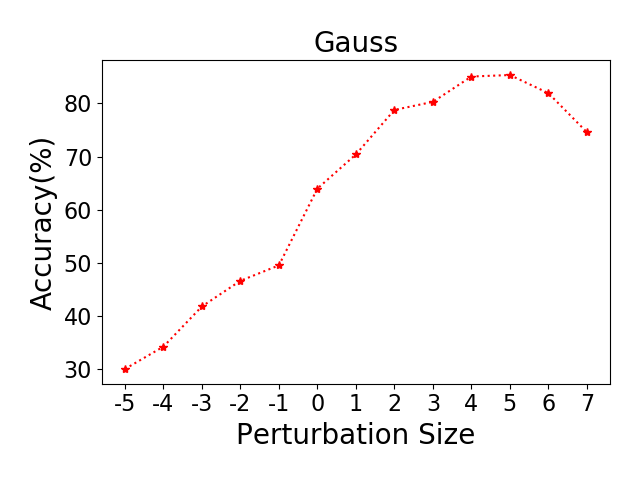}}
	\subfloat{\includegraphics[width=0.33\textwidth]{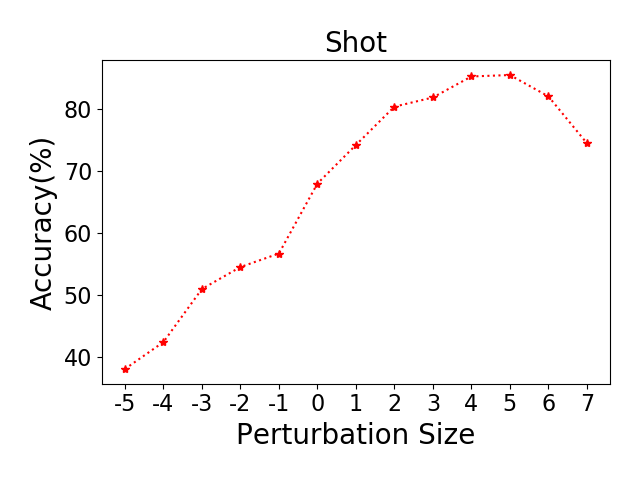}}
	\subfloat{\includegraphics[width=0.33\textwidth]{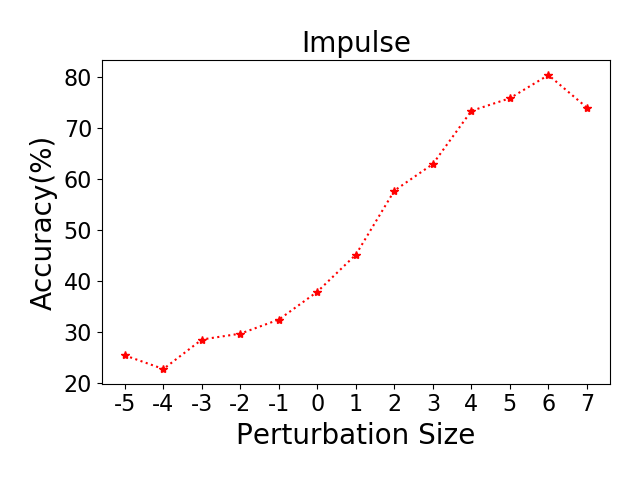}} 
	\vspace{-0.2in}
	\\
	\subfloat{\includegraphics[width=0.33\textwidth]{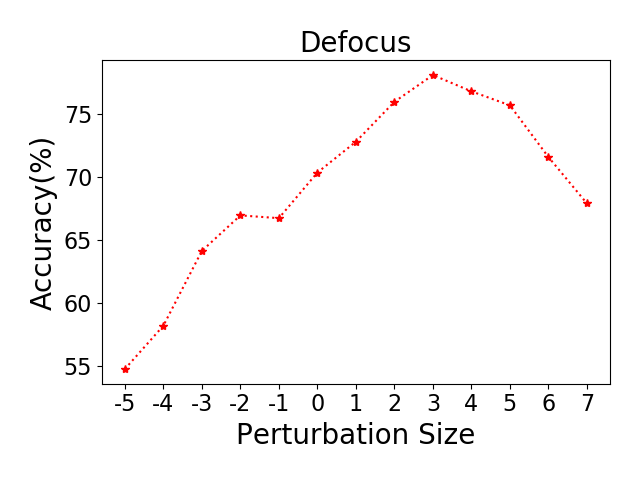}}
	\subfloat{\includegraphics[width=0.33\textwidth]{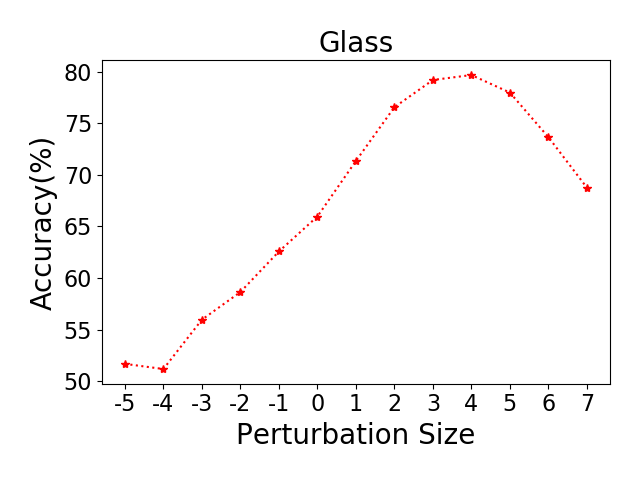}}
	\subfloat{\includegraphics[width=0.33\textwidth]{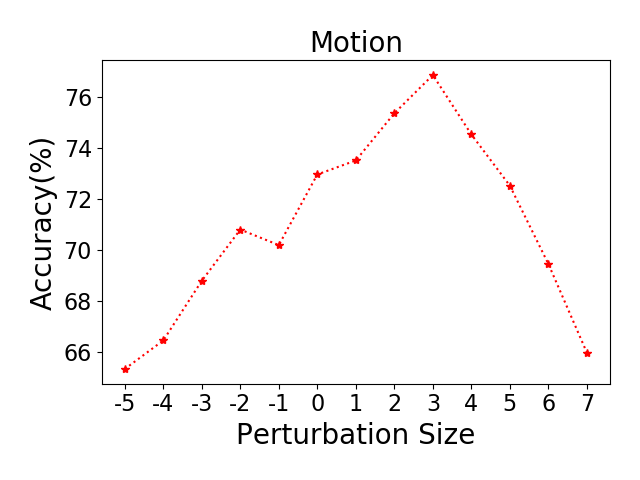}}
	\\
	\vspace{-0.2in}
	\subfloat{\includegraphics[width=0.33\textwidth]{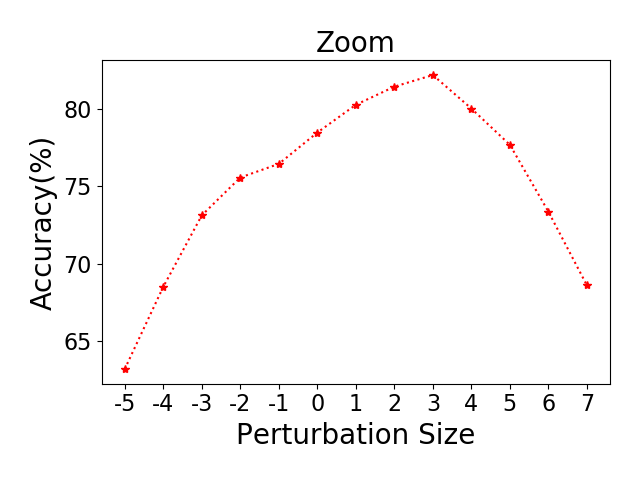}}
	\subfloat{\includegraphics[width=0.33\textwidth]{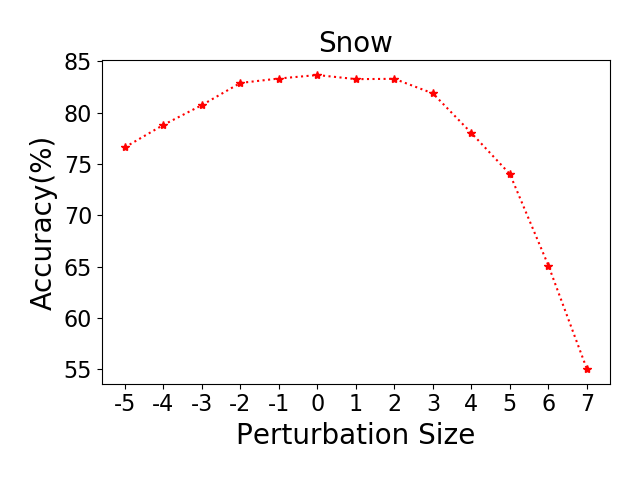}}
	\subfloat{\includegraphics[width=0.33\textwidth]{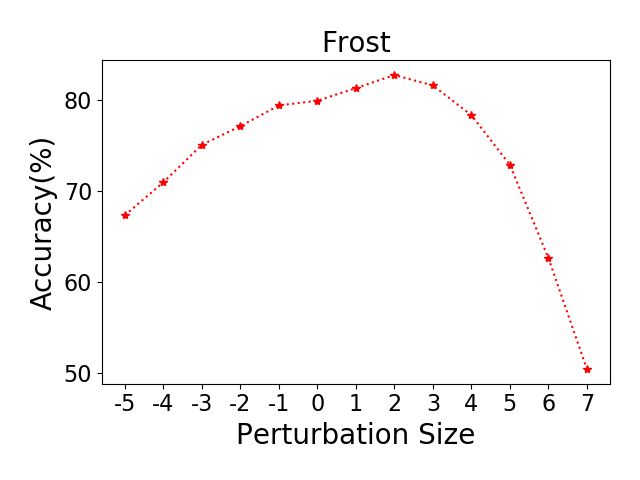}}
	\\
	\vspace{-0.2in}
	\subfloat{\includegraphics[width=0.33\textwidth]{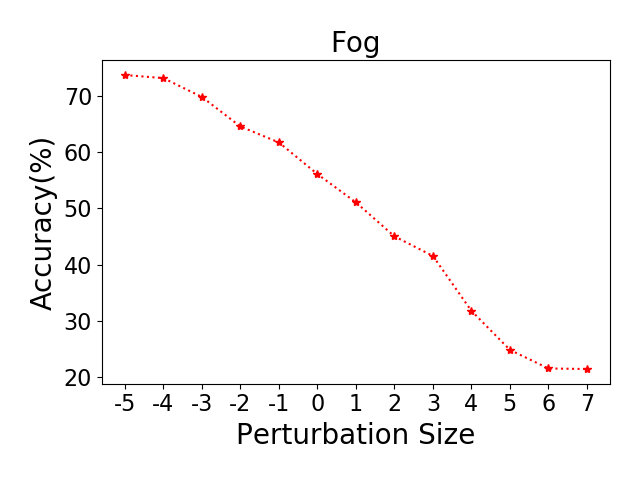}}
	\subfloat{\includegraphics[width=0.33\textwidth]{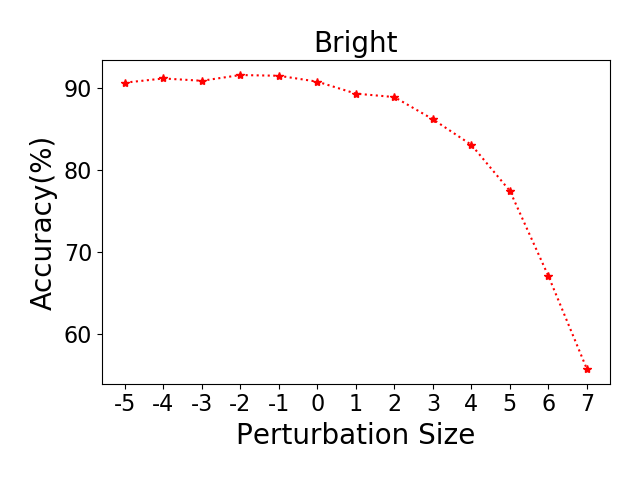}}
	\subfloat{\includegraphics[width=0.33\textwidth]{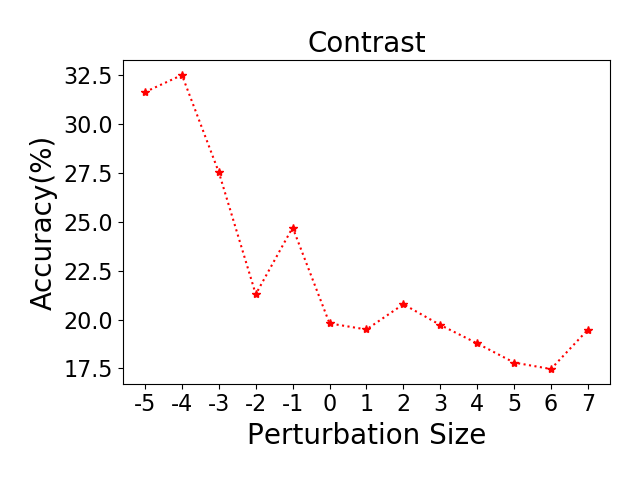}}
	\\
	\vspace{-0.2in}
	\subfloat{\includegraphics[width=0.33\textwidth]{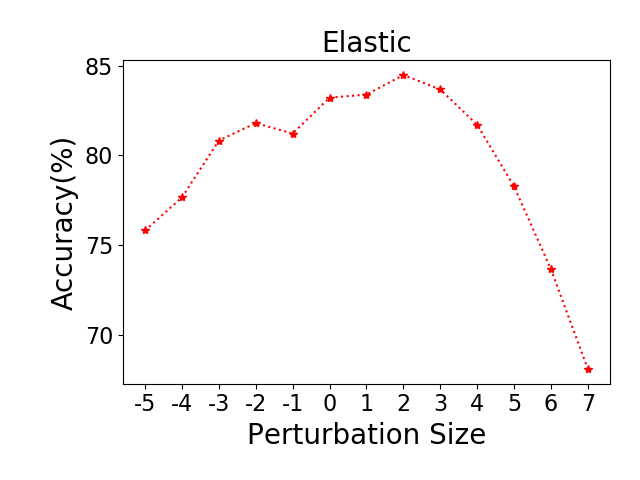}}
	\subfloat{\includegraphics[width=0.33\textwidth]{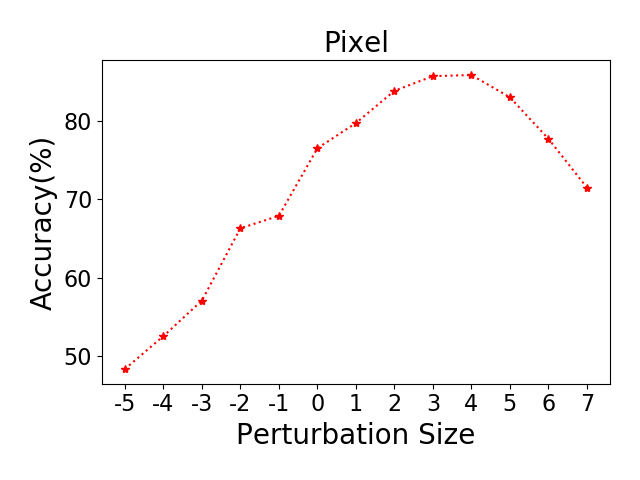}}
	\subfloat{\includegraphics[width=0.33\textwidth]{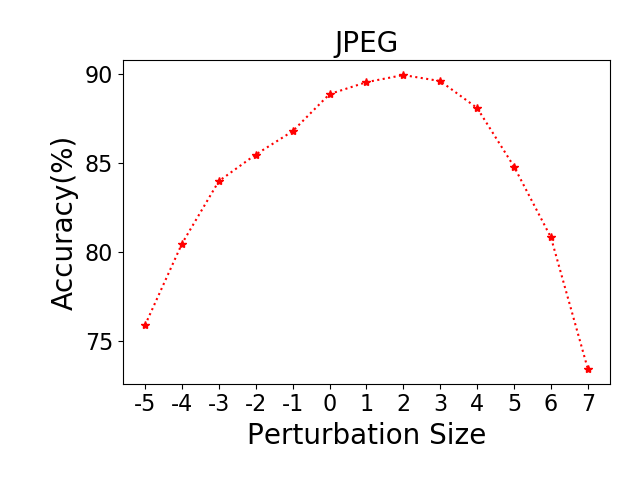}}

	\caption{Accuracy of Adv-$\ell_{2}$ on \texttt{CIFAR10-C} over various perturbation sizes. The $x$-axis means the perturbation size is $2^{x}/12$.}
	\label{fig:adv_l2_r}
\end{figure*}

\begin{figure*}[htbp]\centering
	\subfloat{\includegraphics[width=0.33\textwidth]{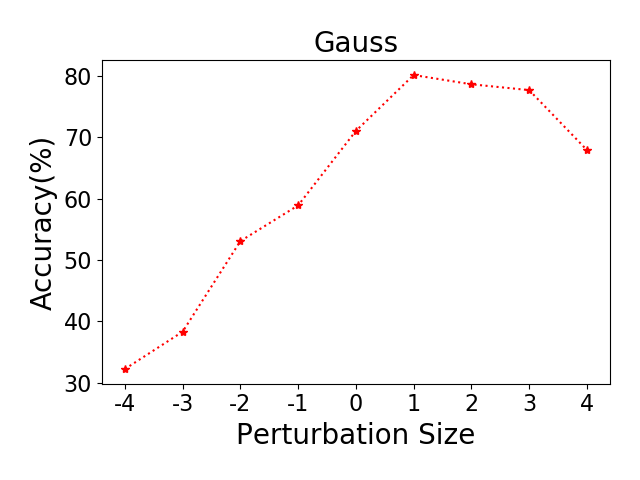}}
	\subfloat{\includegraphics[width=0.33\textwidth]{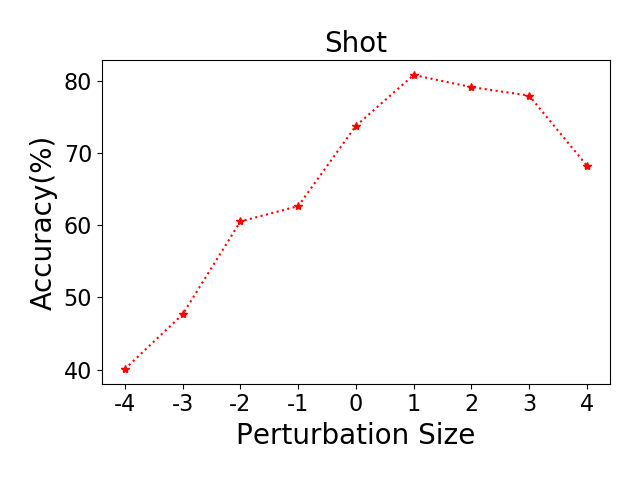}}
	\subfloat{\includegraphics[width=0.33\textwidth]{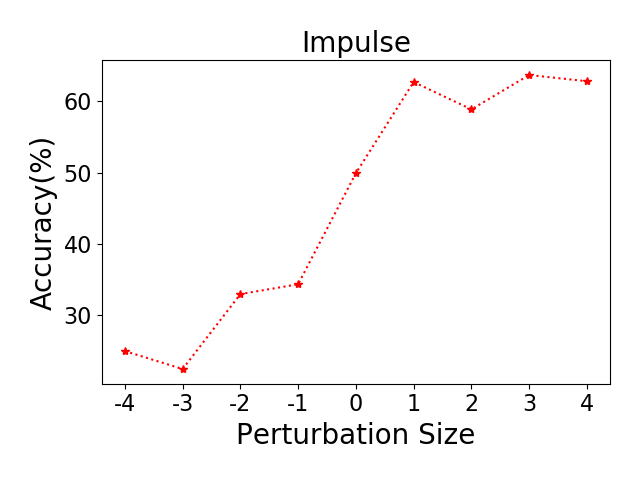}}
	\vspace{-0.2in}
	\\
	\subfloat{\includegraphics[width=0.33\textwidth]{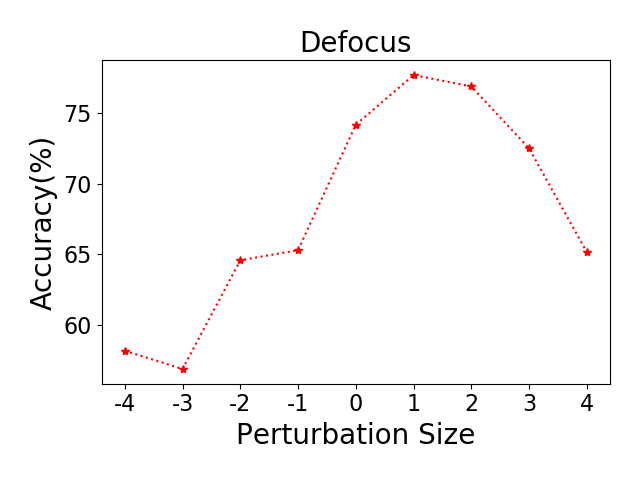}}
	\subfloat{\includegraphics[width=0.33\textwidth]{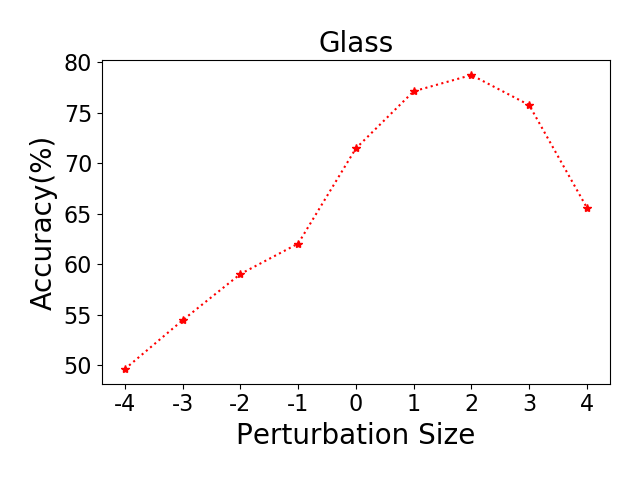}}
	\subfloat{\includegraphics[width=0.33\textwidth]{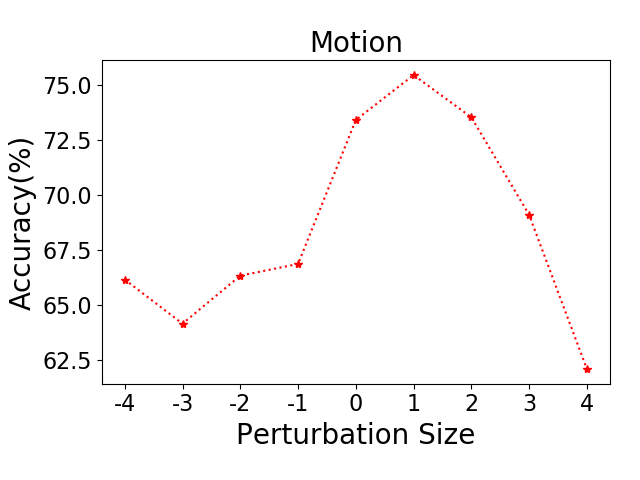}}
	\vspace{-0.2in}
	\\
	\subfloat{\includegraphics[width=0.33\textwidth]{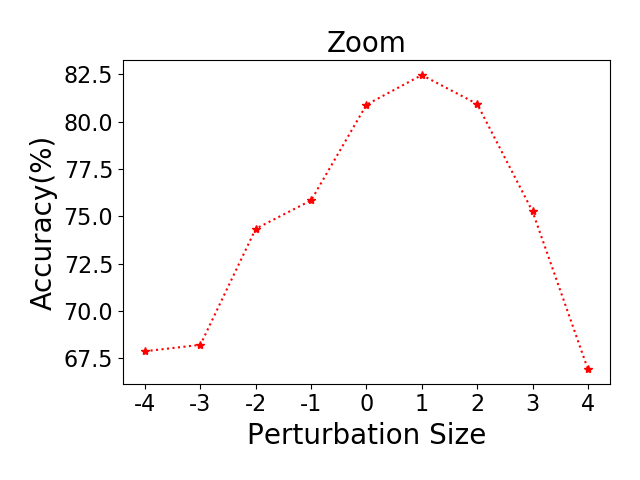}}
	\subfloat{\includegraphics[width=0.33\textwidth]{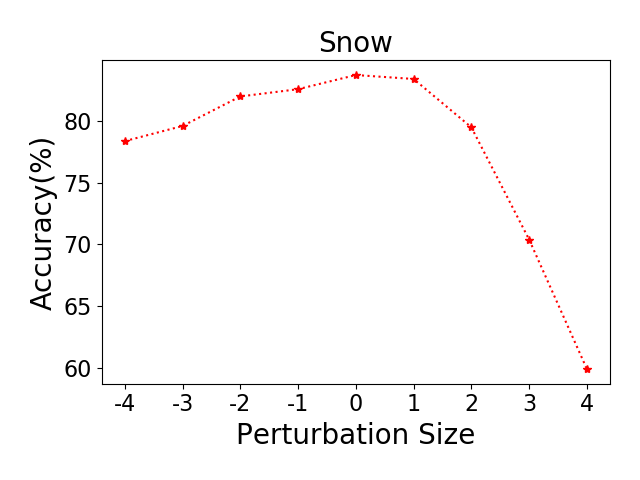}}
	\subfloat{\includegraphics[width=0.33\textwidth]{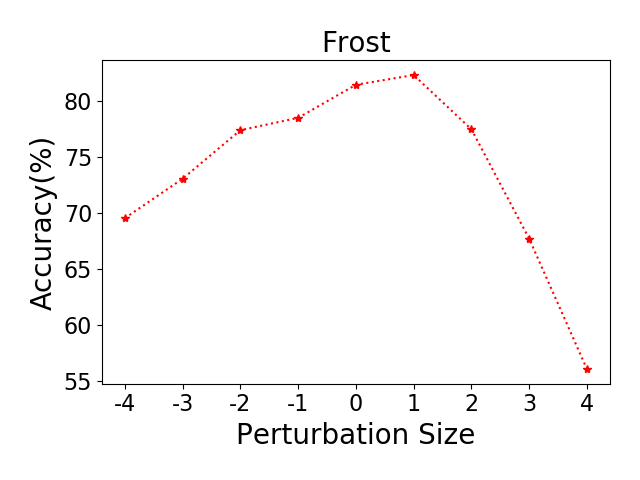}}
	\vspace{-0.2in}
	\\
	\subfloat{\includegraphics[width=0.33\textwidth]{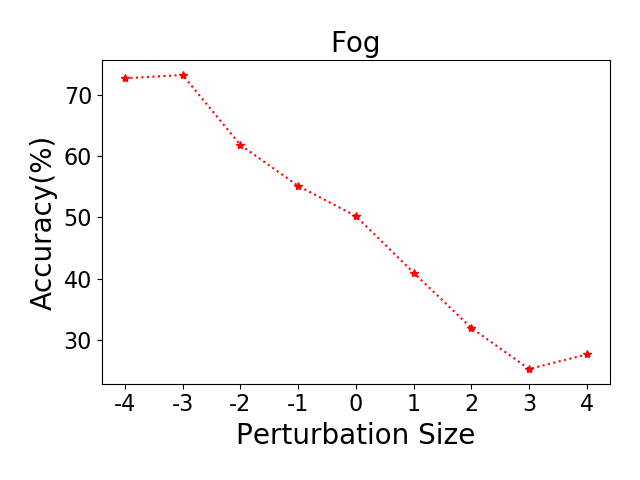}}
	\subfloat{\includegraphics[width=0.33\textwidth]{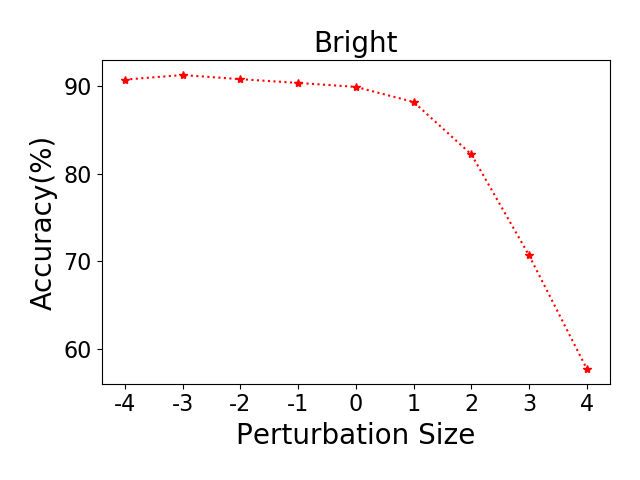}}
	\subfloat{\includegraphics[width=0.33\textwidth]{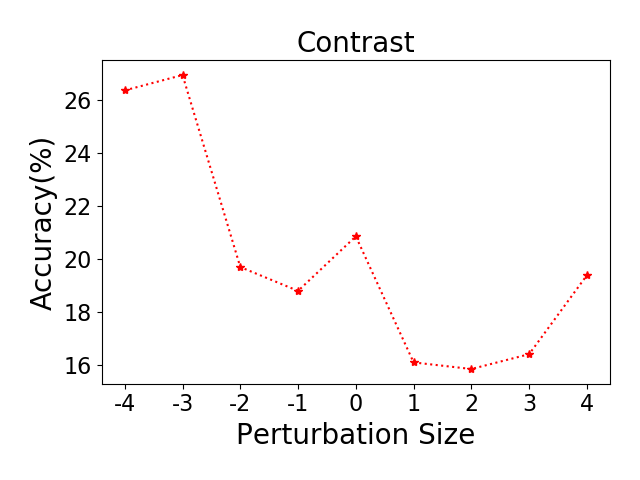}}
	\vspace{-0.2in}
	\\
	\subfloat{\includegraphics[width=0.33\textwidth]{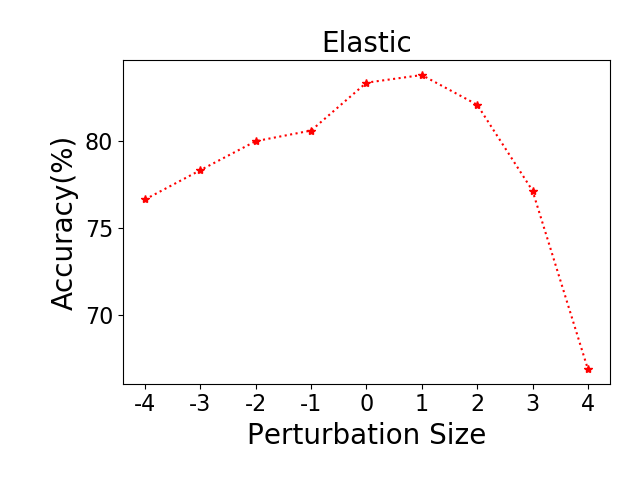}}
	\subfloat{\includegraphics[width=0.33\textwidth]{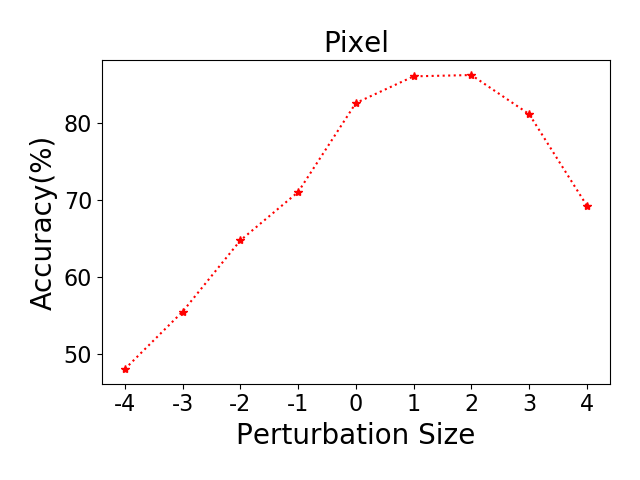}}
	\subfloat{\includegraphics[width=0.33\textwidth]{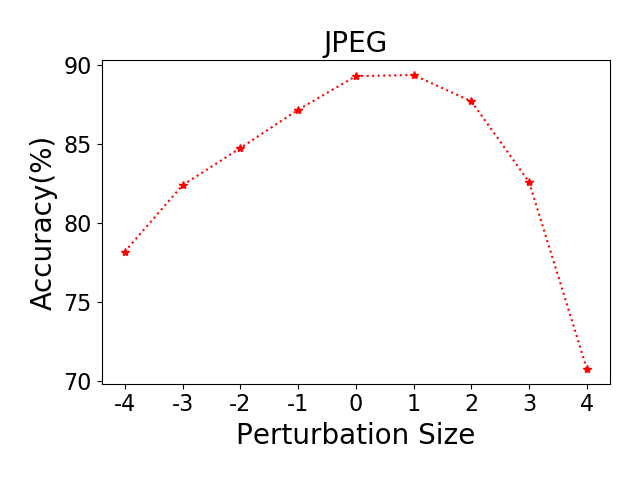}}
	\vspace{-0.2in}
	\caption{Accuracy of Adv-$\ell_{\infty}$ on \texttt{CIFAR10-C} over various perturbation sizes. The $x$-axis means the perturbation size is $2^{x}/255$.}
	\label{fig:adv_linf_r}
\end{figure*}

\begin{figure*}[htbp]\centering
	\subfloat{\includegraphics[width=0.33\textwidth]{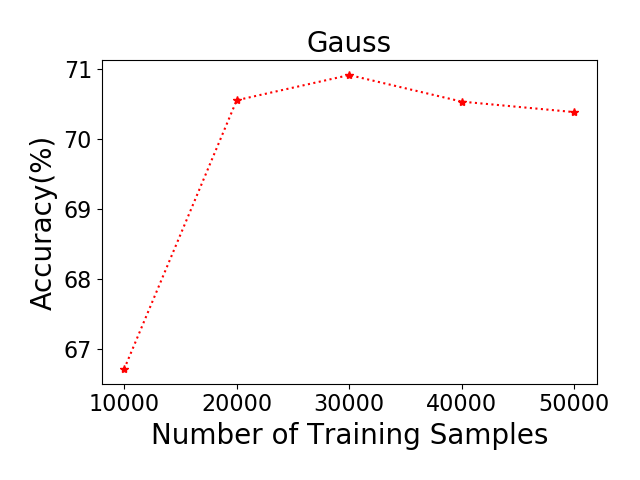}}
	\subfloat{\includegraphics[width=0.33\textwidth]{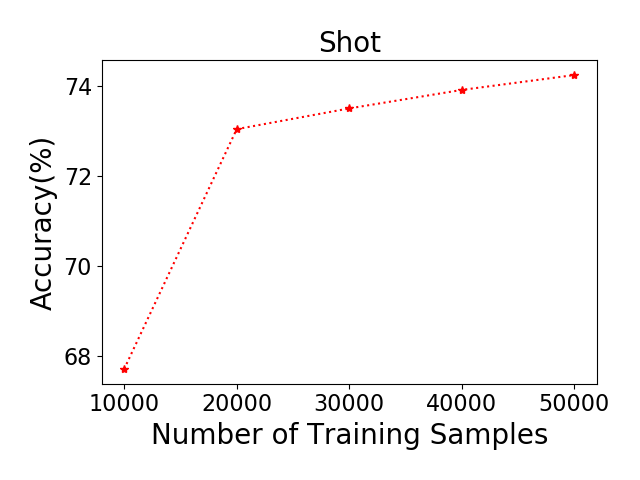}}
	\subfloat{\includegraphics[width=0.33\textwidth]{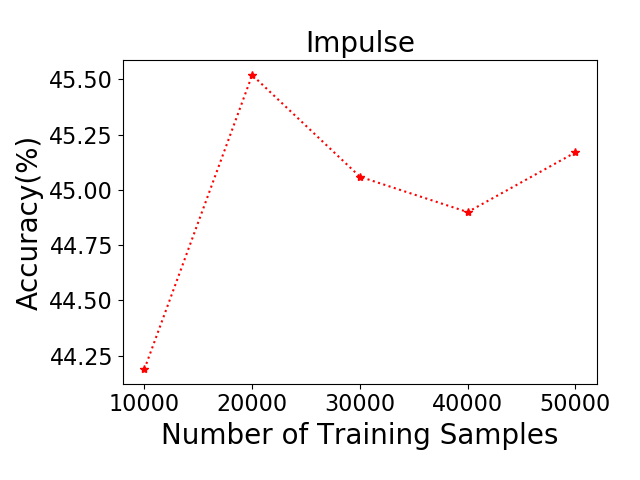}}
	\vspace{-0.2in}
	\\
	\subfloat{\includegraphics[width=0.33\textwidth]{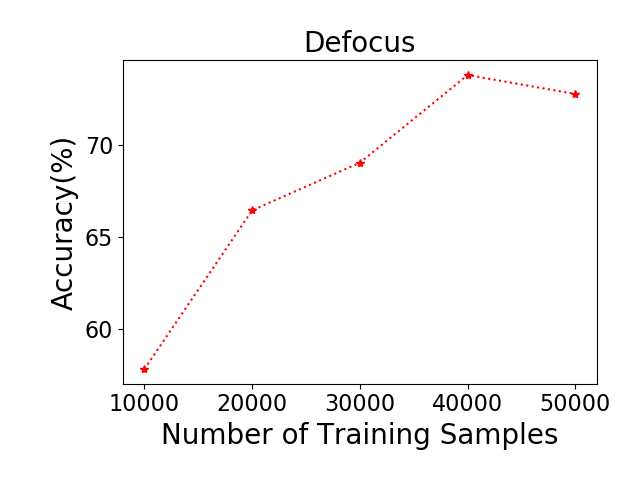}}
	\subfloat{\includegraphics[width=0.33\textwidth]{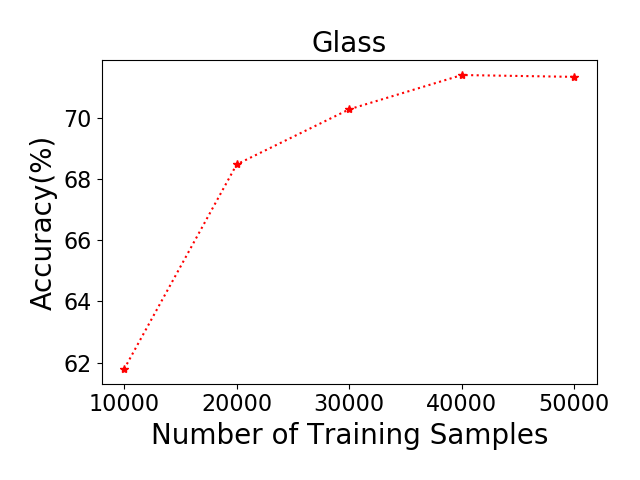}}
	\subfloat{\includegraphics[width=0.33\textwidth]{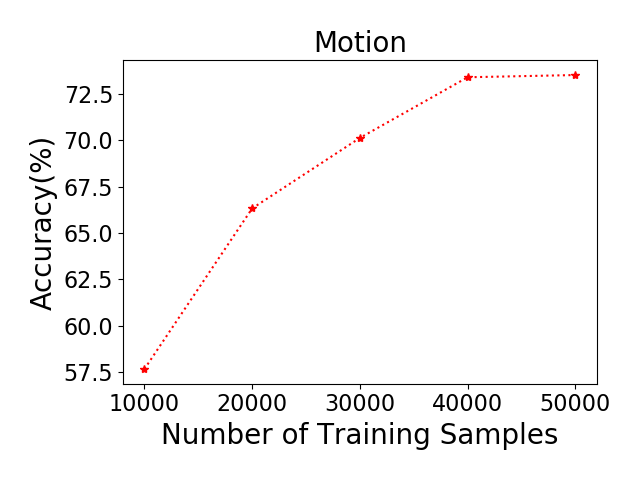}}
	\vspace{-0.2in}
	\\
	\subfloat{\includegraphics[width=0.33\textwidth]{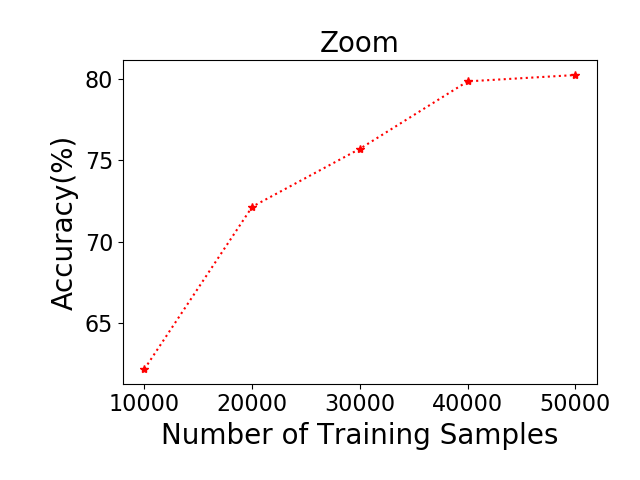}}
	\subfloat{\includegraphics[width=0.33\textwidth]{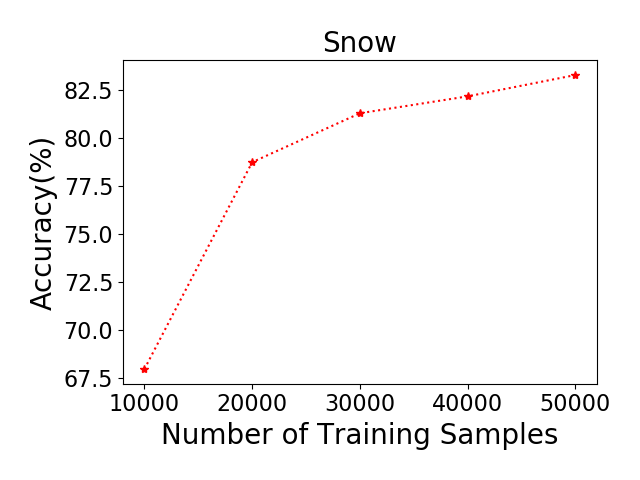}}
	\subfloat{\includegraphics[width=0.33\textwidth]{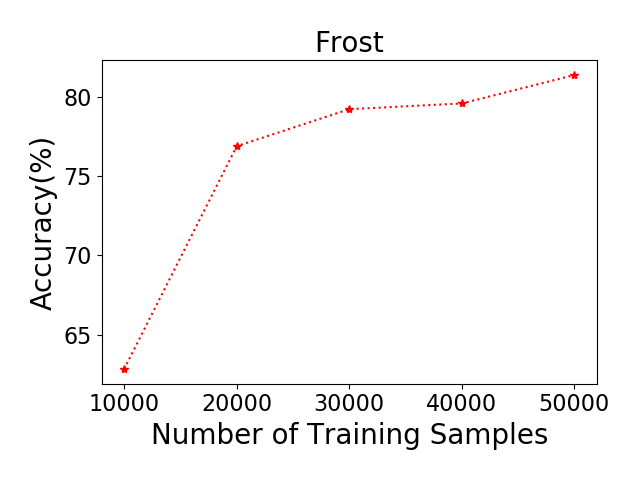}}
	\vspace{-0.2in}
	\\
	\subfloat{\includegraphics[width=0.33\textwidth]{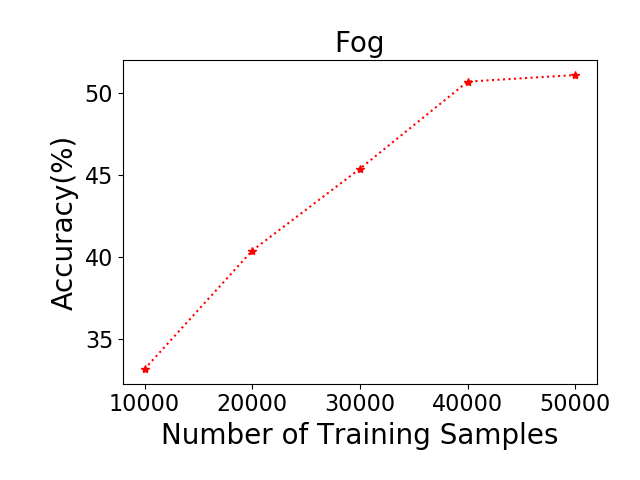}}
	\subfloat{\includegraphics[width=0.33\textwidth]{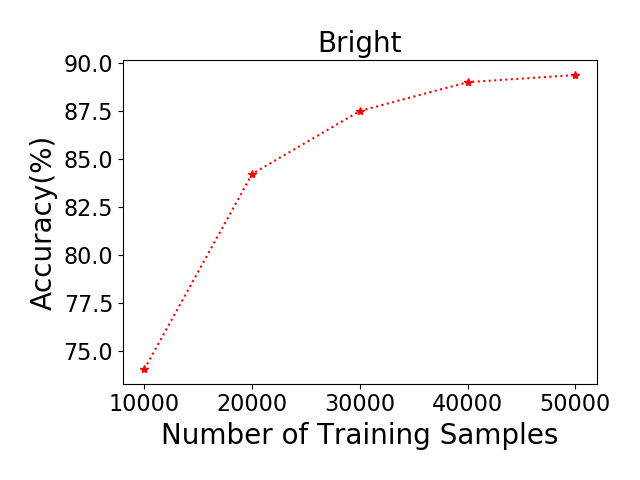}}
	\subfloat{\includegraphics[width=0.33\textwidth]{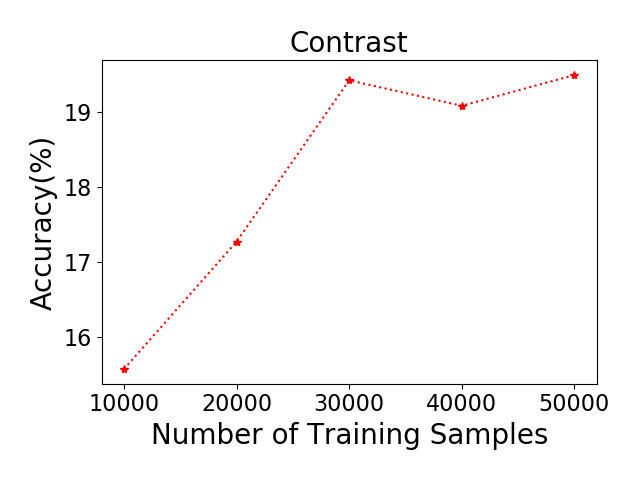}}
	\vspace{-0.2in}
	\\
	\subfloat{\includegraphics[width=0.33\textwidth]{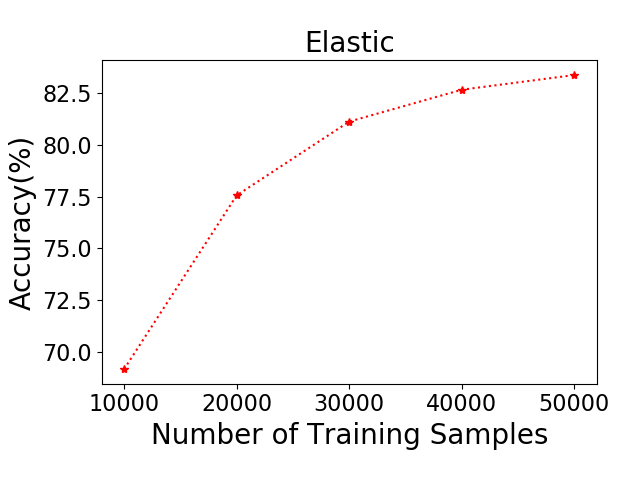}}
	\subfloat{\includegraphics[width=0.33\textwidth]{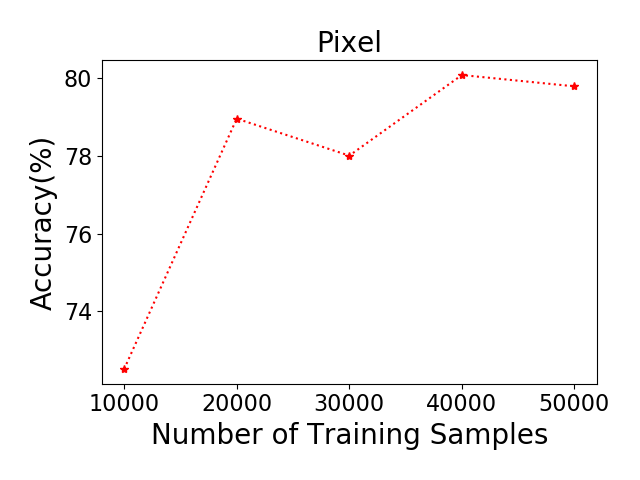}}
	\subfloat{\includegraphics[width=0.33\textwidth]{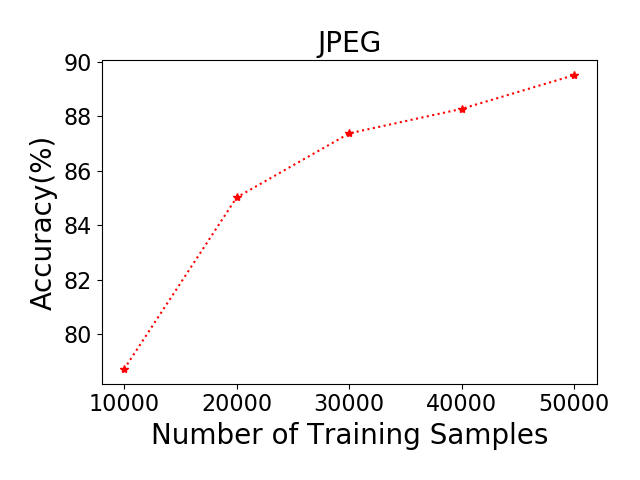}}

	\caption{Accuracy of Adv-$\ell_{2}$ on \texttt{CIFAR10-C} over various numbers of training samples.}
	\label{fig:adv_l2_num}
\end{figure*}

\begin{figure*}[htbp]\centering
	\subfloat{\includegraphics[width=0.33\textwidth]{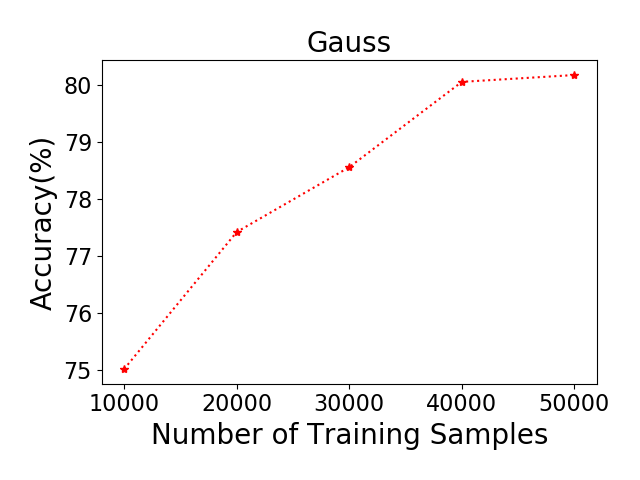}}
	\subfloat{\includegraphics[width=0.33\textwidth]{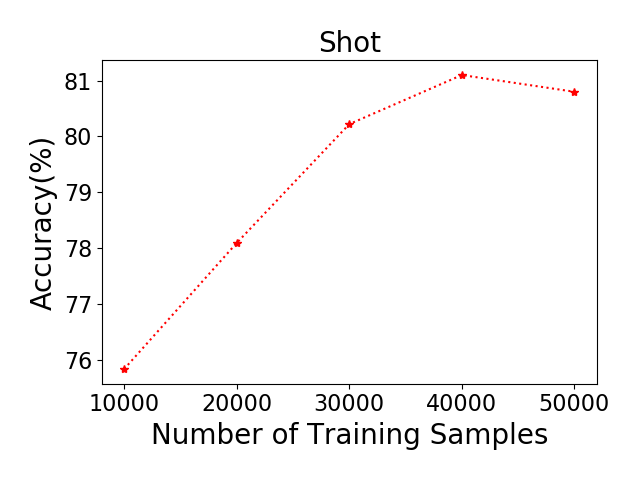}}
	\subfloat{\includegraphics[width=0.33\textwidth]{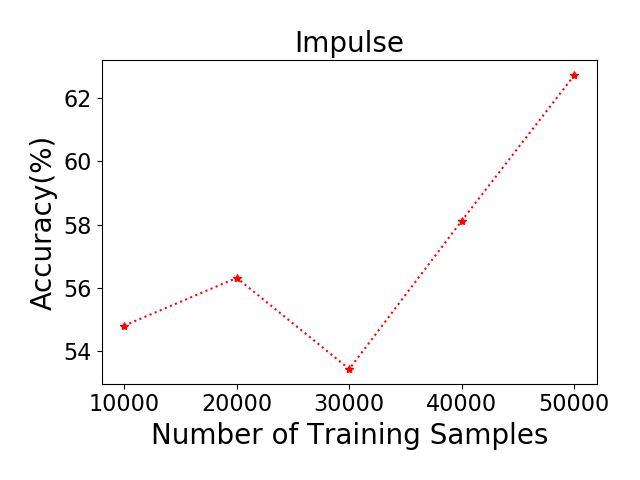}}
	\vspace{-0.2in}
	\\
	\subfloat{\includegraphics[width=0.33\textwidth]{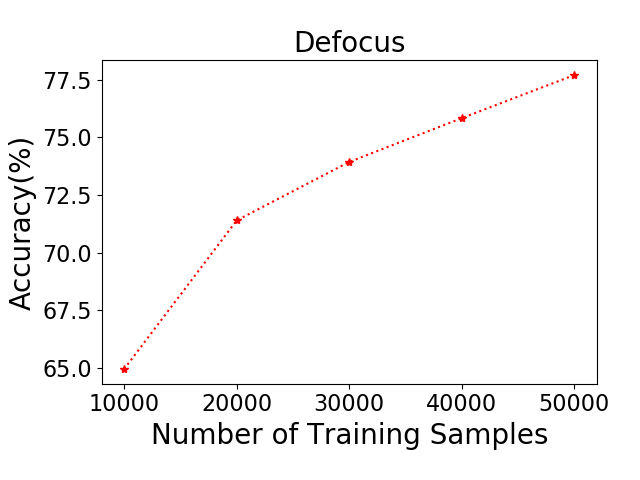}}
	\subfloat{\includegraphics[width=0.33\textwidth]{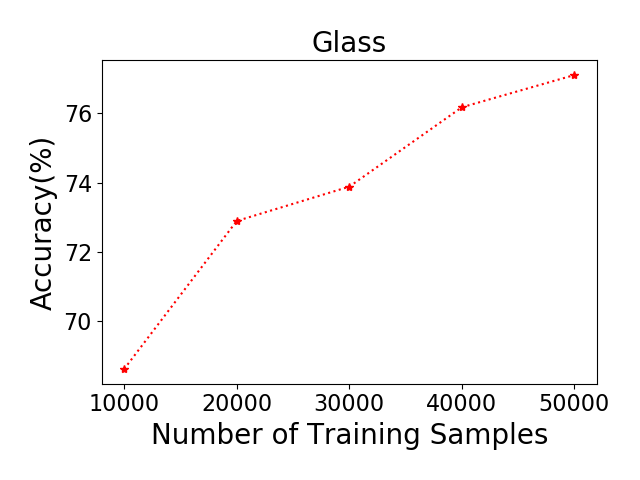}}
	\subfloat{\includegraphics[width=0.33\textwidth]{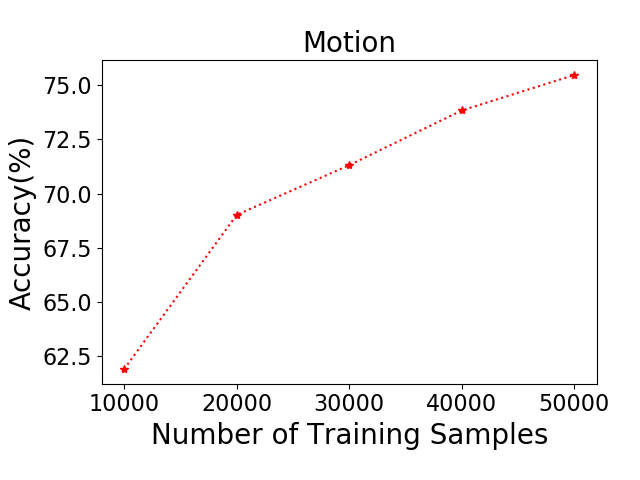}}
	\vspace{-0.2in}
	\\
	\subfloat{\includegraphics[width=0.33\textwidth]{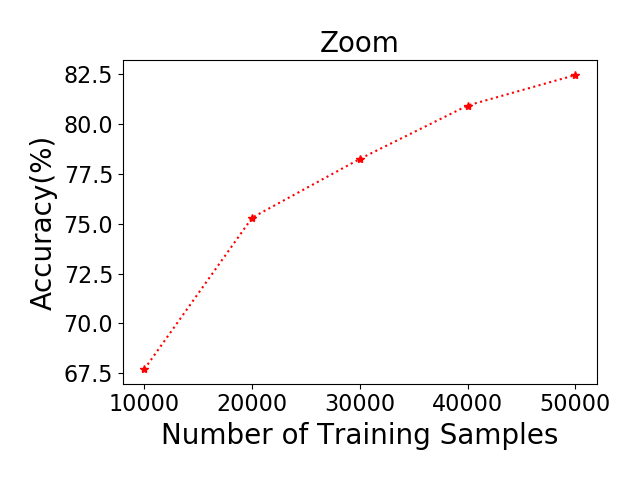}}
	\subfloat{\includegraphics[width=0.33\textwidth]{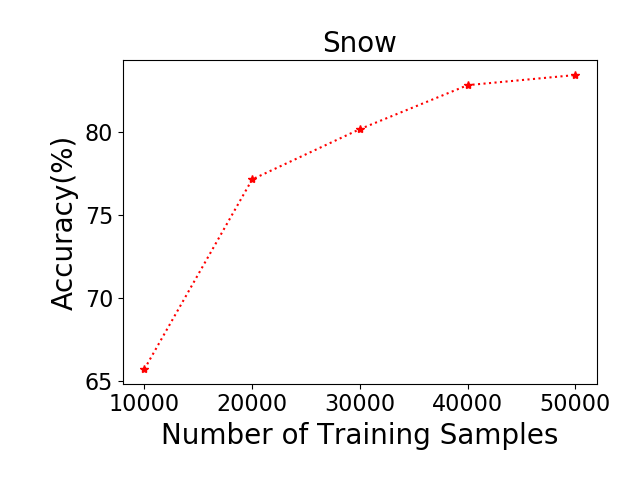}}
	\subfloat{\includegraphics[width=0.33\textwidth]{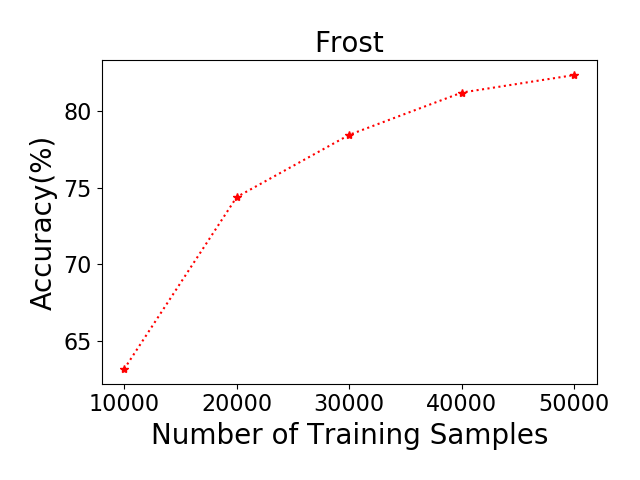}}
	\vspace{-0.2in}
	\\
	\subfloat{\includegraphics[width=0.33\textwidth]{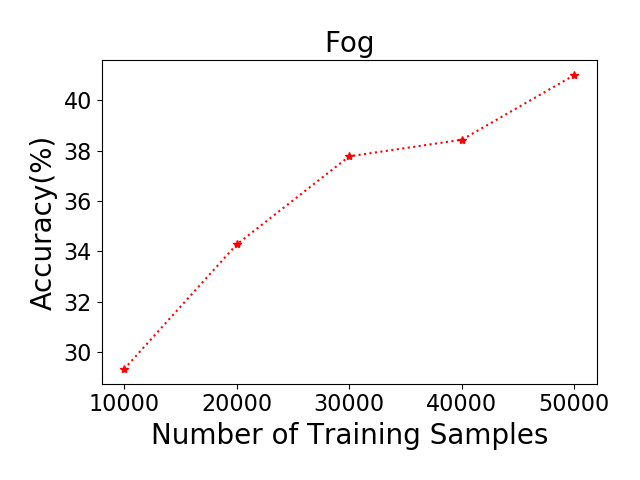}}
	\subfloat{\includegraphics[width=0.33\textwidth]{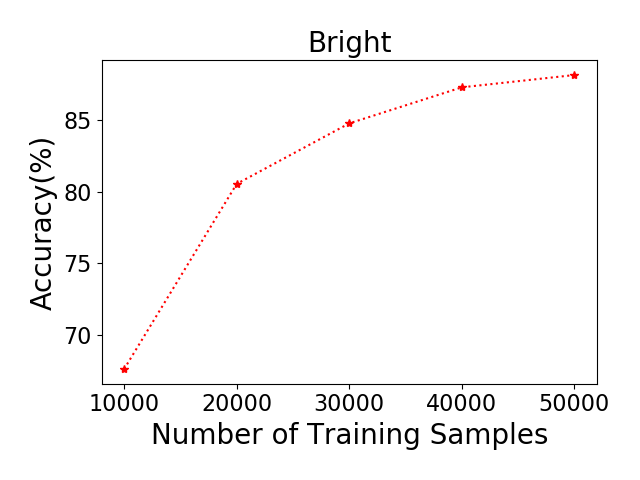}}
	\subfloat{\includegraphics[width=0.33\textwidth]{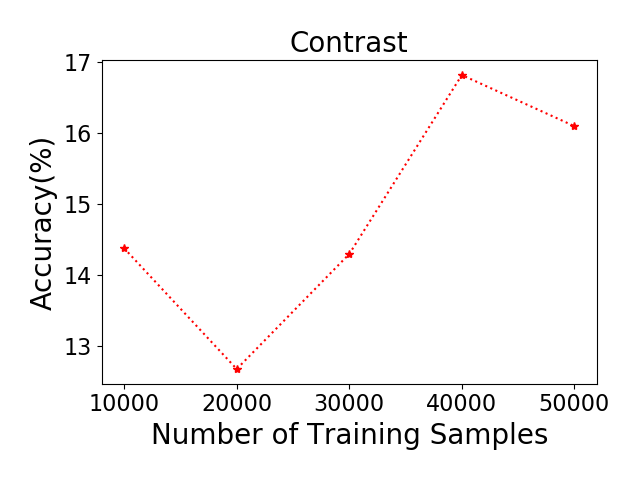}}
	\vspace{-0.2in}
	\\
	\subfloat{\includegraphics[width=0.33\textwidth]{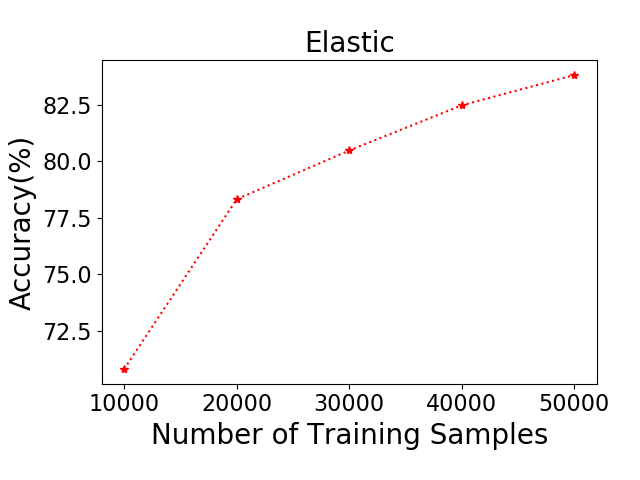}}
	\subfloat{\includegraphics[width=0.33\textwidth]{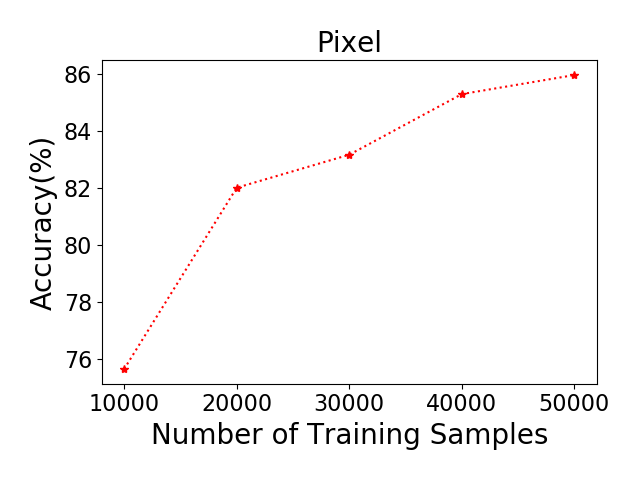}}
	\subfloat{\includegraphics[width=0.33\textwidth]{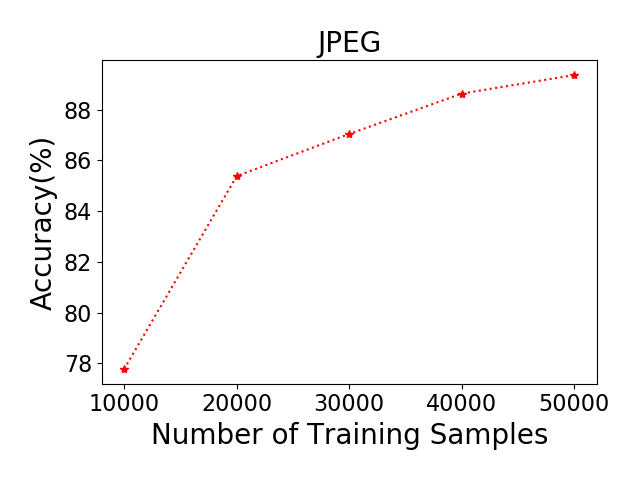}}
	\caption{Accuracy of Adv-$\ell_{\infty}$ on \texttt{CIFAR10-C} over various numbers of training samples.}
	\label{fig:adv_linf_num}
\end{figure*}

% \begin{figure*}[htbp]\centering
% 		\subfloat{\includegraphics[width=0.33\textwidth]{./pic/adv_small_size/adv_l2_fog.png}}
% 		\subfloat{\includegraphics[width=0.33\textwidth]{./pic/adv_small_size/adv_l2_bright.png}}
% 		\subfloat{\includegraphics[width=0.33\textwidth]{./pic/adv_small_size/adv_l2_contrast.png}}

%         \caption{Accuracy of Adv-$\ell_{2}$ on Fog, Bright, and Contrast corrupted data in \texttt{CIFAR10-C} over various perturbation sizes. The $x$-axis means the perturbation size is $2^{x}$.}
%         \label{fig:adv_small_size_l2}
% \end{figure*}

% 	 \begin{figure*}[htbp]\centering
% 		\subfloat{\includegraphics[width=0.33\textwidth]{./pic/adv_small_size/adv_fog.png}}
% 		\subfloat{\includegraphics[width=0.33\textwidth]{./pic/adv_small_size/adv_bright.png}}
% 		\subfloat{\includegraphics[width=0.33\textwidth]{./pic/adv_small_size/adv_contrast.png}}
% 		\caption{Accuracy of Adv-$\ell_{\infty}$ on Fog, Bright, and Contrast corrupted data in \texttt{CIFAR10-C} over various perturbation sizes. The $x$-axis means the perturbation size is $2^{x}/255$.}
% 		\label{fig:adv_small_size}
% 	\end{figure*}

\end{document}